\pdfoutput=1
\documentclass[letterpaper,11pt]{article}
\usepackage[top=1in,bottom=1.25in,left=1in,right=1in]{geometry}

\usepackage{amsmath,amsfonts,bm}

\def\Secref#1{Sec.~\ref{#1}}

\def\1{\bm{1}}

\def\vzero{{\bm{0}}}

\def\vmu{{\bm{\mu}}}
\def\vtheta{{\bm{\theta}}}
\def\vnu{{\bm{\nu}}}
\def\vphi{{\bm{\phi}}}
\def\vepsilon{{\bm{\varepsilon}}}
\def\veps{{\bm{\varepsilon}}}
\def\valpha{{\bm{\alpha}}}

\def\vb{{\bm{b}}}
\def\vc{{\bm{c}}}

\def\ve{{\bm{e}}}
\def\vf{{\bm{f}}}

\def\vh{{\bm{h}}}

\def\vm{{\bm{m}}}

\def\vp{{\bm{p}}}
\def\vq{{\bm{q}}}
\def\vr{{\bm{r}}}

\def\vu{{\bm{u}}}
\def\vv{{\bm{v}}}
\def\vw{{\bm{w}}}
\def\vx{{\bm{x}}}
\def\vy{{\bm{y}}}
\def\vz{{z}}

\def\mA{{\bm{A}}}
\def\mB{{\bm{B}}}
\def\mC{{\bm{C}}}

\def\mE{{\bm{E}}}

\def\mG{{\bm{G}}}
\def\mH{{\bm{H}}}
\def\mI{{\bm{I}}}
\def\mJ{{\bm{J}}}

\def\mP{{\bm{P}}}

\def\mU{{\bm{U}}}
\def\mV{{\bm{V}}}
\def\mW{{\bm{W}}}
\def\mX{{\bm{X}}}

\def\mSigma{{\bm{\Sigma}}}
\def\mTheta{{\bm{\Theta}}}

\DeclareMathAlphabet{\mathsfit}{\encodingdefault}{\sfdefault}{m}{sl}
\SetMathAlphabet{\mathsfit}{bold}{\encodingdefault}{\sfdefault}{bx}{n}

\def\gL{{\mathcal{L}}}

\def\gN{{\mathcal{N}}}

\def\gT{{\mathcal{T}}}

\def\gV{{\mathcal{V}}}

\def\gZ{{\mathcal{Z}}}

\newcommand{\E}{\mathbb{E}}

\newcommand{\R}{\mathbb{R}}
\renewcommand{\P}{\mathbb{P}}

\newcommand{\softmax}{\mathrm{softmax}}

\newcommand{\calT}{\mathcal{T}}
\newcommand{\calS}{\mathcal{S}}

\newcommand{\vxi}{\boldsymbol{\xi}}

\usepackage{graphicx} %

\usepackage[numbers]{natbib}

\usepackage[dvipsnames]{xcolor}
\usepackage{enumitem}
\usepackage{subcaption}
\usepackage{caption}
\usepackage{bbm}
\usepackage{algorithm}
\usepackage{mathtools}
\usepackage{algpseudocode}
\usepackage{tikz}
\usepackage{amssymb}
\usepackage{amsthm}
\usepackage{booktabs}
\usepackage{multirow}
\usepackage{hyperref}
\usepackage{comment}
\usepackage{wrapfig}
\usepackage{siunitx}
\usepackage[normalem]{ulem}
\usepackage[utf8]{inputenc} %
\usepackage[T1]{fontenc}    %
\usepackage{url}            %
\usepackage{nicefrac}       %
\usepackage{microtype}      %
\usepackage{tcolorbox}
\usetikzlibrary{positioning,calc,fit,arrows.meta}

\newtheorem{rmk}{Remark}
\newtheorem{defn}{Definition}

\newtheorem{prop}{Proposition}
\newtheorem{thm}{Theorem}
\newtheorem{cor}{Corollary}
\newtheorem{lemma}{Lemma}

\hypersetup{hidelinks}

\title{Task Vector Geometry Underlies Dual Modes of Task Inference in Transformers\thanks{Code to reproduce all experiments is available at \url{https://github.com/ezyhdxm/mini-ICL}.}}
\author{%
  Hao Yan\textsuperscript{$\dagger$}\qquad
  Haolin Yang\textsuperscript{$\ddagger$}\qquad
  Yiqiao Zhong\textsuperscript{$\dagger$}\\
  \textsuperscript{$\dagger$}Department of Statistics, University of Wisconsin--Madison\\
  \textsuperscript{$\ddagger$}Center for Data Science, New York University\\
  \texttt{hyan84@wisc.edu}\quad \texttt{hy3127@nyu.edu}\quad \texttt{yiqiao.zhong@wisc.edu}%
}
\date{}

\begin{document}

\maketitle
\suppressfloats[t] %

\begin{abstract}

Transformers are effective at inferring the latent task from context via two inference modes: recognizing a task seen during training, and adapting to a novel one. Recent interpretability studies have identified from middle-layer representations task-specific directions, or \textit{task vectors}, that steer model behavior. However, a lack of rigorous foundations hinders connecting internal representations to external model behavior: existing work fails to explain how task-vector geometry is shaped by the training distribution, and what geometry enables out-of-distribution (OOD) generalization.
In this paper, we study these questions in a controlled synthetic setting by training small transformers from scratch on latent-task sequence distributions, which allows a principled mathematical characterization. We show that two inference modes can coexist within a single model. In-distribution behavior is governed by Bayesian task retrieval, implemented internally through convex combinations of learned task vectors. OOD behavior, by contrast, arises through extrapolative task learning, whose representations occupy a subspace nearly orthogonal to the task-vector subspace. %
Taken together, our results suggest that task-vector geometry, training distributions, and generalization behaviors are closely related.

\end{abstract}

\section{Introduction}\label{sec:intro}

Large language models (LLMs) exhibit impressive performance, yet the mechanism underlying their generalization remains unclear. Since the emergence of in-context learning, empirical studies have suggested two modes of task inference \citep{brown2020language,pan2023context}: \emph{memorization-based task retrieval}, where the model updates an implicit posterior over in-distribution (ID) training tasks and retrieves previously learned behaviors, and \emph{context-based task generalization}, where it solves out-of-distribution (OOD) tasks whose latents lie outside the training support. However, the internal representations supporting these two modes remain obscure. This motivates two central questions: \textit{How do hidden states summarize context and retrieve memorized tasks? How does the representation geometry change for novel tasks outside the training support?}

\begin{figure}[t]
\centering
\includegraphics[width=\textwidth]{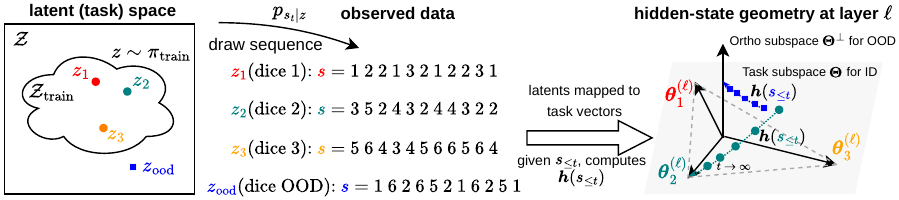}
\caption{
\textbf{Connections between data distribution and representation geometry.} 
This diagram illustrates the synthetic \textit{rolling biased dice} experiment. Each sequence is generated from an unobserved latent $z$ controlling the outcome distribution. A transformer trained on mixture data learns internal task vectors that encode these latents. Two near-orthogonal subspaces largely determine a model's inference mode and generalization behavior.} 
\vspace{-10pt}
\label{fig:main}
\end{figure}

To understand these external behaviors, researchers have analyzed internal representations, broadly centering on the %
\textit{linear representation hypothesis} (LRH) \citep{mikolov-etal-2013-linguistic, mikolov2013distributed, parklinear}, which posits that hidden states encode concepts as linear subspaces. For example, \cite{hendel2023context} discovered that in ICL, LLMs use in-context examples internally to form task-encoding vectors, commonly called \textit{task vectors}. Throughout, we use ``task vector'' exclusively in this activation-space sense, as a direction in the hidden-state space, distinct from the weight-space usage in task arithmetic and model merging. %
Concretely, given a prompt $s = (x_1, f(x_1), x_2, f(x_2), \ldots, x_n, f(x_n), x_{\mathrm{test}})$, an LLM can infer the mapping $f$ from the demonstrations and correctly predict $f(x_{\mathrm{test}})$. Internally, the hidden state in a middle layer contains a task-specific component $\vtheta_f$ that steers the model toward the correct continuation. 

\paragraph{Challenges: lack of rigorous foundations for representation geometry.} While prior work on internal representations and task vectors provides valuable insights \citep{geva2021transformer, bricken2023monosemanticity}, the mechanistic analysis is often model-dependent and its applicability is rarely verified. For example, various notions of task vectors are proposed \citep{hendel2023context, toddfunction, merullo2024languagemodelsimplementsimple}, sometimes leading to contradictory results in the literature. A central missing piece is a theoretical framework that rigorously defines task vectors and clarifies how they relate to the geometry of internal representations. Without such a foundation, existing results are largely heuristic: task vectors are often introduced operationally or extracted empirically, but their precise meaning, validity, and connection to representation structure remain unclear. As a result, current analyses fail to provide a clear mechanistic picture of how pretrained LLMs internally encode task information, how this encoding is shaped by the training data distribution, and how internal representations mediate the link between external statistical structure and model behavior.

\paragraph{Bridging the gap via latent-variable modeling.} To develop foundations for task vector representation geometry grounded in rigorous statistical principles, we mainly study synthetic experiments where we train small transformers from scratch on carefully controlled data. As shown in Fig.~\ref{fig:main}, at a high level, we view tasks as latent variables (or simply, latents) in a space $\gZ$. We model training data as a latent-task mixture: for each sequence, we first sample an unobserved latent $z \in \gZ_{\mathrm{train}} \subset \gZ$ and then sample the sequence from its task-specific distribution $p_z$. Echoing the ``multitask learner'' heuristic of GPT-2 \citep{radford2019language}, this latent-variable view is standard in classical text analysis (e.g., topic models) and has recently provided useful insights on ICL \citep{xieexplanation, garg2022can}, induction heads \citep{olsson2022context, edelman2024evolution}, etc.

This setup allows us to broadly define task vectors as mappings of latents, and characterize how a model infers tasks from interactions between hidden states and task vectors. We adopt a ``gray-box'' perspective that assumes access only to hidden states, making our analysis agnostic to architectural details. Our framework builds on the Markov structure of the data, which yields testable properties of a model's hidden states rather than assumptions (\Secref{sec:math}): when these properties hold, task-vector geometry admits a rigorous characterization, and the picture degrades gracefully as the dependence becomes higher-order (\Secref{app:order-sweep}); when they fail, as in a planted Dyck language whose long-range structure admits no compact summary, common task-vector heuristics break down (\Secref{sec:motivating-example}). Our approach thus unifies prior work on representation geometry \cite{shai2024transformers, shai2026transformers} and inference modes \cite{carroll2025dynamics, lu2025asymptotic, park2025competition}, and delimits constructively where it applies. Our main findings are summarized below.

\begin{enumerate}[itemsep=0pt, topsep=0pt]
    \item \textbf{A principled framework unifies representation geometry and generalization.} We mathematically define task vectors as representations of latent tasks, and  characterize two distinct inference modes: Bayesian task retrieval and extrapolative task learning (\Secref{sec:math}).
    \item \textbf{Task-vector geometry governs Bayesian task retrieval}. For ID sequences, the model infers latent tasks primarily by adjusting the coefficients of task vectors in its hidden states, which eventually converge to the true task vector as sequence length grows (\Secref{sec:task-vectors}).
    \item \textbf{Near-orthogonal subspace governs extrapolative task learning.} For OOD sequences, the model can often generalize by using context-related statistics encoded through near-orthogonal subspaces, which emerge only at sufficiently large task diversity (\Secref{sec:two-modes}).
    \item \textbf{The same geometry appears in pretrained LLMs.} Across 12 pretrained models, in-context task representations concentrate in a mid-depth band, and in the persona belief-dynamics setting of \citet{bigelow2025belief} the simplex coefficients of an instruction-tuned LLM track accumulating in-context evidence like a posterior (\Secref{sec:real_llm}).
\end{enumerate}

\section{Related work}\label{sec:related}

\textbf{Task vectors in language models.} Task vectors have emerged as a prominent paradigm for explaining how language models extract task information from context and use it to steer subsequent predictions. However, where such vectors should be obtained from, and which internal activations truly constitute task vectors, remains underexplored and still somewhat controversial. Existing work typically extracts task vectors using heuristic procedures such as dummy task prompts \citep{hendel2023context}, averaging across multiple task prompts \citep{toddfunction}, or opaque optimization-based selection \citep{li2024incontextlearningstatevector}, and does so from sources including intermediate ICL hidden states \citep{liu2024incontextvectorsmakingcontext}, attention-head outputs \citep{yin2025attention}, and MLP activations \citep{merullo2024languagemodelsimplementsimple}. However, these construction procedures are largely ad hoc: candidate vectors are typically selected from a small set of layers or components and validated mainly by intervention outcomes, rather than derived from a principled characterization of the underlying data distribution or representation geometry. %

\textbf{Bayesian inference and OOD generalization.} Early studies observe that LLMs can operate in two modes: retrieve seen tasks and learn new tasks in context \citep{pan2023context, wei2023larger}. Explanations such as Bayesian beliefs \citep{xieexplanation, shai2024transformers} and induction heads \citep{elhage2021mathematical, olsson2022context, doi:10.1073/pnas.2417182122} focus on either ID or OOD generalization, without reconciling the coexistence of both inference modes. Synthetic experiments in \cite{garg2022can, carroll2025dynamics, lu2025asymptotic, park2025competition} reveal that task diversity in training shapes the inference mode in specific settings, yet no analysis of internal representations is provided. Bayesian latent-variable models, such as PCA \citep{tipping1999probabilistic} and topic models \citep{blei2003latent}, are classical objects in statistics and machine learning, but their generative assumptions have no a priori reason to hold in a transformer, and few works connect them to transformer representations. Closest to our vein are belief-state geometry \citep{shai2024transformers}, which regresses hidden states onto the ground-truth belief states of a single hidden Markov process, and the belief-dynamics study of \citet{bigelow2025belief}, which tracks an LLM's belief over latent personas behaviorally as in-context evidence accumulates. %

\textbf{Near-orthogonal representation geometry.} %
A central idea in mechanistic interpretability (MI) is the \textit{superposition principle} \citep{elhage2022superposition}: models often accommodate a large number of feature vectors by arranging them in near-orthogonal positions, echoing the classical literature in sparse coding \citep{donoho1989uncertainty} and compressed sensing \citep{donoho2006compressed, candes2006robust}. Near-orthogonal geometry is empirically confirmed in transformers \citep{song2023uncovering, shai2026transformers} and serves as the foundation for MI work \citep{bricken2023monosemanticity, templeton2024scaling, park2025the, gurnee2026models}. However, near-orthogonality has not been established as a geometric foundation for the coexistence of ID and OOD inference modes.

\section{Synthetic experiment setup}\label{sec:setup}

\paragraph{Data generation.} In all synthetic experiments, we first draw a latent $z$ from a finite training set $\gZ_{\mathrm{train}} \subset \gZ$ and then sample a sequence from the latent-specific distribution.

\textbf{E1.~Rolling biased dice.} Each latent $z$ corresponds to a probability vector $\vp$ drawn from a Dirichlet prior $\mathrm{Dir}(\boldsymbol{1}_6)$. For each sequence, we first sample $\vp$ and then sample tokens independently from $\vp$.

\textbf{E2.~In-context linear regression.} Following \cite{raventos2023pretraining}, we take the latent $z:= \vw \sim \gN(\vzero, \mI_D)$ as the $D$-dim regression weight vector ($D = 6$). For each sequence, we first sample $\vw$ and then $T/2$ input-output pairs (treating each vector/scalar as a token) $\vx_1,y_1,\vx_2,y_2,\ldots,\vx_{T/2}, y_{T/2}$, where
$
\vx_t \sim \mathcal{N}(\vzero,\mI_D),
y_t = \vx_t^\top \vw + \varepsilon_t,$ and $\varepsilon_t \sim \mathcal{N}(0, 0.25)\,.$

\textbf{E3.~Mixture of Markov chains.} A latent $z := \mP \in \R^{V \times V}$ ($V = 6$) is a transition matrix for a first-order Markov chain. As in \cite{edelman2024evolution}, we draw rows of $\mP$ independently from a Dirichlet prior $\mathrm{Dir}(\boldsymbol{1}_V)$, then sample the Markov sequence $s_{1:T}$ using $\mP$, starting from a uniform initial distribution. 

\textbf{E4.~Planted Dyck language.} A Dyck language (Dyck-1) is the set of balanced bracket strings. %
The latent $z$ is a Dyck string of length $2L$ ($L=10$, giving $C_{10}=16{,}796$ latents). Conditioning on $z$, we draw $T$ background tokens i.i.d.\ uniformly from $\{0,\ldots,V-1\}$ ($V=6$), select each position independently with probability $\rho=0.25$, and plant the next character of $z$ there, overwriting the background token; e.g., with $V=2$ and $z=(()())$, one sequence is $0,1,(,0,(,1,),0,(,1,1,),)$. %

\begin{rmk}[Markov property]
In \texttt{E1}--\texttt{E3}, sequences are first-order Markov given $z$, so we expect a model to infer the unobserved $z$ by compressing the context into a `belief' about $z$. The framework extends to finite order, since an order-$w$ chain is first-order on the window of its last $w$ tokens (experiments up to order $10$ in \Secref{app:order-sweep}). \texttt{E4}, in contrast, has no finite order: no window summarizes the bracket history, and the properties we study break down there (\Secref{sec:motivating-example}).
\end{rmk}

\textbf{Model and training.} We use a decoder-only transformer with rotary embeddings (RoPE) \citep{su2024roformer}.
For experiments \texttt{E1} and \texttt{E3}, the transformer has $6$ layers, $2$ heads, and hidden dimension $d=128$; for \texttt{E2}, we follow \cite{akyurek2022learning} by increasing the number of layers to $16$ (other hyperparameters remain the same). 
All models are trained autoregressively with AdamW \citep{loshchilov2018decoupled}; see \Secref{app:hyperparams} for hyperparameters.
\section{A mathematical framework for representation geometry}\label{sec:math}

This section develops the mathematical framework used throughout the paper. \Secref{sec:task_retrieval} defines task vectors and four geometric properties of hidden states; \Secref{sec:bayesian-mode} links these properties to Bayesian task retrieval, while \Secref{sec:generalization-mode} studies extrapolative task learning and shows when it must move beyond the task-vector subspace. Secs.~\ref{sec:task-vectors}--\ref{sec:two-modes} provide empirical evidence for the two-mode picture.

\subsection{Theoretical properties of task-vector geometry}
\label{sec:task_retrieval}
Let $\vz \in \gZ$ be the unobserved latent and $s_1,s_2,\ldots$ be the observed sequence generated from the distribution $P_z$. 
For a fixed transformer layer \(\ell\), denote by
\(\vh_t^{(\ell)} \in \mathbb R^d\) the residual-stream hidden state at
position \(t\) after processing the prefix \(s_{\le t}\).
Throughout this section, we fix a layer and suppress the superscript $\ell$. Since the results below do not depend on the specific parameter space of any experiment, we identify the training latents $z$ with the index set $[K]:=\{1,\ldots,K\}$ and write $\vmu_t := \E[\vh_t]$ for the mean hidden state at position~$t$.
As noted in \Secref{sec:setup}, tasks \texttt{E1}--\texttt{E3} satisfy the first-order Markov property: $\P(s_{t+1}\mid z, s_{\le t}) = \P(s_{t+1}\mid z, s_t)$. 

Under the Markov property, $\P(s_{t+1}\mid z, s_{\le t})$ depends only on $(z, s_t)$, so a well-trained model's hidden state should likewise be approximately determined by these two factors. Define cell means $\vmu_{t,k,a} := \E[\vh_t \mid z=k, s_t=a]$ and write
\begin{equation}\label{eq:cell-means}
\vh_t \;=\; \vmu_{t,z,s_t} \;+\; \veps_t,
\end{equation}
where $\veps_t := \vh_t - \vmu_{t,z,s_t}$ is the residual not explained by $(z, s_t)$. We state four properties characterizing the hidden-state geometry when task retrieval succeeds.

\textbf{P0.~Long-context stability.}
The residual in Eq.~\ref{eq:cell-means} vanishes asymptotically: $\E\|\veps_t\|^2 \to 0$ as $t\to\infty$. In other words, the hidden state becomes a deterministic function of $(z, s_t)$, meaning the model has successfully summarized the context. While heuristic task-vector extractions \cite{toddfunction, hendel2023context} implicitly assume \texttt{P0}, we show in Sec.~\ref{sec:motivating-example} that this property is not always satisfied. 

\textbf{P1.~Long-context decoupling.}
Given \texttt{P0}, we can further ask how the cell means $\vmu_{t,k,a}$ are structured. The task-specific component naturally gives rise to the definition of \emph{task vectors} \citep{hendel2023context}.
\begin{defn}\label{def:task-vec}
For a given layer, \emph{task vectors} are $\{\vtheta_k\}_{k \le K} \subset \R^d$ defined by
\begin{equation}\label{eq:task-vec:defn}
    \vtheta_k = \lim_{t \to \infty}\big(\E[\vh_t \mid z=k] - \vmu_t\big),
\end{equation}
whenever the limit exists. Centering ensures $\sum_{k \le K} \vtheta_k = \vzero$.
\end{defn}
\texttt{P1} requires that the cell means decompose additively into task and token components. There exist token-encoding vectors $\{\vnu_a\}_{a \in \gV} \subset \R^d$ such that
\begin{equation}
\E \Bigl\|\vh_t - \bigl(\underbrace{\vmu_t}_{\text{global mean}} + \underbrace{\vtheta_z}_{\text{task vector}} + \underbrace{\vnu_{s_t}}_{\text{token encoding}}\bigr)\Bigr\|_2^2 \to 0
\qquad \text{as } t \to \infty.
\label{eq:long_context_stability_main}
\end{equation}
This additive structure is a form of the \textit{linear representation hypothesis} \citep{parklinear}; in ANOVA terms, \texttt{P1} requires that the interaction between $z$ and $s_t$ is negligible.
\begin{rmk}[Task vectors as class-conditional means]\label{rmk:class-means}
Definition~\ref{def:task-vec} identifies task vectors with centered class-conditional means of hidden states, the construction used in prior work \citep{hendel2023context, toddfunction}. Under the additive decomposition of \texttt{P1}, this choice is not special: alternative estimators such as LDA, PCA after conditioning on the current token, and linear probes recover the same vectors (\Secref{app:estimators}).
\end{rmk}

\textbf{P2.~Finite-context interpolation.}
At finite context, the model represents its knowledge about latent tasks via a convex combination of task vectors. There exist nonnegative coefficients $\{\beta_{t,k}\}_{t\ge 1,\, k\le K}$ satisfying $\sum_{k\le K}\beta_{t,k}=1$ (depending on context $s_1,\ldots,s_t$) such that
\begin{equation}
\vh_t \approx \vmu_t + \sum_{k\le K}\beta_{t,k}\,\vtheta_k + \vnu_{s_t}.
\label{eq:additive_interp_main}
\end{equation}
\texttt{P2} formalizes the \textit{superposition principle} \citep{elhage2022superposition}: hidden states are linear combinations of task vectors.

\textbf{P3.~Bayesian posterior alignment.}
\texttt{P3} strengthens \texttt{P2} by pinning down the same coefficients $\beta_{t,k}$: they approximate the Bayesian posterior,
\begin{equation}
\beta_{t,k} \approx \alpha_{t,k} := \P(z=k \mid s_{\le t}).
\label{eq:posterior_alignment_main}
\end{equation}

This provides a mechanistic characterization of \textit{Bayesian inference}: while the Bayesian perspective has been used to explain ICL \citep{xieexplanation}, its internal geometry has remained unclear. These properties are progressively stronger: \texttt{P3} $\Rightarrow$ \texttt{P2}, \texttt{P1} and \texttt{P1} $\Rightarrow$ \texttt{P0} in the long-context limit; see \Secref{app:property-proofs}. These properties can be verified empirically. When they hold, they yield task vectors with a precise meaning, recovering the objects described in the literature (\Secref{sec:task-vectors}); when they fail, as for the planted Dyck language in \Secref{sec:motivating-example}, they mark the boundary beyond which task vectors do not exist. All empirical claims below are stated for the layer band at which they are evaluated (\Secref{sec:emp-task-vector}). %

\subsection{Task retrieval: task-vector geometry mirrors Bayesian inference}\label{sec:bayesian-mode}

When a model has memorized a task $z$ during training, it can retrieve the latent task $z$ from test prompts via approximate Bayesian  prediction; we call this the \textbf{Bayesian task retrieval mode} (\texttt{M1}). Formally, let $\pi_{\mathrm{train}}$ be the prior distribution over $\gZ_{\mathrm{train}} \subset \gZ$. The optimal prediction under \texttt{M1} is
\begin{equation}\label{eq:bayes}
    \P(s_{t+1}\!=\!a \mid s_{\le t}) \!=\! \int_{\gZ'} \P(s_{t+1}\!=\!a \mid z, s_{t})\, p(z \mid s_{\le t}) \mathrm{d}z, ~~ p(z \mid s_{\le t}) = \frac{p_z(s_{\le t})\,\pi_{\mathrm{train}}(z)}{\int_{\gZ'} p_z(s_{\le t})\,\pi_{\mathrm{train}}(z)},
\end{equation}
where $\gZ' = \gZ_{\mathrm{train}}$ and $p(z \mid s_{\le t})$ is the exact $\alpha_{t,k}$ in Eq.~\ref{eq:posterior_alignment_main}. %
By Eqs.~\ref{eq:additive_interp_main} and~\ref{eq:posterior_alignment_main}, \(\vh_t\) carries the same posterior weights as Eq.~\ref{eq:bayes}: after a shared centering map removes \(\vmu_t\), the informative component is \(\sum_{k\le K}\alpha_{t,k}(\vtheta_k+\vnu_{s_t})\). Thus Bayesian retrieval reduces to a finite lookup-and-mix computation: lookup \(\P(\cdot\mid z=k,s_t)\) from the task/token anchors and mix the results with \(\alpha_{t,k}\).

\begin{thm}[Informal; task-vector subspace can implement Bayesian inference]
\label{thm:bayes_realization_from_hidden}
Suppose that at some layer \(\vh_t\) satisfies \texttt{P0}--\texttt{P3}. Under the structural conditions formalized in \Secref{app:transformer-construct}, two additional transformer blocks and a suitable unembedding can make the ID prediction distribution approximate the Bayesian predictive distribution in Eq.~\ref{eq:bayes}.
\end{thm}
\Secref{app:transformer-construct} gives the proof. Theorem~\ref{thm:bayes_realization_from_hidden} is about expressivity instead of training dynamics: it guarantees that hidden states satisfying \texttt{P0}--\texttt{P3} suffice to produce the Bayesian predictive distribution, so the Bayesian interpretation is not vacuous. \Secref{sec:task-vectors} confirms empirically that models learn such a solution. %

\subsection{Task learning: OOD inference operates outside task-vector subspace}
\label{sec:generalization-mode}

How can a model infer \(\vz' \notin \gZ_{\mathrm{train}}\) from context and achieve OOD generalization? One possibility is that it applies Eq.~\ref{eq:bayes} with an uninformative prior over the full space \(\gZ\) rather than only the training support, extrapolating beyond seen latents. We call this the \textbf{extrapolative task-learning mode} (\texttt{M2}).

\texttt{M2} relies on context statistics (empirical unigram for \texttt{E1}, ridge solution for \texttt{E2}, and empirical bigram for \texttt{E3}) rather than memorized task vectors (\Secref{app:approx-posterior}).
The key question is whether such extrapolative inference behavior is enabled by a different geometry compared with \texttt{M1}. 
We give a theoretical result: when the OOD prediction map has intrinsic dimension $d_0 > k_\star$, where $k_\star$ is the task-vector subspace dimension, a model implementing \texttt{M2} cannot confine its representations to that subspace.

More formally, suppose $\gV$ is discrete and fix \(a_\star\in\gV\). Let
$\vp(\vz):=\P(\cdot\mid \vz,a_\star)$ denote the next-token distribution.
Assume \texttt{P0} holds and that, for $\vz$ near an interior point
\(\vz_0\in\gZ\), the limit
$\vmu_\star(\vz):=\lim_{t\to\infty}\vmu_{t,\vz,a_\star}$ exists. Define the
task subspace
$\calT:=\operatorname{span}(\vtheta_2-\vtheta_1,\dots,\vtheta_K-\vtheta_1)$,
set $k_\star:=\dim\calT$ and $\calS:=\vnu_{a_\star}+\calT$. Let
$d_0:=\operatorname{rank}\nabla\vp(\vz_0)$ be the local intrinsic dimension
of $\vp$ at $\vz_0$.

\begin{thm}[Informal, distance to task subspace lower bound]
\label{thm:task-distance-lower}
Assume \texttt{P0}. Suppose that on some neighborhood \(U\ni \vz_0\), the prediction \(\hat{\vp}(\vz)\) decoded from \(\vmu_\star(\vz)\) is locally \(L\)-Lipschitz:
$
\|\hat{\vp}(\vz)-\hat{\vp}(\vz')\|_2
\le
L\,\|\vmu_\star(\vz)-\vmu_\star(\vz')\|_2,
$
for all
$
\vz,\vz'\in U.
$
Additionally, assume that
$
\|\hat{\vp}(\vz)-\vp(\vz)\|_2\le \delta,
$
for all
$
\vz\in U
$
and \(\{\vmu_\star(\vz):\vz\in U\}\) is bounded. Then there exist local constants \(c,C>0\) such that
\[
\sup_{\vz\in U} \inf_{\vc \in \calS} \|\vmu_\star(\vz)-\vc\|_2
\ge
\frac{c-\delta}{L}
-
\frac{C}{2^{d_0/k_\star}-1}
.
\]
\end{thm}

The second term in the lower bound decays exponentially in \(d_0/k_\star\), so when \(d_0 \gg k_\star\) the model cannot implement \texttt{M2} accurately while staying close to $\calS$; smoother decoding (smaller \(L\)) forces a larger deviation. Note that this result establishes a separation between \texttt{M2} and the task subspace, but does not force the geometry to be orthogonal by itself. See \Secref{sec:simplex-packing-obstruction} for the formal theorem and Remark~\ref{rmk:thm2-scope} therein for more discussion.

\section{Task-vector geometry represents Bayesian posterior beliefs}\label{sec:task-vectors}
Section \ref{sec:task_retrieval} predicts the following behavior under the Bayesian task-retrieval mode (\texttt{M1}).

\begin{tcolorbox}
\textbf{Task-vector geometry represents Bayesian inference}: For ID sequences, the model represents training tasks as learned task vectors and updates their mixture weights from context, approximately tracking the Bayesian posterior. 
\end{tcolorbox}

\subsection{Evaluation results of properties \texttt{P0}--\texttt{P3}}\label{sec:emp-task-vector}
We show that the stated properties of task vectors are mostly satisfied for \texttt{E1}--\texttt{E3}. Fix $K=3$. For each experiment, we estimate task vectors by replacing the expectation in Eq.~\ref{eq:task-vec:defn} with empirical averages $\hat \E$ over 1,024 independent sequences per task; estimated quantities use hat notation.

\textbf{Evaluating long-context stability.} We assess how much the variance of hidden states is attributed to latent $z \in [K]$ and the current token $a \in \gV$. For experiments \texttt{E1} and \texttt{E3} with discrete tokens, we calculate $\hat \vmu_{t,k, a} := \hat \E[\vh_t \mid z=k, s_t = a]$ and $\hat \vmu_t := \hat \E[\vh_t]$. By a balanced-cell ANOVA diagnostic,
\[
\underbrace{\vphantom{\sum_{k,a}}\hat \E \|\vh_{t} - \hat \vmu_t\|^2 \vphantom{\sum_{k,a}} }_{\mathrm{SS}_{\mathrm{total}}:~\widehat{\mathrm{Var}}(\vh_t)}
\!=\!
\underbrace{(KV)^{-1}\sum_{k,a} \hat \E \big[\|\hat \vmu_{t,k,a} - \hat \vmu_t\|^2\big]}_{\mathrm{SS}_{\mathrm{between}}}
\!+\!
\underbrace{(KV)^{-1}\sum_{k,a} \hat \E \big[\|\vh_t - \hat \vmu_{t,k,a}\|^2 | z=k, s_t=a\big]}_{\mathclap{\mathrm{SS}_{\mathrm{within}}:\ \widehat{\mathrm{Var}}(\vh_t \mid z, s_t)}}.
\]
and we measure the residual variance ratio $ \mathrm{SS}_{\mathrm{within}} / \mathrm{SS}_{\mathrm{total}}$. This ratio is zero if and only if hidden states $\vh_{t}$ are completely determined by latents and current tokens $z, a$. For experiment \texttt{E2} where tokens are continuously valued, we approximate the proportion of conditional variance using analysis of covariance (ANCOVA) and report the $R^2$ of the linear fit $\vh_t \approx \vmu_t + \vtheta_k + \mB_k \vx_t$ where $\mB_k$ is a task-specific coefficient matrix. See~\Secref{app:additive-separability} for details.

\begin{table}[t]
\centering
\scriptsize
\renewcommand{\arraystretch}{0.9}
\setlength{\tabcolsep}{4pt}
\caption{Hidden-state variance decomposition at the final context position.
\textit{(a)} Residual variance ratio $\mathrm{SS}_{\mathrm{within}}/\mathrm{SS}_{\mathrm{total}}$: unexplained variance; lower is better.
\textit{(b)} Interaction proportion $\eta^2_{\mathrm{interaction}}$: share of $\mathrm{SS}_{\mathrm{between}}$ from task--token interaction; lower supports additivity.}
\begin{tabular}{@{}lcccccc@{\hspace{10pt}}lcccccc@{}}
\toprule
\textbf{Layer} & 0 & 1 & 2 & 3 & 4 & 5 & \textbf{Layer} & 4 & 6 & 8 & 10 & 12 & 14 \\
\midrule
\multicolumn{14}{l}{\textit{(a) Residual variance ratio $\mathrm{SS}_{\mathrm{within}} / \mathrm{SS}_{\mathrm{total}}$}} \\[2pt]
\texttt{E1} (Dice)
& 0.012 & 0.021 & 0.025 & 0.017 & 0.006 & 0.004
& \texttt{E2} (Linear)
& 0.075 & 0.084 & 0.061 & 0.043 & 0.040 & 0.045 \\
\texttt{E3} (Markov)
& 0.037 & 0.071 & 0.029 & 0.012 & 0.007 & 0.005
&  &  &  &  &  &  \\
\midrule
\multicolumn{14}{l}{\textit{(b) Interaction proportion $\eta^2_{\mathrm{interaction}}$}} \\[2pt]
\texttt{E1} (Dice)
& 0.003 & 0.006 & 0.012 & 0.013 & 0.007 & 0.004
& \texttt{E2} (Linear)
& ${<}0.001$ & ${<}0.001$ & 0.007 & 0.040 & 0.081 & 0.117 \\
\texttt{E3} (Markov)
& ${<}0.001$ & 0.004 & 0.019 & 0.125 & 0.420 & 0.512
& & & & & & & \\
\bottomrule
\end{tabular}
\label{tab:p0_p1_combined}
\vspace{-10pt}
\end{table}

Table~\ref{tab:p0_p1_combined}(a) confirms that at the last context position the residual variance ratio is small across all layers and all three experiments, so latents and last tokens together account for most of the hidden-state variance at large $t$. Fig.~\ref{fig:task_vector_r2} (\Secref{app:additive-separability}) shows the full residual-ratio curves across context positions: the ratio is larger at early positions where limited context hinders task inference, and smaller in later layers where depth aids it.

\textbf{Evaluating long-context decoupling.}
We assess the  additive separability as in Eq.~\ref{eq:long_context_stability_main} using ANOVA, splitting $\mathrm{SS}_{\mathrm{between}}$ into task, token, and interaction components (\Secref{app:additive-separability}).
Table~\ref{tab:p0_p1_combined}(b) reports the interaction proportion $\eta^2_{\mathrm{interaction}}$ at the last context position: it is small across most layers in all three experiments, confirming the additive model and validating \texttt{P1}.

\textbf{Evaluating finite-context interpolation.}
We test \texttt{P2} at finite context by fitting the interpolation model in Eq.~\ref{eq:additive_interp_main} with simplex-constrained coefficients $\beta_{t,k} \ge 0$, $\sum_k \beta_{t,k} = 1$ (Fig.~\ref{fig:averaging_r2}). Across all three experiments, $R^2$ is close to $1$ in most layers and positions. Late layers of \texttt{E2} and \texttt{E3}, however, develop task-specific token encodings (large $\eta^2_{\mathrm{interaction}}$ in Table~\ref{tab:p0_p1_combined}(b) and Fig.~\ref{fig:additive_sep} in~\Secref{app:additive-separability}), causing $R^2$ to degrade, plateauing at $0.60$ and $0.52$ in the last two layers of \texttt{E3}. \Secref{app:ols-probe-e3} gives a finer-grained variance decomposition by layer. Overall, the interpolation model provides substantial support for \texttt{P2}.

\begin{figure}[t]
\centering
\includegraphics[width=\linewidth]{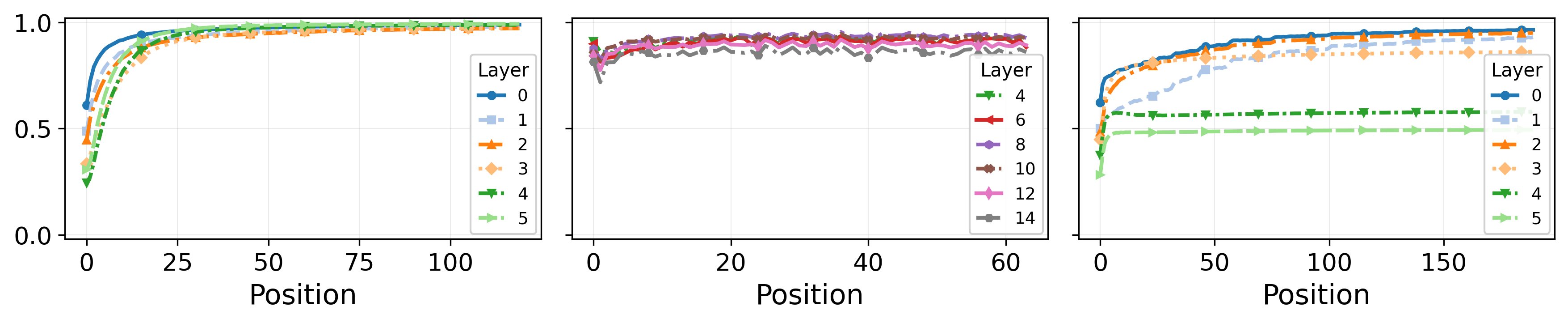}
\caption{\textbf{Finite-context interpolation approximately holds.} $R^2$ of the interpolation model (Eq.~\ref{eq:additive_interp_main}) across context positions and layers for \texttt{E1, E2, E3}. Values are close to $1$ except in late layers and early positions.
}
\vspace{-4mm}
\label{fig:averaging_r2}
\end{figure}

\textbf{Evaluating Bayesian posterior alignment.} We compare the ground-truth posterior $\alpha_{t,k}$ (\Secref{app:posterior-computation}) with the model's simplex-projected coefficients $\beta_{t,k}$ (Fig.~\ref{fig:beta_alpha_trajectory}). Across all three tasks, $\bar\beta_{t,k}$ closely tracks $\alpha_{t,k}$ as context accumulates. \Secref{app:posterior-alignment-layers} confirms that the agreement strengthens in later layers and is not an artifact of the simplex projection.

\textbf{Depth dependence.} The empirical claims in this paper refer to a mid-depth band of layers, selected as the layers maximizing held-out ID task-subspace $R^2$, because the properties are not uniform across depth; Remark~\ref{rmk:depth} in \Secref{app:frozen-band} summarizes how each property varies with depth.

\begin{figure}[t]
\centering
\includegraphics[width=\linewidth]{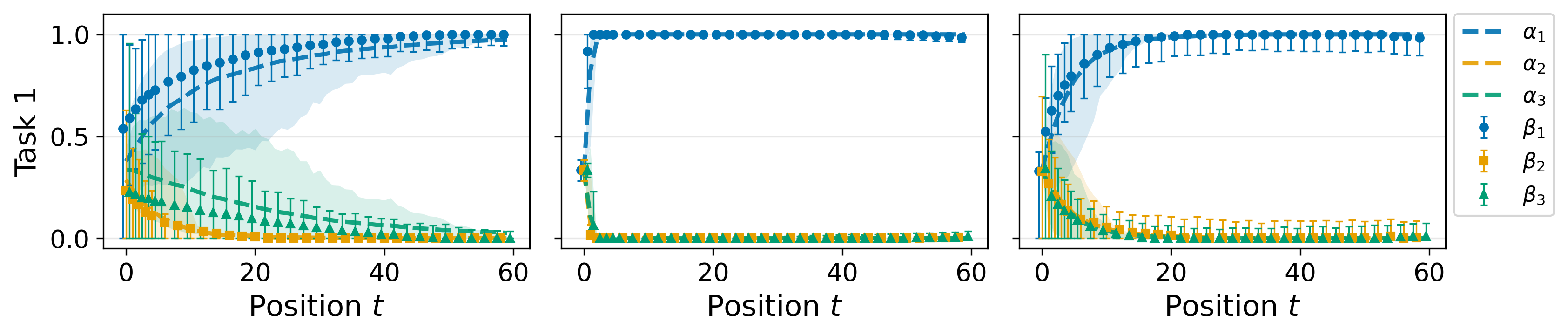}
\caption{\textbf{Bayesian posterior alignment.} Simplex-projected coefficients $\beta_{t,k}$ (markers, 10--90\textsuperscript{th} percentile error bars) vs.\ ground-truth posterior $\alpha_{t,k}$ (dashed lines, shaded bands) for \texttt{E1, E2, E3}.
}
\vspace{-10pt}
\label{fig:beta_alpha_trajectory}
\end{figure}

\subsection{Intervention analysis of task vectors} 
\label{sec:task-vec:interv}
To go beyond observational analysis, we intervene on hidden representations to test whether task vectors causally control predictions. At each position, we replace the task-subspace component with a target interpolation
$
\sum_k \alpha_{t,k}^{*}\, \hat{\vtheta}_k,
$
where $\alpha_{t,k}^{*}$ is randomly drawn from the simplex (see~\Secref{app:alpha_injection}).  
This intervention steers the hidden representation to a chosen point in the task-vector simplex. 
We calculate the difference between steered model outputs and the theoretical predictions based on $\alpha_{t,k}^{*}$; we also include the unsteered model as a baseline.

Fig.~\ref{fig:injection_simplex} shows that steering moves the model's outputs toward the corresponding mixture predictions: KL divergence drops from $0.11$ to $0.03$ in \texttt{E1} and from $0.23$ to $0.06$ in \texttt{E3}, and RMSE drops from $2.8$ to $1.1$ in \texttt{E2}. This demonstrates that the task-vector subspace provides a controllable representation through which the model interpolates between tasks.
Note that the intervention identifies the individual vertices, not merely their span: same-plane bases with a wrong vertex assignment (relabeled or rotated in-plane) fail to steer (KL $0.14$--$0.20$ vs.\ $0.06$; \Secref{app:vertex-identity}). See~\Secref{app:alpha_injection} for details and baselines. %

\begin{figure}[t]
\centering
\includegraphics[width=\linewidth]{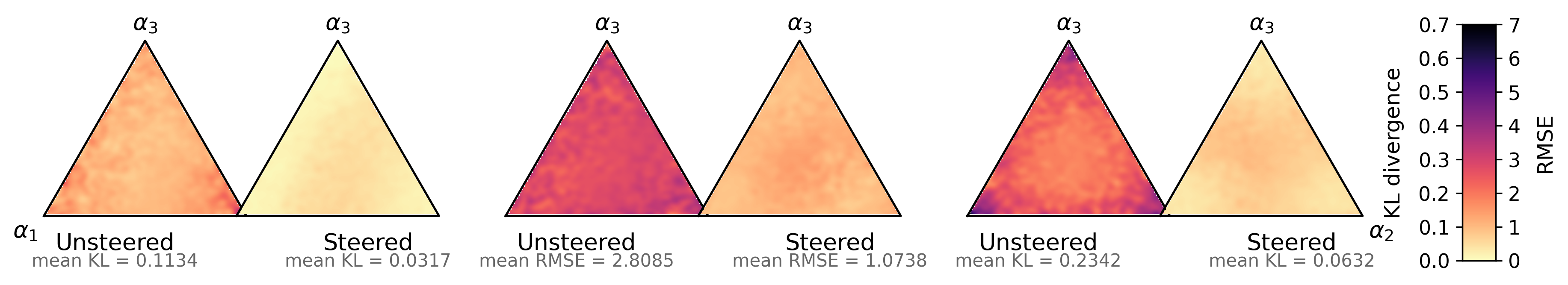}
\caption{\textbf{Causal intervention via the task-vector simplex.} Substituting $\beta_{t,k}$ with randomly drawn $\alpha_{t,k}^*$ steers model outputs to align with the corresponding mixture predictions, confirming causal effects.
}
\vspace{-10pt}
\label{fig:injection_simplex}
\end{figure}

\section{Coexistence and  geometry of two inference modes}\label{sec:two-modes}

\label{sec:two-mode}

\Secref{sec:task-vectors} established that task-vector geometry implements Bayesian retrieval for memorized tasks. We now ask whether extrapolative task learning can coexist with retrieval, and what geometry supports it.
Sec.~\ref{sec:bayesian-mode} and~\ref{sec:generalization-mode} hinted that increasing the support of $\gZ_{\mathrm{train}}$ shifts the model toward \texttt{M2} mode, since the task-vector subspace alone cannot carry task inference over a high-dimensional latent space.
To investigate this, we consider the training mixture $\gZ_{\mathrm{train}} = \gZ_{\mathrm{major}} \cup \gZ_{\mathrm{minor}}$: 3 major tasks in $\gZ_{\mathrm{major}}$ receive prior 0.9 and $N_{\mathrm{minor}}=2^m$ minor tasks share 0.1, as described in \Secref{app:training-mixture}.
This split underlies the geometric analysis of \Secref{sec:representation_geometry}: the major tasks fix a task-vector subspace, and we test whether \texttt{M2} operates within it by projecting OOD representations onto this subspace. That geometric question is meaningful only if both modes arise, so we first verify behaviorally that sweeping the diversity exponent $m$ drives the predicted \texttt{M1}-to-\texttt{M2} transition.

\textbf{Task diversity drives the model's inference mode.} We train transformers under varying $m$ and measure each model's predictive distribution against the ideal predictions from both \texttt{M1} and \texttt{M2} (Eq.~\ref{eq:bayes}), yielding two KL divergences under varying $m$ and training steps. %
Plotting $\log(\mathrm{KL}_{\mathrm{bayesian}}/\mathrm{KL}_{\mathrm{extrapolative}})$ in Fig.~\ref{fig:kl_combined} reveals clear phase transitions across \texttt{E1}--\texttt{E3}: early in training and at low diversity the model is better described by \texttt{M1}, whereas later in training and at higher diversity it is better described by \texttt{M2}. The diagonal boundary indicates that greater task diversity induces the shift earlier in training; red regions also coincide with stronger OOD performance (Sec.~\ref{app:id-ood-loss}).
These results indicate that increasing task diversity drives a transition from Bayesian task retrieval toward extrapolative task learning, echoing the results in \cite{raventos2023pretraining, lu2025asymptotic, park2025competition}. Since the minor tasks share a fixed prior mass of $0.1$, the training data per minor task shrinks as $m$ grows, so the transition could reflect either limited data per task or limited model capacity; \Secref{app:capacity} separates the two. Holding the amount of data per minor task fixed as $m$ grows removes the transition over the feasible range, whereas widening the model from $d=32$ to $512$ or training longer only delays it. The transition is thus driven by data per task rather than model capacity, in line with the diversity threshold of \citet{raventos2023pretraining} and the transient generalizing solutions of \citet{carroll2025dynamics}. %

\begin{figure}[t]
    \centering
    
    \includegraphics[width=\linewidth]{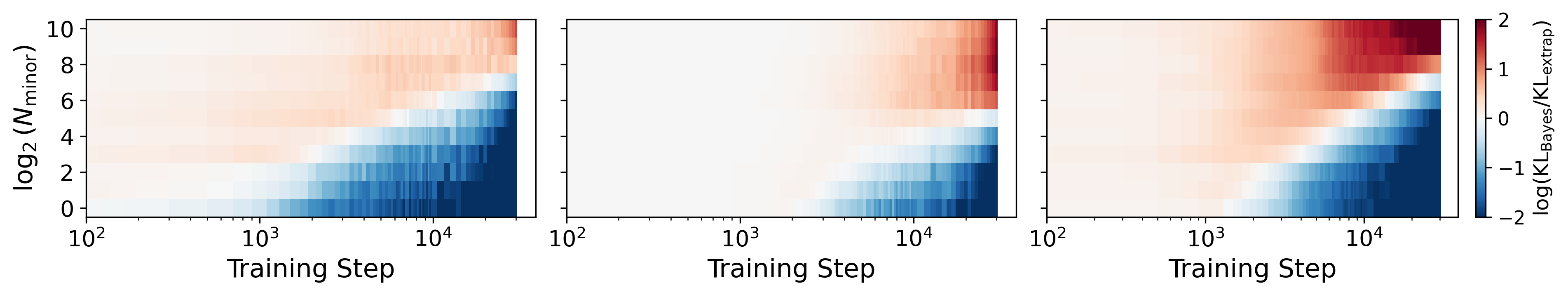}
    \caption{\textbf{Phase transition between two inference modes.} Panels (left to right): \texttt{E1}, \texttt{E2}, \texttt{E3}. 
    {\color{blue} Blue} (negative) indicates that the model is closer to Bayesian task retrieval \texttt{M1}, {\color{red} red} (positive) indicates closer to the extrapolative task learning \texttt{M2}, and white indicates no clear preference. Higher task diversity generally promotes \texttt{M2}. %
    } 
    \vspace{-10pt}
    \label{fig:kl_combined}
\end{figure}

\subsection{Coexisting near-orthogonal subspaces}
\label{sec:representation_geometry}

We hypothesize a stronger geometry for the \texttt{M2} inference mode than Theorem~\ref{thm:task-distance-lower} suggests.

\begin{tcolorbox}
\textbf{Near-orthogonal representation hypothesis}: a transformer internally encodes both inference modes via subspaces, with Bayesian task retrieval carried by the task-vector subspace and extrapolative task learning by a nearly orthogonal subspace. 
\end{tcolorbox}
To test this hypothesis, we train transformers for varying $m$ and extract task vectors as in \Secref{sec:task-vectors}. For averaged hidden states from OOD sequences, we compute the projection $R^2$ onto the major task subspace $\operatorname{col}(\hat{\mTheta})$: $R^2=0$ indicates orthogonality, $R^2=1$ full alignment.

\textbf{Results.}
Fig.~\ref{fig:ood_r2_combined} reports $R^2$ over training for varying $m$. At low diversity $R^2$ remains high: the major task subspace still explains part of the computation on OOD sequences. As $m$ grows, $R^2$ drops consistently across \texttt{E1}--\texttt{E3}, and at large diversity OOD representations are nearly orthogonal to $\operatorname{col}(\hat{\mTheta})$. This separation is \emph{learned}: all curves start at high $R^2$, and the drop emerges progressively only for large $m$ (additional plots in \Secref{app:traj-simplex}). Since a random subspace in high dimension is already nearly orthogonal to any fixed low-dimensional one, a low $R^2$ alone is weak evidence; \Secref{app:principal-angles} therefore compares against a random-subspace null and shows that the OOD-relevant subspace starts inside the task subspace, moves to chance-level overlap as \texttt{M2} develops, and, unlike a random subspace, is necessary for OOD prediction. Comparing Figs.~\ref{fig:kl_combined} and~\ref{fig:ood_r2_combined} reveals the same two-phase dynamic across \texttt{E1}--\texttt{E3}: an undifferentiated phase (no KL preference between \texttt{M1}/\texttt{M2}, uniformly high $R^2$) followed by diversity-driven specialization, in which high-$m$ models move toward \texttt{M2} and $R^2$ drops toward zero; the ID and OOD losses show the same pattern (\Secref{app:id-ood-loss}). This staged pattern echoes the two-phase dynamics of grokking~\citep{wang2024grokked} and attention-head differentiation~\citep{wang2024differentiation}.

\subsection{Pretrained language models}\label{sec:real_llm}

\textbf{Word-function tasks in pretrained LLMs.} We ask whether the same geometry appears in pretrained models, using the 12 word-function tasks of \citet{toddfunction} (e.g., English$\to$French, Antonyms) as ID tasks and, as OOD tasks, variants that apply an unnatural transform to the input word or to the answer. We ran this analysis for all 12 tasks, every layer, and 12 pretrained models from 7 families (0.5B--8B); Fig.~\ref{fig:real_llm_traj} (middle) shows one setting, Qwen2.5-7B at layer 20 with three ID tasks and their OOD variants, next to the \texttt{E3} trajectories (left). In both, ID prompts converge toward the true task vertex as demonstrations accumulate, while OOD prompts stay nearly orthogonal to the task subspace (small $R^2$). All 12 models show the same pattern, with the ID task-subspace $R^2$ peaking at mid-depth; Gemma-3-4B is the exception (Secs.~\ref{app:real_llm}--\ref{app:llm-twelve}).

\textbf{Persona belief dynamics.} \citet{bigelow2025belief} study how an instruction-tuned LLM's belief about the persona it is speaking as evolves with in-context evidence. In our framework, the persona is the latent task, the statements it would make are samples from that task, and the belief over personas is the posterior over tasks. Using their persona data \citep{perez2022discovering}, a prompt lists $N$ statements of one persona, followed by the cue ``Another statement about myself:'', and we take the hidden state of Llama-3.1-8B-Instruct at the cue (layer 6; details in \Secref{app:persona}). Task vectors for $K=3$ personas are built from held-out statements, and each hidden state is fit as a sum-to-one combination of them, whose coefficients we read as the model's belief. As $N$ grows from $0$ to $128$, the coefficient on the context persona rises from $1/3$ to $0.95$ while the others fall to near zero, and the hidden states move from the center of the simplex to the persona's vertex (Fig.~\ref{fig:real_llm_traj}, right). Lacking an oracle posterior, this is qualitative rather than formal evidence for \texttt{P3} (replicated on Qwen2.5-7B-Instruct in \Secref{app:persona}).

\begin{figure}[t]
    \centering
   \includegraphics[width=\linewidth]{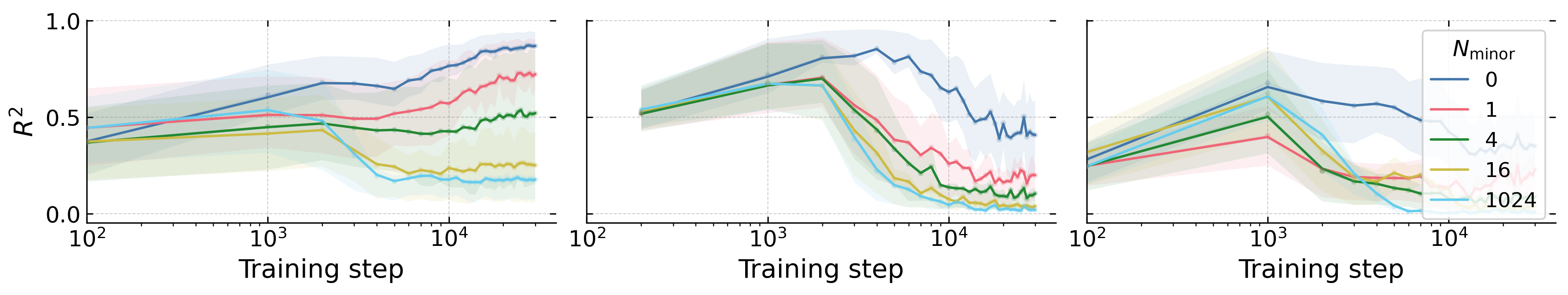}
   \caption{\textbf{Projection $R^2$ of OOD hidden states on major task subspace.} Panels: \texttt{E1}, \texttt{E2}, \texttt{E3}.
    As task diversity increases, projection $R^2$ becomes smaller, supporting the near-orthogonal representation hypothesis.}
    \vspace{-10pt}
    \label{fig:ood_r2_combined}
\end{figure}

\begin{figure}[t]
\centering%
\begin{subfigure}{0.333\textwidth}%
\centering%
\includegraphics[width=\textwidth]{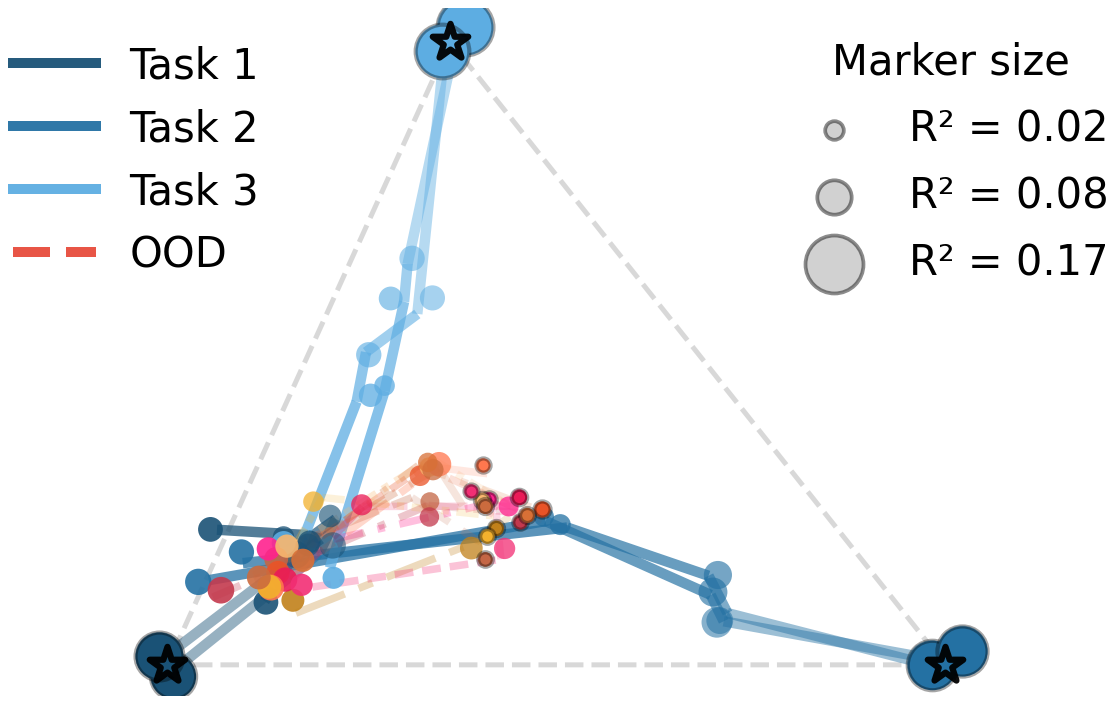}
\end{subfigure}%
\hfill%
\begin{subfigure}{0.333\textwidth}%
\centering%
\includegraphics[width=\textwidth]{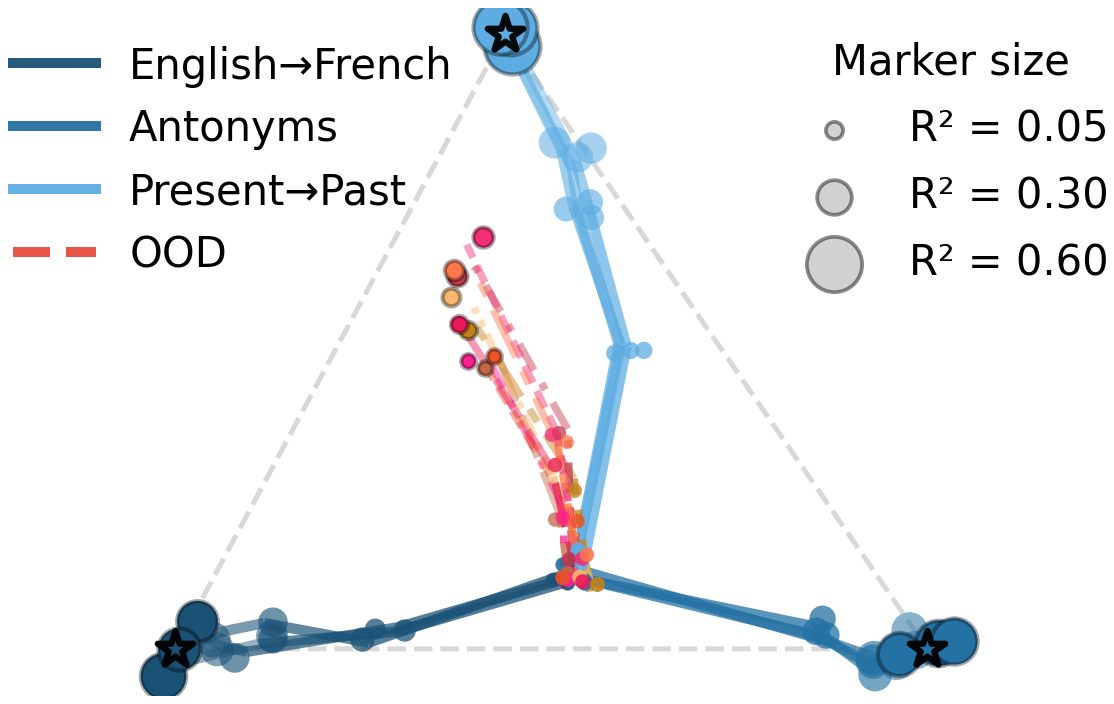}
\end{subfigure}%
\hfill%
\begin{subfigure}{0.333\textwidth}%
\centering%
\includegraphics[width=\textwidth]{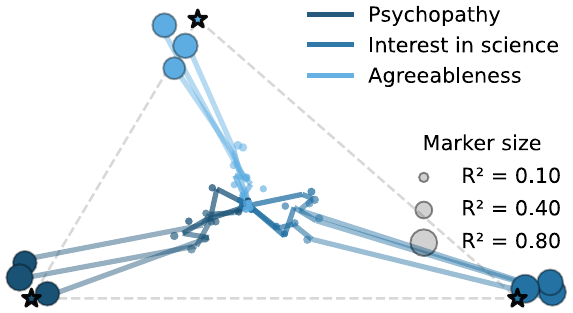}
\end{subfigure}%
\caption{
\textbf{Task-subspace trajectories for \texttt{E3} (left), Qwen2.5-7B word-function tasks (middle), and Llama-3.1-8B-Instruct personas (right).} Task vectors of 3 tasks form a triangle onto which hidden states are projected; marker size encodes the projection $R^2$.
ID prompts: $\vh_t$ converges toward the true task vertex as context grows. OOD prompts (middle): hidden states remain near-orthogonal to the task subspace (small $R^2$). Personas (right): contexts of $N=0$--$128$ persona statements move from the center to the persona's vertex, reaching $R^2 \approx 0.85$.
}
\vspace{-10pt}
\label{fig:real_llm_traj}
\end{figure}

\subsection{Causal interventions on subspaces}
\label{sec:orth_ablation_main}

To validate the causal role of the two subspaces, we intervene on hidden states and measure task-specific losses.
Let $\Delta\gL_{\mathrm{mode}} / g_{\mathrm{mode}}$ denote the fractional loss degradation for $\mathrm{mode} \in \{\mathrm{maj},\mathrm{ood}\}$, where $g_{\mathrm{mode}}$ is the total in-context loss reduction, so that $100\%$ means all ICL for that mode is erased.

\textbf{Task-subspace intervention.}
With $\mP_{\rm task}$ projecting onto $\operatorname{col}(\hat{\mTheta})$,
$\vh\!\leftarrow\!\vh-\gamma\mP_{\rm task}\vh$ selectively degrades major-task performance while largely preserving OOD losses (Table~\ref{tab:interventions}, left), showing that Bayesian task retrieval depends on $\operatorname{col}(\hat{\mTheta})$.

\textbf{Orthogonal subspace intervention.}
Within $\operatorname{col}(\hat{\mTheta})^\perp$ we identify a low-rank subspace $\hat{\mV}_\mathrm{opt}$ via optimization on held-out minor-task data and apply $\vh \leftarrow \vh - \gamma \hat{\mV}_\mathrm{opt}\hat{\mV}_\mathrm{opt}^\top \vh$ ($\gamma>1$).
Table~\ref{tab:interventions} (right) shows large, selective loss increases for OOD tasks, approaching $100\%$ of the OOD in-context gain, while major-task loss and a random same-rank baseline remain near zero.
This double dissociation establishes a near-orthogonal mechanism for extrapolative task learning. It holds across the mid-depth range and is lost only at the last layers, where the shared prediction signal dominates (Secs.~\ref{app:intervention_perlayer} and~\ref{app:principal-angles}).
Sec.~\ref{app:orth_ablation} gives details; Sec.~\ref{app:orth_decomp} shows that $\hat{\mV}_\mathrm{opt}^\top\vh$ linearly encodes context statistics.

\begin{table}[t]
\centering
\scriptsize
\renewcommand{\arraystretch}{0.95}
\setlength{\tabcolsep}{5pt}
\caption{%
\textbf{Causal intervention on subspaces confirms near-orthogonal representations.}
Relative loss degradation $\Delta\gL_{\mathrm{mode}} / g_{\mathrm{mode}} \times 100\%$, averaged over the mid-layer band listed per experiment (criterion in \Secref{sec:emp-task-vector}); per-layer curves in \Secref{app:intervention_perlayer}, frozen-band replication in \Secref{app:frozen-band}.
``Rand.''\ is the mean over $5$ random same-rank directions from $\operatorname{col}(\hat{\mTheta})^{\perp}$.
}
\label{tab:interventions}
\begin{tabular}{l l cc ccc}
\toprule
& &
\multicolumn{2}{c}{Task Subspace Interv.\ (\%)} &
\multicolumn{3}{c}{Orth.\ Subspace Interv. (\%)} \\
\cmidrule(lr){3-4}\cmidrule(lr){5-7}
Exp. & Layers &
\multicolumn{1}{c}{Maj.} & \multicolumn{1}{c}{OOD} & %
\multicolumn{1}{c}{Maj.} & \multicolumn{1}{c}{OOD} & %
\multicolumn{1}{c}{Rand.} \\
\midrule
\texttt{E1} & $1$--$4$  & 211.1 &    12.1 &    %
4.0 &  95.8 & %
2.9 \\
\texttt{E3} & $2$--$4$  & 132.0 &    13.1 &    %
2.1 &  99.6 & %
3.0 \\
\texttt{E2} & $9$--$12$ &  83.6 & $-$1.0  & %
3.0 & 85.4 & %
1.9 \\
\bottomrule
\end{tabular}
\vspace{-10pt}
\end{table}

\section{Non-Markovian case: when hidden states don't summarize context}\label{sec:motivating-example}
\vspace{-2mm}
The previous sections rely on a Markovian premise. %
Here we present a non-Markovian counterexample to show why that premise matters. When long-range structure breaks local summarization, even the commonly assumed property \texttt{P0} fails. Such failure requires dependence of unbounded order: for Markov chains of finite order $w$, the \texttt{M1}--\texttt{M2} transition, \texttt{P0}--\texttt{P3} (with a window of $w$ tokens), and the near-orthogonal dissociation persist up to order $10$ (\Secref{app:order-sweep}), whereas the Dyck language, which has no finite order, lies beyond this range. %
 This demonstrates that the validity of common heuristics depends sensitively on the data-generating assumptions, highlighting the need for a principled framework and the value of its theoretical grounding.

\begin{figure}[!t]
\centering
\includegraphics[width=\linewidth]{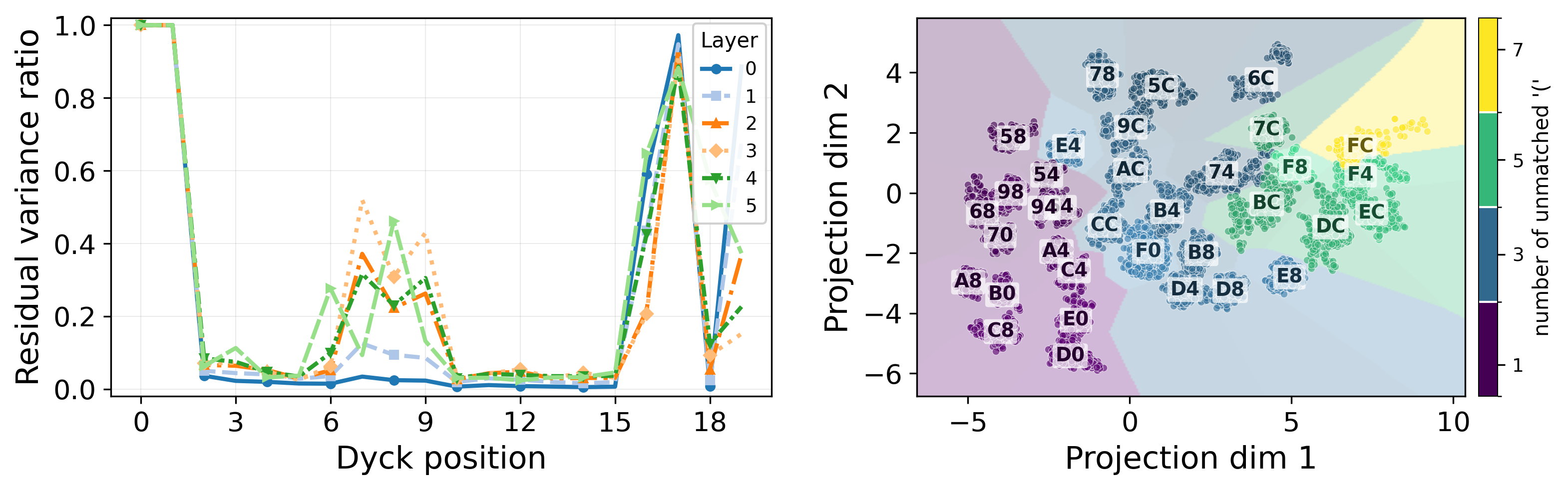}
\caption{\textbf{Planted Dyck language (\texttt{E4}) exhibits long-context dependence and prefix memorization.}
\textbf{\textit{Left}}: residual variance ratio of $\vh_t$ after conditioning on $z$ and the last 3 planted characters at each Dyck position. Unlike \texttt{E1}--\texttt{E3}, a local window leaves substantial residual variance, spiking where long-range bracket matching is required.
\textbf{\textit{Right}}: 2D projection of final-layer hidden states at Dyck prefix length $l=7$ ($35$ prefix classes, labeled by hexadecimal codes; see Fig.~\ref{fig:dyck_proj}). %
} 
\vspace{-15pt}
\label{fig:dyck_combined}
\end{figure}

\textbf{Long-context dependence.} We train a small autoregressive transformer to near-optimal loss and, on held-out sequences, measure the variance of $\vh_t$ conditioned on the latent $z$ and a local window of preceding tokens. In Fig.~\ref{fig:dyck_combined} (left), the residual variance ratio (hidden-state variance \emph{not} explained by $z$ and the last~3 planted characters) at each Dyck position~$j$ (the $j$-th planted bracket) spikes to nearly~$1$ where long-range bracket matching is required and drops to near zero elsewhere. The 2D projection in Fig.~\ref{fig:dyck_combined} (right) further shows that the model encodes the full Dyck prefix at each position rather than a simple summary such as the running open-minus-close count. Conditioning on that count leaves $58$--$84\%$ of the variance unexplained, versus about $3\%$ for the full prefix (\Secref{app:order-sweep}). See App.~\ref{app:dyck} for the prefix-probe construction and validation accuracies.

\section{Limitations and Future Work}

This work presents a modest attempt to bridge generalization and representation geometry in transformers. First, our analysis of \texttt{E1}--\texttt{E3} relies on the Markov property of synthetic data; which natural data admit a compact task statistic remains open. Second, our near-orthogonal subspace does not provide a mechanistic account of OOD generalization, which is complex in general, and its orthogonality lacks a theoretical explanation. Third, our empirical claims are stated for a mid-depth band of layers, and how the task representation is re-encoded at the last layers is left open. Fourth, our pretrained-LLM evidence is observational, and causal validation in pretrained models, including steering along persona task vectors, remains to be done. Finally, our analysis is theory-inspired, but not yet a formal theory of generalization or learning dynamics, left to future work.

\section*{Acknowledgments}
We thank Keith Levin, Karl Rohe, Joshua Cape, Hanbaek Lyu, Sebastien Roch, Robert Nowak, and Liang Twist Shan for helpful discussions, and the anonymous reviewers and area chair for feedback that improved the paper.
Y.Z.~is partially supported by NSF-DMS grant 2412052 and by a Coefficient Giving (formerly Open Philanthropy) grant. The authors declare no competing interests.

\bibliographystyle{plainnat}
\bibliography{refs}

\clearpage

\appendix

\paragraph{Organization of the appendix.}
We organize our appendix into five sections, following the sequential order of the main text. Specifically,
\begin{itemize}
    \item \Secref{app:experiment-details} provides the architecture and training hyperparameters used throughout our synthetic experiments (complementing the Setup in \Secref{sec:setup}).
    \item \Secref{app:theoretical-foundations} contains the theoretical material supporting \Secref{sec:math}: the chain of implications among Properties \texttt{P0}--\texttt{P3} (\Secref{app:property-proofs}), the formal proof of Theorem~\ref{thm:bayes_realization_from_hidden} on Bayesian realizability (\Secref{app:transformer-construct}), explicit formulas for the extrapolative mode \texttt{M2} (\Secref{app:approx-posterior}), and the proof of Theorem~\ref{thm:task-distance-lower} on the simplex packing obstruction (\Secref{sec:simplex-packing-obstruction}).
    \item \Secref{app:task-vec-computation} provides the computational and empirical details behind the evaluation of Properties \texttt{P0}--\texttt{P3} and the simplex intervention reported in \Secref{sec:task-vectors}, together with alternative task-vector estimators (\Secref{app:estimators}) and the vertex-identity controls for the intervention (\Secref{app:vertex-identity}).
    \item \Secref{app:two-modes-evidence} collects additional evidence and details for the two-mode picture in \Secref{sec:two-modes}: the major/minor training mixture, ID/OOD loss dynamics, the capacity versus per-task-data sweeps (\Secref{app:capacity}), simplex trajectories of hidden states, the pretrained-LLM experiments (Qwen2.5-7B, the twelve-model sweep in \Secref{app:llm-twelve}, and the persona belief-dynamics study in \Secref{app:persona}), per-layer breakdowns of the causal interventions, the principal-angle analysis showing that the near-orthogonality of the OOD subspace is learned (\Secref{app:principal-angles}), and the layer-band criterion with a replication of Table~\ref{tab:interventions} using fixed bands (\Secref{app:frozen-band}).
    \item \Secref{app:order-sweep} extends the analysis beyond first-order Markov data: finite-order sweeps of the transition, \texttt{P0}--\texttt{P3}, and the interventions, and the conditional-variance measurement for Dyck.
    \item \Secref{app:dyck-section} reports the prefix-probe analysis of the planted Dyck experiment (\texttt{E4}) discussed in \Secref{sec:motivating-example}.
\end{itemize}

\section{Experiment Details}\label{app:experiment-details}

\label{app:hyperparams}
Code for all experiments, including data generation, training, and the analyses reported in this paper, is available at \url{https://github.com/ezyhdxm/mini-ICL}.

\paragraph{Architecture and training hyperparameters.}
All four experiments use decoder-only transformers with pre-layer-norm, rotary positional encodings, and a $4\times$ MLP expansion ratio, trained autoregressively with AdamW. The biased dice (\texttt{E1}), mixture of Markov chains (\texttt{E3}), and planted Dyck (\texttt{E4}) share a $6$-layer, $128$-dimensional backbone, while the in-context linear regression (\texttt{E2}) uses a deeper $16$-layer model. Per-experiment values for depth, width, context length, optimizer, learning-rate schedule, and other training settings are listed in Tables~\ref{tab:app-arch} and~\ref{tab:app-training}; the resulting checkpoints are the ones reused in every subsequent empirical analysis.

\paragraph{Compute and runtime.}
\label{app:compute-runtime}
All training and analyses are run on a single NVIDIA RTX 3090 (24\,GB). A full $30$\,k--$32$\,k-step training run takes roughly $10$ minutes for \texttt{E1}, $13$ minutes for \texttt{E3} and \texttt{E4}, and $50$ minutes for \texttt{E2} (whose $16$-layer backbone dominates the compute). The two-mode and orthogonal-subspace analyses (\Secref{sec:two-modes} and \Secref{app:two-modes-evidence}) sweep the minor-pool size $N_{\mathrm{minor}} \in \{0, 1, 2, 4, \dots, 1024\}$ and retrain the model for each value, so a single-seed sweep takes about $2$ GPU-hours for \texttt{E1}, $2.5$ for \texttt{E3}, and $9$ for \texttt{E2}. Downstream analyses (task-vector extraction, projection $R^2$, simplex and orthogonal interventions, posterior alignment) all complete in well under one GPU-hour per checkpoint. 

\begin{table}[h]
\centering
\small
\caption{Architecture hyperparameters. All models are decoder-only
transformers with pre-layer-norm and $4\times$ MLP expansion.}
\begin{tabular}{lccccccc}
\toprule
 & Layers & Heads & $d$ & FF dim & $T$ & Pos.\ enc. & Batch \\
\midrule
Biased dice       & 6  & 2 & 128 & 512 & 128 & rotary & 128 \\
Markov chains     & 6  & 2 & 128 & 512 & 192 & rotary & 128 \\
Planted Dyck      & 6  & 2 & 128 & 512 & 192 & rotary & 128 \\
Linear regression & 16 & 2 & 128 & 512 & 128 & rotary & 256 \\
\bottomrule
\end{tabular}
\label{tab:app-arch}
\end{table}

\begin{table}[h]
\centering
\small
\caption{Training hyperparameters. All experiments use AdamW. The
triangle schedule linearly increases the learning rate during warmup and
then linearly decays it to the minimum; the cosine schedule uses cosine
annealing after warmup.}
\begin{tabular}{lccccccc}
\toprule
 & Steps & LR & Weight decay & Warmup & Schedule & Grad clip & Min LR \\
\midrule
Biased dice       & 30\,k & $4\!\times\!10^{-4}$ & $4\!\times\!10^{-4}$ & 15\,k & triangle & --- & $10^{-5}$ \\
Markov chains     & 30\,k & $4\!\times\!10^{-4}$ & $4\!\times\!10^{-4}$ & 15\,k & triangle & --- & $10^{-5}$ \\
Planted Dyck      & 32\,k & $2\!\times\!10^{-4}$ & $2\!\times\!10^{-4}$ & 8\,k  & cosine   & 0.4 & $2\!\times\!10^{-5}$ \\
Linear regression & 30\,k & $2\!\times\!10^{-4}$ & $1\!\times\!10^{-4}$ & 15\,k & triangle & --- & $10^{-5}$ \\
\bottomrule
\end{tabular}
\label{tab:app-training}
\end{table}

\section{Theoretical Foundations of the Property Framework}\label{app:theoretical-foundations}

This section collects the proofs and additional derivations that support the mathematical framework introduced in \Secref{sec:math}.
\begin{itemize}
    \item \Secref{app:property-proofs} establishes the chain of implications among Properties~\texttt{P0}--\texttt{P3}.
    \item \Secref{app:transformer-construct} provides the formal statement and proof of Theorem~\ref{thm:bayes_realization_from_hidden} on Bayesian realizability.
    \item \Secref{app:approx-posterior} derives the closed-form predictors that underlie the extrapolative task-learning mode~\texttt{M2}.
    \item \Secref{sec:simplex-packing-obstruction} proves the simplex packing obstruction (Theorem~\ref{thm:task-distance-lower}).
\end{itemize}

\subsection{Theoretical relationship between properties \texttt{P0}--\texttt{P3}} %
\label{app:property-proofs}

This subsection establishes the implication chain among the four geometric properties \texttt{P0}--\texttt{P3} introduced in \Secref{sec:task_retrieval}: %
\[
\texttt{P3} \;\Rightarrow\; \texttt{P2}, \qquad
\texttt{P3} \;\Rightarrow\; \texttt{P1}\ \text{in the long-context limit}, \qquad
\texttt{P1} \;\Rightarrow\; \texttt{P0}.
\]
This justifies treating the empirical evidence for \texttt{P3} reported in \Secref{sec:emp-task-vector} as evidence for the weaker properties \texttt{P0}--\texttt{P2} as well.

\begin{prop}[\texttt{P3} $\Rightarrow$ \texttt{P2}]
\label{prop:p3-implies-p2}
If \texttt{P3} holds, then \texttt{P2} holds.
\end{prop}
\begin{proof}
By \texttt{P3}, the hidden state admits the decomposition $\vh_t \approx \vmu_t + \sum_{k \le K} \beta_{t,k}\,\vtheta_k + \vnu_{s_t}$ with coefficients $\beta_{t,k} \approx \alpha_{t,k} := \P(z=k \mid s_{\le t})$.
Because $\{\alpha_{t,k}\}_{k \le K}$ is a probability distribution, we have $\alpha_{t,k} \ge 0$ and $\sum_{k \le K} \alpha_{t,k} = 1$, so the coefficients $\beta_{t,k}$ are nonnegative and sum to one.
This is precisely the simplex constraint required by \texttt{P2}.
\end{proof}

\begin{prop}[\texttt{P3} $\Rightarrow$ \texttt{P1} in the long-context limit]
\label{prop:p3-implies-p1}
Suppose \texttt{P3} holds and the posterior concentrates on the true latent, i.e., $\alpha_{t,k} \to \delta_{k,z}$ almost surely as $t \to \infty$, where $z$ denotes the true latent.
Then \texttt{P1} holds in the limit $t \to \infty$.
\end{prop}
\begin{proof}
Substituting $\alpha_{t,k} \to \delta_{k,z}$ into the \texttt{P3} decomposition yields
\[
\vh_t \;\approx\; \vmu_t + \sum_{k \le K} \alpha_{t,k}\,\vtheta_k + \vnu_{s_t}
\;\longrightarrow\;
\vmu_t + \vtheta_z + \vnu_{s_t},
\]
recovering exactly the additive structure required by \texttt{P1}.
The hypothesis $\alpha_{t,k}\to\delta_{k,z}$ a.s.\ holds in our experiments by the ergodic theorem for finite-state Markov chains (\texttt{E1}, \texttt{E3}) and an analogous argument for \texttt{E2} \citep[e.g.,][Ch.~10]{vandervaart1998asymptotic}.
\end{proof}

\begin{prop}[\texttt{P1} $\Rightarrow$ \texttt{P0}]
\label{prop:p1-implies-p0}
If \texttt{P1} holds, then \texttt{P0} holds.
\end{prop}
\begin{proof}
Write $\veps_t = \vh_t - \vmu_{t,z,s_t}$, where $\vmu_{t,z,s_t} = \E[\vh_t \mid z, s_t]$ is the cell mean.
Because the conditional expectation minimizes mean squared error over all functions of $(z, s_t)$, we have
\[
\E\bigl\|\veps_t\bigr\|^2
= \E\bigl\|\vh_t - \vmu_{t,z,s_t}\bigr\|^2
\;\le\; \E\bigl\|\vh_t - g(z,s_t)\bigr\|^2
\]
for every measurable $g\colon [K] \times \gV \to \R^d$.
Taking $g(z,s_t) = \vmu_t + \vtheta_z + \vnu_{s_t}$, which depends on $(z,s_t)$ alone, gives
\[
\E\bigl\|\veps_t\bigr\|^2
\;\le\; \E\bigl\|\vh_t - (\vmu_t + \vtheta_z + \vnu_{s_t})\bigr\|^2
\;\longrightarrow\; 0
\qquad \text{as } t \to \infty,
\]
where the right-hand side vanishes by the \texttt{P1} guarantee (Eq.~\ref{eq:long_context_stability_main}).
\end{proof}

\subsection{Proof of Theorem~\ref{thm:bayes_realization_from_hidden} in \Secref{sec:bayesian-mode} %
}
\label{app:transformer-construct}

This subsection gives the formal statement and proof of Theorem~\ref{thm:bayes_realization_from_hidden} from \Secref{sec:bayesian-mode}. It shows that posterior-aligned representations (properties \texttt{P2}--\texttt{P3}) are sufficient for a transformer to approximate the Bayesian predictive distribution under the task-retrieval mode \texttt{M1}.
The high-level idea is that, once the residual stream encodes the posterior weights $\alpha_{t,k}$ as the coefficients in front of the task vectors $\vtheta_k$, the remaining Bayesian computation
\[
\vq_{t+1}
\;=\;
\sum_{k=1}^K \alpha_{t,k}\,\vp_{k,s_t}
\]
reduces to two operations, each realizable by one transformer block: (i) a finite lookup that mixes the per-pair conditionals $\vp_{k,a}$ into the residual stream, followed by (ii) an approximate coordinatewise logarithm whose softmax recovers $\vq_{t+1}$.
The output unembedding is also part of the construction below, so the theorem does not assume any property of a pre-existing unembedding matrix.
The full construction is given in \Secref{sec:thm-statement-proof} and relies on two technical lemmas stated and proved in \Secref{sec:technical-lemmas}.

\subsubsection{Detailed Statement of Theorem~\ref{thm:bayes_realization_from_hidden} and Proof}
\label{sec:thm-statement-proof}

We now state the formal version of Theorem~\ref{thm:bayes_realization_from_hidden}, an \emph{expressivity} statement.
Posterior-aligned residuals (\texttt{P2}--\texttt{P3}) together with the structural conditions (i)--(iv) below suffice for two appended transformer blocks to approximate the Bayesian predictive $\vq_{t+1}$ within $\ell_\infty$ tolerance $\delta$.
The lookup MLP in Block~1 has hidden width $KV$.
Block~2 has hidden width $m_{\log}(\delta,\varepsilon,V)+2V$, where $m_{\log}$ is the width of a coordinatewise ReLU spline approximation to $\log$ on $[\varepsilon,1]$ at the accuracy used below.
These are constructive upper bounds; we do not claim width optimality, nor do we address whether training finds this configuration or whether conditions (i)--(iv) hold for arbitrary trained models.

\begin{thm}[Internal geometry mirrors Bayesian inference; formal version of Theorem~\ref{thm:bayes_realization_from_hidden}]
Consider a LayerNorm-free transformer with residual-stream dimension~$d$, vocabulary $\gV$ of size $V$ (identified with $[V]$), and training latent set $\gZ_{\mathrm{train}}=[K]$.
Suppose the following conditions hold.
\begin{enumerate}[label=(\roman*)]
\item \textbf{Uniform posterior-aligned representation.}
At some layer~$\ell$, fix a set of positions $\mathcal I$.
For every $t\in\mathcal I$, the residual stream satisfies
\[
\vh_t
=
\vmu_t+\sum_{k=1}^K\alpha_{t,k}\vtheta_k+\vnu_{s_t},
\qquad
\alpha_{t,k}=\P(z=k\mid s_{\le t}),
\]
where $\{\vtheta_k\}_{k=1}^K\subset\R^d$ are task vectors and $\{\vnu_a\}_{a\in\gV}\subset\R^d$ are token-encoding vectors (cf.\ properties \texttt{P2}--\texttt{P3} in \Secref{sec:task_retrieval}).

\item \textbf{Shared centering map.}
Let
\[
\mathcal M_{\mathcal I}:=\operatorname{span}\{\vmu_t:t\in\mathcal I\},
\qquad
\mathcal U:=\operatorname{span}\bigl(\{\vtheta_k\}_{k\le K}\cup\{\vnu_a\}_{a\in\gV}\bigr).
\]
Assume $\mathcal M_{\mathcal I}\cap\mathcal U=\{0\}$.
Equivalently, there exists a fixed linear map $\mC:\R^d\to\R^d$ such that
\[
\mC\vmu_t=0,
\qquad
\mC\vtheta_k=\vtheta_k,
\qquad
\mC\vnu_a=\vnu_a
\]
for all $t\in\mathcal I$, $k\le K$, and $a\in\gV$.
We call such a map a shared centering map.

\item \textbf{Linear independence.}
The family
\[
\{\vtheta_1,\ldots,\vtheta_K\}\cup\{\vnu_a:a\in\gV\}
\]
is linearly independent in $\R^d$; in particular $d\ge K+V$.

\item \textbf{Bounded-away transition probabilities.}
There exists $\varepsilon\in(0,1/V)$ such that
\[
\P(s_{t+1}=b\mid z=k,\,s_t=a)\ge\varepsilon
\]
for all $t\in\mathcal I$, $b\in\gV$, $k\in[K]$, and $a\in\gV$.
\end{enumerate}

Then for every $\delta>0$, there exist an output unembedding $\mW_{\mathrm{U}}\in\R^{V\times d}$ and two additional transformer blocks, all shared across $t\in\mathcal I$, with the blocks appended after layer~$\ell$, such that for every $t\in\mathcal I$ the model's next-token prediction $\hat{\vp}_{t+1}\in\R^V$ satisfies
\[
\|\hat{\vp}_{t+1}-\vq_{t+1}\|_\infty<\delta,
\]
where
\[
\vq_{t+1}:=\sum_{k=1}^K\alpha_{t,k}\vp_{k,s_t},
\qquad
\vp_{k,a}:=\bigl(\P(s_{t+1}=b\mid z=k,\,s_t=a)\bigr)_{b\in\gV}.
\]
Moreover, the first appended MLP has hidden width $KV$, and the second has hidden width $m_{\log}(\delta,\varepsilon,V)+2V$.
\end{thm}

\begin{proof}
We first fix a gauge for the construction.
Although task and token effects are defined earlier in a centered gauge, the lookup construction is most convenient in an uncentered gauge.
In particular, we do not impose $\sum_{k=1}^K\vtheta_k=0$ or $\sum_{a\in\gV}\vnu_a=0$, since either constraint would create a nontrivial linear dependence and conflict with condition~(iii).
All structural assumptions in the theorem, including the shared-centering map $\mC$, are understood in this uncentered gauge.
This is only a change of coordinates at the level of represented states: setting
\[
\bar{\vtheta}:=K^{-1}\sum_{k=1}^K\vtheta_k,
\qquad
\bar{\vnu}:=V^{-1}\sum_{a\in\gV}\vnu_a,
\]
and replacing
\[
(\vmu_t,\vtheta_k,\vnu_a)
\mapsto
(\vmu_t+\bar{\vtheta}+\bar{\vnu},\,\vtheta_k-\bar{\vtheta},\,\vnu_a-\bar{\vnu})
\]
leaves $\vmu_t+\sum_k\alpha_k\vtheta_k+\vnu_a$ unchanged for every $\valpha\in\Delta^{K-1}$ and token $a$.

We append two transformer blocks after layer~$\ell$.
In both blocks, set the attention output projection to zero, so each attention sublayer contributes only the residual identity.
It remains to construct the two MLP sublayers and the output unembedding.

Let $\vx_t:=\vh_t$ be the input to the first appended block and set
\[
\vu_t:=\mC\vx_t.
\]
By conditions~(i)--(ii),
\[
\vu_t
=
\sum_{k=1}^K\alpha_{t,k}\vtheta_k+\vnu_{s_t}
=
\sum_{k=1}^K\alpha_{t,k}(\vtheta_k+\vnu_{s_t}).
\]
Write
\[
\vu_{k,a}:=\vtheta_k+\vnu_a.
\]
Also define the token-embedding matrix
\[
\mE:=(\vnu_1,\ldots,\vnu_V)\in\R^{d\times V}.
\]
By condition~(iii), $\mE$ has full column rank.

\paragraph{Step 1: Build the Bayesian mixture while keeping $\vmu_t$.}
Define
\[
\vf(k,a):=\mE\vp_{k,a}-\vu_{k,a}
=\mE\vp_{k,a}-\vtheta_k-\vnu_a.
\]
By Lemma~\ref{lem:mlp_lookup_convex}, there is a width-$KV$ two-layer ReLU MLP $\texttt{MLP}_{\mathrm{lookup}}:\R^d\to\R^d$ such that, for $\vu_t=\sum_k\alpha_{t,k}\vu_{k,s_t}$,
\[
\texttt{MLP}_{\mathrm{lookup}}(\vu_t)
=
\sum_{k=1}^K\alpha_{t,k}\vf(k,s_t)
=
\mE\sum_{k=1}^K\alpha_{t,k}\vp_{k,s_t}-\vu_t
=
\mE\vq_{t+1}-\vu_t.
\]
Composing the first-layer affine forms with $\mC$, the map $\vx_t\mapsto\texttt{MLP}_{\mathrm{lookup}}(\mC\vx_t)$ is still a two-layer ReLU MLP with hidden width $KV$.
Therefore the first residual block outputs
\[
\vy_t
:=
\vx_t+\texttt{MLP}_{\mathrm{lookup}}(\mC\vx_t)
=
\vmu_t+\vu_t+(\mE\vq_{t+1}-\vu_t)
=
\vmu_t+\mE\vq_{t+1}.
\]
Thus Block~1 does not cancel the nuisance component $\vmu_t$.
It carries $\vmu_t$ forward and only converts the task/token component into the embedded Bayesian predictive $\mE\vq_{t+1}$.

\paragraph{Step 2: Read out $\vq_{t+1}$, cancel only at the logits, and apply softmax.}
Since $\operatorname{col}(\mE)=\operatorname{span}\{\vnu_a:a\in\gV\}\subset\mathcal U$ and $\mathcal M_{\mathcal I}\cap\mathcal U=\{0\}$, we have
\[
\mathcal M_{\mathcal I}\cap\operatorname{col}(\mE)=\{0\}.
\]
By condition~(iv), $\vq_{t+1}\in\Delta_V^\varepsilon$.
Applying Lemma~\ref{lem:approx_recover_prob_subspace} with $\mathcal M=\mathcal M_{\mathcal I}$ gives a constructed unembedding $\mW_{\mathrm{U}}\in\R^{V\times d}$ and a width-$m_{\log}(\delta,\varepsilon,V)+2V$ two-layer ReLU MLP $\texttt{MLP}_{\mathrm{out}}:\R^d\to\R^d$ such that, for $\vy_t=\vmu_t+\mE\vq_{t+1}$,
\[
\left\|
\softmax\!\left(\mW_{\mathrm{U}}\bigl(\vy_t+\texttt{MLP}_{\mathrm{out}}(\vy_t)\bigr)\right)
-
\vq_{t+1}
\right\|_\infty
<\delta.
\]
Using this MLP as the second block's MLP sublayer proves the claim.
All maps used above, namely $\mC$, $\mE$, $\texttt{MLP}_{\mathrm{lookup}}$, $\mW_{\mathrm{U}}$, and $\texttt{MLP}_{\mathrm{out}}$, are fixed across $t\in\mathcal I$, so the construction is uniform over positions.
\end{proof}

\subsubsection{Technical Lemmas and Proofs}
\label{sec:technical-lemmas}

This subsection states and proves the two technical lemmas invoked in the construction of \Secref{sec:thm-statement-proof}: Lemma~\ref{lem:mlp_lookup_convex} is used in the first appended block, and Lemma~\ref{lem:approx_recover_prob_subspace} is used in the second.
Both isolate a single capability of a two-layer ReLU MLP and are stated in self-contained form, independent of the transformer setting.

\begin{lemma}[Width-$KV$ lookup and convex-linearity in the latent index]
\label{lem:mlp_lookup_convex}
Let $K,V,d_{\mathrm{out}}\in\mathbb N$.
Suppose
\[
\{\vtheta_1,\ldots,\vtheta_K,\vnu_1,\ldots,\vnu_V\}\subset\R^d
\]
is linearly independent, and define
\[
\vu_{k,a}:=\vtheta_k+\vnu_a,
\qquad
(k,a)\in[K]\times[V].
\]
For any map $\vf:[K]\times[V]\to\R^{d_{\mathrm{out}}}$, there exists a two-layer ReLU MLP $L:\R^d\to\R^{d_{\mathrm{out}}}$ with hidden width $KV$ such that
\[
L(\vu_{k,a})=\vf(k,a)
\]
for all $(k,a)$, and, more generally,
\[
L\left(\sum_{k=1}^K\alpha_k\vu_{k,a}\right)
=
\sum_{k=1}^K\alpha_k\vf(k,a)
\]
for every fixed $a\in[V]$ and every $\valpha\in\Delta^{K-1}:=\{\valpha\in\R^K:\alpha_k\ge0,\sum_k\alpha_k=1\}$.
\end{lemma}

\begin{proof}
By linear independence, there is a biorthogonal dual family
\[
\{\vtheta_1^*,\ldots,\vtheta_K^*,\vnu_1^*,\ldots,\vnu_V^*\}\subset\R^d
\]
satisfying
\[
\langle\vtheta_i^*,\vtheta_k\rangle=\mathbf 1\{i=k\},
\qquad
\langle\vtheta_i^*,\vnu_a\rangle=0,
\]
and
\[
\langle\vnu_j^*,\vnu_a\rangle=\mathbf 1\{j=a\},
\qquad
\langle\vnu_j^*,\vtheta_k\rangle=0.
\]
For each pair $(i,j)\in[K]\times[V]$, define
\[
\vphi_{ij}:=\vtheta_i^*+\vnu_j^*.
\]
Then
\[
\langle\vphi_{ij},\vu_{k,a}\rangle
=
\mathbf 1\{i=k\}+\mathbf 1\{j=a\}.
\]
Define one hidden unit per pair:
\[
 h_{ij}(\vu):=\sigma\bigl(\langle\vphi_{ij},\vu\rangle-1\bigr),
 \qquad \sigma(r):=\max\{r,0\}.
\]
On the vertices $\vu_{k,a}$, this gives
\[
 h_{ij}(\vu_{k,a})=\mathbf 1\{i=k,\,j=a\}.
\]
Now fix a token $a$ and a simplex vector $\valpha$.
For
\[
\vu_a^{(\valpha)}
:=
\sum_{k=1}^K\alpha_k\vu_{k,a}
=
\sum_{k=1}^K\alpha_k\vtheta_k+\vnu_a,
\]
we have
\[
\langle\vphi_{ij},\vu_a^{(\valpha)}\rangle
=
\alpha_i+\mathbf 1\{j=a\}.
\]
Therefore,
\[
 h_{ij}(\vu_a^{(\valpha)})
=
\begin{cases}
\sigma(\alpha_i)=\alpha_i, & j=a,\\
\sigma(\alpha_i-1)=0, & j\ne a,
\end{cases}
\]
because $0\le\alpha_i\le1$.
Finally define
\[
L(\vu):=\sum_{i=1}^K\sum_{j=1}^V h_{ij}(\vu)\,\vf(i,j).
\]
This is a two-layer ReLU MLP with hidden width $KV$, output weights $\vf(i,j)$, and zero output bias.
The vertex and convex-linearity claims follow from the displayed identities.
\end{proof}

\begin{rmk}[Connection to key--value memory]
The construction above can be interpreted through the lens of key--value memory, as discussed in \citet{geva2021transformer}.
In our setting, the vectors $\{\vphi_{ij}\}$ play the role of keys, which detect whether the input corresponds to a specific pair $(i,j)$ via inner products and thresholding.
The second layer associates each key with a value $\vf(i,j)$, and the MLP output is obtained by aggregating the retrieved values.
\end{rmk}

\begin{lemma}[Probability recovery with a constructed unembedding]
\label{lem:approx_recover_prob_subspace}
Let $\mathcal M\subset\R^d$ be a linear subspace, let $\mE\in\R^{d\times V}$ have full column rank, and assume
\[
\mathcal M\cap\operatorname{col}(\mE)=\{0\}.
\]
For every $\varepsilon\in(0,1/V)$ and $\delta>0$, there exist an output unembedding $\mW_{\mathrm{U}}\in\R^{V\times d}$ and a two-layer ReLU MLP $F:\R^d\to\R^d$ of hidden width $m_{\log}(\delta,\varepsilon,V)+2V$ such that
\[
\sup_{\vm\in\mathcal M,\,\vp\in\Delta_V^\varepsilon}
\left\|
\softmax\!\left(\mW_{\mathrm{U}}\bigl(\vm+\mE\vp+F(\vm+\mE\vp)\bigr)\right)-\vp
\right\|_\infty
<\delta,
\]
where
\[
\Delta_V^\varepsilon
:=
\left\{\vp\in\R^V:p_b\ge\varepsilon\text{ for all }b\in[V],\ \sum_{b=1}^V p_b=1\right\}.
\]
\end{lemma}

\begin{proof}
Because \(M\cap\operatorname{col}(\mE)=\{0\}\), we have
\(M+\operatorname{col}(\mE)=M\oplus\operatorname{col}(\mE)\).
Thus every vector in this sum has a unique decomposition
\(\vm+\mE\vr\) with \(\vm\in M\) and \(\vr\in\mathbb R^V\).
Define a linear map
\(\mP:M\oplus\operatorname{col}(\mE)\to\mathbb R^V\) by
\[
\mP(\vm+\mE\vr)=\vr.
\]
Extend $\mP$ linearly to all of $\R^d$, and set
\[
\mW_{\mathrm{U}}:=\mP.
\]
Then
\[
\mW_{\mathrm{U}}\vm=\vzero\quad(\vm\in\mathcal M),
\qquad
\mW_{\mathrm{U}}\mE=\mI_V.
\]
In particular, the required unembedding is constructed from the decomposition; no pre-existing output head is assumed.

Choose $\eta>0$ such that $e^{2\eta}-1<\delta$.
Let $G:\R^V\to\R^V$ be a two-layer ReLU network satisfying
\[
\sup_{\vp\in\Delta_V^\varepsilon}\|G(\vp)-\log\vp\|_\infty<\eta.
\]
Such a $G$ can be obtained coordinatewise by a piecewise-linear ReLU spline approximation to $\log$ on the compact interval $[\varepsilon,1]$; equivalently, this follows from standard single-hidden-layer approximation results such as \citet{pinkus1999approximation}.
Let $m_{\log}(\delta,\varepsilon,V)$ denote the width of this coordinatewise approximation.

Define
\[
F(\vy):=\mE\,G(\mW_{\mathrm{U}}\vy)-\mE\mW_{\mathrm{U}}\vy.
\]
The first term is a two-layer ReLU MLP of width $m_{\log}(\delta,\varepsilon,V)$, obtained by composing the input layer of $G$ with $\mW_{\mathrm{U}}$ and the output layer with $\mE$.
The second term $-\mE\mW_{\mathrm{U}}\vy$ is a rank-at-most-$V$ linear map, so it can be represented exactly using at most $2V$ ReLU hidden units.
Indeed, if $\mE\mW_{\mathrm{U}}=\mA\mB$ with $\mB\in\R^{r\times d}$, $\mA\in\R^{d\times r}$, and $r\le V$, then
\[
-\mE\mW_{\mathrm{U}}\vy
=
-\mA\bigl(\sigma(\mB\vy)-\sigma(-\mB\vy)\bigr).
\]
Concatenating the hidden units for the two terms gives a two-layer ReLU MLP of width $m_{\log}(\delta,\varepsilon,V)+2V$.

For $\vy=\vm+\mE\vp$ with $\vm\in\mathcal M$ and $\vp\in\Delta_V^\varepsilon$, we have $\mW_{\mathrm{U}}\vy=\vp$.
Therefore
\[
\mW_{\mathrm{U}}\bigl(\vy+F(\vy)\bigr)
=
\mW_{\mathrm{U}}\vy
+
\mW_{\mathrm{U}}\mE\,G(\mW_{\mathrm{U}}\vy)
-
\mW_{\mathrm{U}}\mE\mW_{\mathrm{U}}\vy
=
G(\vp),
\]
where we used $\mW_{\mathrm{U}}\mE=\mI_V$.
The output distribution is therefore $\softmax(G(\vp))$.
Write $G(\vp)=\log\vp+\ve$ with $\|\ve\|_\infty<\eta$.
For each coordinate $b$,
\[
\softmax(G(\vp))_b
=
\frac{p_b e^{e_b}}{\sum_{c=1}^V p_c e^{e_c}}.
\]
Since $e^{-\eta}\le e^{e_c}\le e^\eta$ for every $c$, the ratio $\softmax(G(\vp))_b/p_b$ lies in $[e^{-2\eta},e^{2\eta}]$.
Thus
\[
\left|\softmax(G(\vp))_b-p_b\right|
\le e^{2\eta}-1
<\delta.
\]
Taking the maximum over $b$ proves the lemma.
\end{proof}

\subsection{Details About Extrapolative Task Learning Mode: Explicit Formulas %
}\label{app:approx-posterior}

This subsection derives the closed-form predictors that define the extrapolative task-learning mode \texttt{M2} introduced in \Secref{sec:generalization-mode}.
Concretely, \texttt{M2} is the Bayesian posterior predictive obtained from Eq.~\ref{eq:bayes} by replacing the prior $\pi_{\mathrm{train}}$ over the finite training support $\gZ_{\mathrm{train}}$ with the \emph{generative} prior over the full latent space $\gZ$, i.e.\ the same parametric distribution that was used to sample tasks during training (symmetric Dirichlet for \texttt{E1} and \texttt{E3}, isotropic Gaussian for \texttt{E2}; cf.\ \Secref{sec:setup}).
For each of \texttt{E1}--\texttt{E3}, conjugacy of this prior with the per-token likelihood collapses the integral over $\gZ$ into a closed-form predictor that depends on the observed sequence only through a low-dimensional context statistic: a smoothed empirical unigram for \texttt{E1}, the ridge solution for \texttt{E2}, and a smoothed empirical bigram for \texttt{E3}.
The derivations below thus pin down the concrete predictive targets that a model implementing \texttt{M2} must encode; whether the trained model actually encodes them, and where in the residual stream, is a separate empirical question taken up in the orthogonal-subspace decomposition of \Secref{app:orth_decomp}.

\paragraph{Biased dice (\texttt{E1}).}
Each latent task $z = \vp \in \gZ = \Delta^{V-1}$ is sampled from the symmetric Dirichlet prior
$\mathrm{Dir}(\boldsymbol{1}_V)$.
More generally, for $\vp \in \Delta^{V-1}$ drawn from $\mathrm{Dir}(\alpha_0 \boldsymbol{1}_V)$, the probability density function is
$$
f\left(p_1, \ldots, p_V ; \alpha_0\right)=\frac{\Gamma(\alpha_0 V)}{\Gamma(\alpha_0)^V} \prod_{a=1}^V p_a^{\alpha_0-1}.
$$
Recall the token count $n_a(t)$ defined in Equation~\eqref{eq:token-count}.
The posterior predictive after observing $s_{\le t}$ is
\[
\begin{aligned}
    \P(s_{t+1}=a \mid s_{\le t}) 
    &= \int_{\vp \in \gZ} \P(s_{t+1}=a \mid \vp, s_{\le t})\, \P(\vp \mid s_{\le t})\, d\vp \\
    &= \frac{\int_{\vp \in \gZ} \P(s_{t+1}=a \mid \vp)\, \P(s_{\le t} \mid \vp)\, \P(\vp)\, d\vp}{\int_{\vp \in \gZ} \P(s_{\le t} \mid \vp)\, \P(\vp)\, d\vp} \\
    &= \frac{\int_{\vp \in \gZ}  p_a \prod_{b=1}^V p_{b}^{n_b(t) + \alpha_0 - 1}\, d\vp}{\int_{\vp \in \gZ} \prod_{b=1}^V p_{b}^{n_b(t) + \alpha_0 - 1}\, d\vp} \\
    &= \frac{\Gamma(n_a(t)+\alpha_0+1) \prod_{b\neq a}^V \Gamma(n_b(t)+\alpha_0)}{\Gamma(t + V\alpha_0 + 1)} \cdot \frac{\Gamma(t + V\alpha_0)}{\prod_{b=1}^V \Gamma(n_b(t)+\alpha_0)}\\
    &=\frac{n_a(t)+\alpha_0}{t + V\alpha_0},
\end{aligned}
\]
which is a smoothed empirical unigram estimator.

\paragraph{Noisy linear regression (\texttt{E2}).}
Each latent task $z = \vw \in \gZ = \R^D$ is drawn from an isotropic Gaussian prior
\[
\vw \sim \mathcal{N}(\vzero,\, \tau^{-2}\sigma^2 \mI_D),
\]
which matches the generative prior of \Secref{sec:setup} when $\tau^2 = \sigma^2$.
Conditional on $\vw$, inputs are i.i.d.\ $\vx_t \sim \mathcal{N}(\vzero, \mI_D)$ and
$y_t \mid \vx_t, \vw \sim \mathcal{N}(\vx_t^\top \vw,\, \sigma^2)$.
For the prediction of $y_{t+1}$ at a fresh covariate $\vx_{t+1}$, let
$\mX_{\le t} = (\vx_1,\ldots,\vx_t)^\top \in \R^{t\times D}$ and
$\vy_{\le t} = (y_1,\ldots,y_t)^\top \in \R^t$ collect the observed pairs.
By Gaussian--Gaussian conjugacy, the posterior over $\vw$ is again Gaussian:
\[
\vw \mid \mX_{\le t}, \vy_{\le t}
\;\sim\;
\mathcal{N}\!\left(\hat{\vw}_{\mathrm{ridge}}(t),\; \sigma^2\,\mSigma_t\right),
\qquad
\mSigma_t := \bigl(\mX_{\le t}^\top \mX_{\le t} + \tau^2 \mI_D\bigr)^{-1},
\]
where $\hat{\vw}_{\mathrm{ridge}}(t)$ is the ridge solution \citep{xieexplanation,raventos2023pretraining,akyurek2022learning}
\[
\hat{\vw}_{\mathrm{ridge}}(t)
\;=\;
\bigl(\mX_{\le t}^\top \mX_{\le t} + \tau^2 \mI_D\bigr)^{-1} \mX_{\le t}^\top \vy_{\le t}.
\]
Marginalising $\vw$ yields the posterior predictive
\[
\begin{aligned}
\P\bigl(y_{t+1} \mid \vx_{t+1}, \mX_{\le t}, \vy_{\le t}\bigr)
&= \int_{\R^D} \P(y_{t+1} \mid \vx_{t+1}, \vw)\, \P(\vw \mid \mX_{\le t}, \vy_{\le t})\, d\vw \\
&= \mathcal{N}\!\left(\vx_{t+1}^\top \hat{\vw}_{\mathrm{ridge}}(t),\;\, \sigma^2\bigl(1 + \vx_{t+1}^\top \mSigma_t\, \vx_{t+1}\bigr)\right).
\end{aligned}
\]
Hence the posterior predictive mean is the ridge regressor evaluated at $\vx_{t+1}$,
and the predictive variance inflates the noise level $\sigma^2$ by the
input-dependent factor $1 + \vx_{t+1}^\top \mSigma_t \vx_{t+1}$ that captures
remaining uncertainty about $\vw$.

\paragraph{Mixture of Markov chains (\texttt{E3}).}
Each latent task $z = \mP \in \gZ = (\Delta^{V-1})^V$ is a row-stochastic
$V \times V$ transition matrix, with each row drawn independently from the
symmetric Dirichlet prior used to generate the training tasks (\Secref{sec:setup}):
\[
\mP_{a, \cdot} \;\stackrel{\text{i.i.d.}}{\sim}\; \mathrm{Dir}(\alpha_0 \boldsymbol{1}_V),
\qquad a = 1, \ldots, V.
\]
For a Markov sequence $s_{\le t}$ generated under $\mP$ from a uniform initial
distribution, let
\[
N_{ab}(t) \;:=\; \sum_{\tau = 2}^{t} \mathbf{1}\{s_{\tau-1} = a,\, s_{\tau} = b\}
\;\in\; \mathbb{N},
\qquad
N_{a\cdot}(t) \;:=\; \sum_{b=1}^{V} N_{ab}(t)
\;\in\; \mathbb{N},
\]
denote the count of $a \to b$ transitions and the row total, respectively.
The likelihood factorises across rows,
\[
\P(s_{\le t} \mid \mP) \;=\; \prod_{a=1}^{V} \prod_{b=1}^{V} P_{a,b}^{N_{ab}(t)},
\]
which is a product of multinomial likelihoods. Combined with the
row-independent Dirichlet prior, Dirichlet--multinomial conjugacy gives a
posterior that factorises across rows,
\[
\mP_{a, \cdot} \mid s_{\le t}
\;\sim\;
\mathrm{Dir}\!\left(\alpha_0 + N_{a1}(t),\, \ldots,\, \alpha_0 + N_{aV}(t)\right),
\qquad a = 1, \ldots, V.
\]
The posterior predictive at the next position, given current state $s_t = a$,
is then
\[
\begin{aligned}
\P\bigl(s_{t+1} = b \mid s_t = a,\, s_{\le t}\bigr)
&= \int_{\gZ} P_{a, b}\, \P(\mP \mid s_{\le t})\, d\mP \\
&= \E\!\left[ P_{a, b} \mid s_{\le t} \right] \\
&= \frac{N_{ab}(t) + \alpha_0}{N_{a\cdot}(t) + V\alpha_0},
\end{aligned}
\]
where the last equality uses the mean of a Dirichlet distribution.
This is a smoothed empirical bigram estimator, exactly analogous to the
smoothed empirical unigram derived for \texttt{E1}: it pools every observed
transition out of state $a$ and interpolates with the uniform prior at rate
controlled by~$\alpha_0$.

\paragraph{Summary.}
In all three experiments, the predictor underlying \texttt{M2} is the exact
Bayesian posterior predictive under the parametric prior used to generate the
training tasks, but evaluated over the \emph{full} latent space $\gZ$ rather
than only the finite training support $\gZ_{\mathrm{train}}$. Conjugacy of
each prior--likelihood pair (Dirichlet--categorical for \texttt{E1},
Gaussian--Gaussian for \texttt{E2}, and row-wise Dirichlet--multinomial for
\texttt{E3}) yields a closed-form predictor that depends on the observed
sequence only through a sufficient context statistic: the empirical unigram
counts $n_a(t)$ for \texttt{E1}, the ridge solution $\hat{\vw}_{\mathrm{ridge}}(t)$
for \texttt{E2}, and the empirical bigram counts $N_{ab}(t)$ for \texttt{E3}.
This explains why these context statistics arise naturally as the targets that
the model must encode in order to implement \texttt{M2}.

\subsection{Proof of Theorem~\ref{thm:task-distance-lower} in Section~\ref{sec:generalization-mode}
}
\label{sec:simplex-packing-obstruction}

This subsection states and proves the formal version of Theorem~\ref{thm:task-distance-lower}.
The result makes precise the intuition that the extrapolative task-learning mode \texttt{M2} cannot be implemented from inside a task-vector subspace whose dimension $k_\star$ is smaller than the intrinsic dimension $d_0$ of the optimal prediction map: a volumetric packing argument forbids the required separation of latents within a $k_\star$-dimensional affine slab unless the alignment error $\varepsilon$ grows accordingly.
Theorem~\ref{thm:finite-context-packing-clean} below gives the explicit lower bound on $\varepsilon$, which informally takes the form $\varepsilon \gtrsim (c-\delta)/L - C/(2^{d_0/k_\star}-1)$ reported in \Secref{sec:generalization-mode}.

\begin{rmk}[Scope of Theorem~\ref{thm:task-distance-lower}]\label{rmk:thm2-scope}
The theorem is a packing obstruction: the extrapolative predictive family has intrinsic dimension $d_0 \gg k_\star$, so a Lipschitz decoder near the task simplex cannot realize that many predictions, and \texttt{M2} representations are forced away from the task subspace. However, it only establishes a separation result, without forcing the geometry of the exit to be orthogonal. \Secref{sec:two-modes} shows empirically that the exit is near-orthogonal and that this geometry is learned during training rather than inherited from initialization. \Secref{app:principal-angles} offers a hypothetical explanation: under limited shared dimensions, near-orthogonal placement minimizes interference between retrieval and extrapolation \citep{welch1974lower, elhage2022superposition}.
\end{rmk}

\paragraph{Notation.}
\label{par:standing-notation}
We collect the notation used throughout this appendix, extending the conventions of \Secref{sec:math}.
Let $\gZ$ denote the latent space with training set $\gZ_{\mathrm{train}}=[K]$.
At a fixed layer, $\vh_t\in\R^d$ is the hidden state at position~$t$;
$\{\vtheta_k\}_{k\le K}$ and $\{\vnu_a\}_{a\in\gV}$ are the task vectors and token-encoding vectors (Definition~\ref{def:task-vec});
$\vmu_{t,k,a}:=\E[\vh_t\mid z{=}k,\,s_t{=}a]$ are the cell means (Eq.~\ref{eq:cell-means}).
We assume \texttt{P0} together with the existence of the limiting cell mean
\(\vmu(\vz,a):=\lim_{t\to\infty}\vmu_{t,\vz,a}\) for each \((\vz,a)\) in the
neighborhood of interest. (Note that \texttt{P0} alone controls only the
conditional residual variance \(\E\|\vh_t-\vmu_{t,z,s_t}\|^2\) and does not by
itself guarantee convergence of \(\vmu_{t,\vz,a}\); we therefore impose the
existence of the limit explicitly.) Under these assumptions the (long-context)
decoder maps \(\vmu(\vz,a)\) to the Bayes-optimal next-token distribution
\(\P(\cdot\mid \vz,a)\).

\paragraph{Setup.}
We now introduce the symbols specific to the simplex-confinement question, augmenting the notation of Sec.~\ref{sec:generalization-mode}. Fix an anchor token \(a_\star\in\gV\) and write
\[
\vp(\vz):=\P(\cdot \mid \vz,a_\star) \in \Delta^{V-1},
\qquad
\vmu_\star(\vz):=\vmu(\vz,a_\star)\in \R^d,
\]
for the conditional next-token distribution and the limiting cell mean at \(a_\star\). Let
\[
\calT:=\operatorname{span}(\vtheta_2-\vtheta_1,\dots,\vtheta_K-\vtheta_1),
\qquad
\calS:=\vnu_{a_\star}+\calT,
\]
denote the task-vector subspace and its affine translate at \(\vnu_{a_\star}\), and set \(k_\star:=\dim(\calT)\le K-1\). We assume \(k_\star\ge 1\), as otherwise \(\calS\) is a singleton and the conclusions are trivial.

\paragraph{Discrete-output scope.}
The theorem is stated for finite discrete output spaces: \(a_\star \in \mathcal V\) and
\(\vp(z)=\P(\cdot \mid z,a_\star)\in\Delta^{V-1}\). It therefore directly covers the discrete-output settings \texttt{E1} and \texttt{E3}. The linear-regression setting E2 is an empirical continuous-output analogue. A formal continuous-output version would require replacing the simplex-valued map \(\vp\) by an appropriate notion of distance between posterior-predictive regression distributions. We omit this extension for simplicity.

\paragraph{Gauge convention} The affine anchor in \(\calS=\vnu_{a_\star}+\calT\) is a gauge choice. The proof below uses only the affine set \(\calS\) and distances \(\operatorname{dist}(\vmu_\star(z), \calS)\), not the separate decomposition of its anchor into global-mean and token components. Indeed, if \texttt{P1} is written in the centered gauge
\[
    \vh_t \approx \vmu_t+\vtheta_z+\vnu_{s_t}
\]
and \(\vmu_t\to\vmu_\infty\), then the long-context ID cell means at anchor token
\(a_\star\) lie in
\[
    \vmu_\infty+\vnu_{a_\star}+\calT .
\]
We absorb this common translate into the affine anchor, equivalently writing
\[
    \bar{\vnu}_{a_\star}:=\vmu_\infty+\vnu_{a_\star},
    \qquad
    \calS=\bar{\vnu}_{a_\star}+T,
\]
and then drop the bar. This is only a notational convention for interpreting \(\calS\)
under \texttt{P1}; Theorem~\ref{thm:task-distance-lower} itself assumes only the limiting cell means \(\vmu(z,a)\) specified above.

\begin{defn}[Local regularity setup]
\label{def:local-regularity}
Fix a base point \(\vz_0\in \operatorname{int}(\gZ)\) and assume that \(\vp\) is \(C^1\) on a neighborhood of \(\vz_0\). Let
\[
\mJ_0:=\nabla_{\vz}\vp(\vz_0),
\qquad
d_0:=\operatorname{rank}(\mJ_0)\ge 1,
\qquad
\sigma_0:=s_{d_0}(\mJ_0)>0,
\qquad
\mathcal E:=(\ker \mJ_0)^\perp,
\]
where \(s_{d_0}(\mJ_0)\) denotes the smallest nonzero singular value of \(\mJ_0\). Choose a radius \(r_0>0\) with \(B(\vz_0,r_0)\subseteq \gZ\), and define
\[
U_{\mathcal E,0}:=(\vz_0+\mathcal E)\cap B(\vz_0,r_0).
\]
Then \(U_{\mathcal E,0}\) is a \(d_0\)-dimensional ball in the affine subspace \(\vz_0+\mathcal E\) along which \(\vp\) has full local rank; the integer \(d_0\) makes precise the ``number of independent directions along which \(\vp(\vz)\) varies'' from Sec.~\ref{sec:generalization-mode}.
\end{defn}

The next theorem formalizes the informal claim of Theorem~\ref{thm:task-distance-lower} in Sec.~\ref{sec:generalization-mode}. Throughout, we write
\[
\operatorname{dist}(\vx,\calS) := \inf_{\vc \in \calS} \|\vx - \vc\|_2
\]
for the Euclidean distance from a point \(\vx\in\R^d\) to the affine set \(\calS\), and we collect the regularity assumptions used in the theorem.

\begin{defn}[Local decoder regularity]
\label{def:local-decoder-regularity}
Let \(\hat{\vp}(\vz)\) be the predictor obtained by passing \(\vmu_\star(\vz)\) through the model's decoder. We say the setup is \emph{\((M,L,\delta)\)-regular on a neighborhood} \(\mathcal U\subseteq\gZ\) of \(\vz_0\) if there exist constants \(M<\infty\), \(L>0\), and \(\delta\ge 0\) such that, for all \(\vz,\vz'\in \mathcal U\),
\begin{align}
\|\vmu_\star(\vz)-\vmu_\star(\vz')\|_2 &\le M,
\label{eq:local-mu-diameter}\\
\|\hat{\vp}(\vz)-\hat{\vp}(\vz')\|_2
&\le
L\,\|\vmu_\star(\vz)-\vmu_\star(\vz')\|_2,
\label{eq:local-decoder-lipschitz}\\
\|\hat{\vp}(\vz)-\vp(\vz)\|_2 &\le \delta.
\label{eq:local-prediction-error}
\end{align}
We further define the worst-case distance from \(\vmu_\star\) to the affine task subspace on \(\mathcal U\),
\begin{equation}
 \varepsilon := \sup_{\vz\in\mathcal U}\operatorname{dist}\bigl(\vmu_\star(\vz),\calS\bigr).
\label{eq:local-simplex-distance}
\end{equation}
\end{defn}

\begin{thm}[Finite-context packing obstruction]
\label{thm:finite-context-packing-clean}
Assume the setup of Definition~\ref{def:local-regularity} and that the
\((M,L,\delta)\)-regularity of Definition~\ref{def:local-decoder-regularity}
holds on a neighborhood \(\mathcal U\subseteq\gZ\) of \(\vz_0\).
Then there exists \(r_{\mathrm{loc}}\in(0,r_0]\) such that
\[
U_{\mathcal E}:=(\vz_0+\mathcal E)\cap B(\vz_0,r_{\mathrm{loc}})
\subseteq \mathcal U,
\qquad
\alpha:=\sigma_0/2,
\]
and the following hold.

\smallskip\noindent
\textbf{(a) Local co-Lipschitz estimate for the ground-truth map.}
For all \(\vz,\vz'\in U_{\mathcal E}\),
\begin{equation}
\label{eq:finite-direct-colip}
\|\vp(\vz)-\vp(\vz')\|_2
\ge
\alpha\,\|\vz-\vz'\|_2.
\end{equation}

\smallskip\noindent
\textbf{(b) Finite-context packing inequality.}
For any \(0<\rho<r_{\mathrm{loc}}\) such that
\[
\frac{\alpha\rho-2\delta}{L}>2\varepsilon,
\]
one has
\begin{equation}
\label{eq:finite-direct-packing}
\left(\frac{r_{\mathrm{loc}}}{\rho}\right)^{d_0}
\le
\left(
1+\frac{2M}{
\sqrt{
\left(\dfrac{\alpha\rho-2\delta}{L}\right)^2-4\varepsilon^2
}}
\right)^{k_\star}.
\end{equation}

\smallskip\noindent
\textbf{(c) Lower bound on the confinement radius.}
For any \(0<\rho<r_{\mathrm{loc}}\) such that \(\alpha\rho>2\delta\),
\begin{equation}
\label{eq:finite-direct-eps-lower}
\varepsilon
\ge
\frac{1}{2}
\left[
\left(\frac{\alpha\rho-2\delta}{L}\right)^2
-
\left(
\frac{2M}{
\left(r_{\mathrm{loc}}/\rho\right)^{d_0/k_\star}-1
}
\right)^2
\right]_+^{1/2},
\end{equation}
where \([x]_+:=\max\{x,0\}\).
\end{thm}

\begin{proof}
Choose a linear isometry
\[
\mW:\R^{d_0}\to \mathcal E.
\]
Define
\[
\widetilde{\vp}(\vu):=\vp(\vz_0+\mW\vu),
\qquad
\vu\in B_{\R^{d_0}}(\vzero,r_0).
\]
Then \(\widetilde{\vp}\) is \(C^1\), and
\[
\nabla \widetilde{\vp}(\vzero)=\mJ_0\mW.
\]
Since \(\mW\) is an isometry onto \(\mathcal E=(\ker \mJ_0)^\perp\), the smallest singular value of
\(\nabla \widetilde{\vp}(\vzero)\) equals \(\sigma_0\). By continuity of
\(\nabla \widetilde{\vp}\) at \(\vzero\), after shrinking \(r_{\mathrm{loc}}\in(0,r_0]\) if necessary,
we may assume both that \(U_{\mathcal E}\subseteq \mathcal U\) and that
\[
\sup_{\|\vu\|_2<r_{\mathrm{loc}}}
\|\nabla \widetilde{\vp}(\vu)-\nabla \widetilde{\vp}(\vzero)\|_{\mathrm{op}}
\le
\frac{\sigma_0}{2}.
\]
Now fix \(\vu,\vv\in B_{\R^{d_0}}(\vzero,r_{\mathrm{loc}})\). By the fundamental theorem of calculus,
\[
\widetilde{\vp}(\vu)-\widetilde{\vp}(\vv)
=
\nabla \widetilde{\vp}(\vzero)(\vu-\vv)
+
\int_0^1
\Bigl(
\nabla \widetilde{\vp}\bigl(\vv+\tau(\vu-\vv)\bigr)-\nabla \widetilde{\vp}(\vzero)
\Bigr)(\vu-\vv)\,d\tau.
\]
Therefore
\begin{align*}
\|\widetilde{\vp}(\vu)-\widetilde{\vp}(\vv)\|_2
&\ge
\|\nabla \widetilde{\vp}(\vzero)(\vu-\vv)\|_2
-
\frac{\sigma_0}{2}\|\vu-\vv\|_2 \\
&\ge
\frac{\sigma_0}{2}\|\vu-\vv\|_2.
\end{align*}
Transporting this estimate back through \(\mW\) gives \eqref{eq:finite-direct-colip}. This proves part~(a).

Fix \(0<\rho<r_{\mathrm{loc}}\), and let
\[
\{\vz_1,\dots,\vz_N\}\subset U_{\mathcal E}
\]
be a maximal \(\rho\)-separated subset. Since the set is maximal,
\[
U_{\mathcal E}
\subseteq
\bigcup_{i=1}^N \bigl((\vz_0+\mathcal E)\cap B(\vz_i,\rho)\bigr).
\]
Taking \(d_0\)-dimensional Euclidean volumes inside the affine space \(\vz_0+\mathcal E\), the standard ball-volume comparison (e.g.\ Corollary 4.2.13 in \cite{vershynin2018high}) yields
\begin{equation}
\label{eq:finite-direct-N-lb}
N\ge \left(\frac{r_{\mathrm{loc}}}{\rho}\right)^{d_0}.
\end{equation}

For \(i\neq j\), \eqref{eq:finite-direct-colip} and
\eqref{eq:local-prediction-error} give
\begin{align*}
\|\hat{\vp}(\vz_i)-\hat{\vp}(\vz_j)\|_2
&\ge
\|\vp(\vz_i)-\vp(\vz_j)\|_2
-
\|\hat{\vp}(\vz_i)-\vp(\vz_i)\|_2
-
\|\hat{\vp}(\vz_j)-\vp(\vz_j)\|_2 \\
&\ge
\alpha \|\vz_i-\vz_j\|_2 - 2\delta \\
&\ge
\alpha\rho-2\delta.
\end{align*}
Using \eqref{eq:local-decoder-lipschitz}, we obtain
\[
\|\vmu_\star(\vz_i)-\vmu_\star(\vz_j)\|_2
\ge
\lambda(\rho),
\qquad
\lambda(\rho):=\frac{\alpha\rho-2\delta}{L}.
\]

Let \(\Pi_{\calS}\) be the orthogonal projection onto \(\calS\), and set
\[
\vy_i:=\Pi_{\calS}\vmu_\star(\vz_i)\in\calS,
\qquad
\vu_i:=\vmu_\star(\vz_i)-\vy_i\in\calT^\perp.
\]
By \eqref{eq:local-simplex-distance},
\[
\|\vu_i\|_2=\operatorname{dist}\!\bigl(\vmu_\star(\vz_i),\calS\bigr)\le \varepsilon.
\]
Since \(\vy_i-\vy_j\in\calT\) and \(\vu_i-\vu_j\in\calT^\perp\), the Pythagorean theorem gives
\[
\|\vy_i-\vy_j\|_2^2
=
\|\vmu_\star(\vz_i)-\vmu_\star(\vz_j)\|_2^2
-
\|\vu_i-\vu_j\|_2^2
\ge
\lambda(\rho)^2-4\varepsilon^2.
\]
If \(\lambda(\rho)>2\varepsilon\), the points \(\vy_1,\dots,\vy_N\) are therefore \(\Delta(\rho):=\sqrt{\lambda(\rho)^2-4\varepsilon^2}\)-separated in~\(\calS\).

Setting \(\vc_0:=\Pi_{\calS}\vmu_\star(\vz_0)\), the \(1\)-Lipschitz property of \(\Pi_{\calS}\) and \eqref{eq:local-mu-diameter} give \(\|\vy_i-\vc_0\|_2\le M\), so
\[
\{\vy_1,\dots,\vy_N\}\subseteq \calS\cap B(\vc_0,M).
\]

Assume now that \(\lambda(\rho)>2\varepsilon\). The balls \(\calS\cap B(\vy_i,\Delta(\rho)/2)\), \(i=1,\dots,N\), are pairwise disjoint and contained in \(\calS\cap B(\vc_0,M+\Delta(\rho)/2)\). Comparing \(k_\star\)-dimensional Euclidean ball volumes once more (Corollary 4.2.13 in \cite{vershynin2018high}) gives
\[
N\le\bigl(1+2M/\Delta(\rho)\bigr)^{k_\star}.
\]
Combining with \eqref{eq:finite-direct-N-lb} yields \eqref{eq:finite-direct-packing}, proving part~(b).

Finally, fix \(0<\rho<r_{\mathrm{loc}}\) with \(\alpha\rho>2\delta\), and define
\[
A(\rho):=
\frac{2M}{
\left(r_{\mathrm{loc}}/\rho\right)^{d_0/k_\star}-1
}.
\]
If \(\lambda(\rho)^2\le A(\rho)^2\), then \eqref{eq:finite-direct-eps-lower} is immediate,
since its right-hand side is zero. Assume therefore that \(\lambda(\rho)^2>A(\rho)^2\).
If, toward a contradiction, \eqref{eq:finite-direct-eps-lower} failed, then we would have
\[
4\varepsilon^2<\lambda(\rho)^2-A(\rho)^2.
\]
In particular, \(\lambda(\rho)>2\varepsilon\), so part~(b) applies and yields
\[
\left(\frac{r_{\mathrm{loc}}}{\rho}\right)^{d_0/k_\star}
\le
1+\frac{2M}{\sqrt{\lambda(\rho)^2-4\varepsilon^2}}.
\]
Rearranging,
\[
\sqrt{\lambda(\rho)^2-4\varepsilon^2}
\le
\frac{2M}{
\left(r_{\mathrm{loc}}/\rho\right)^{d_0/k_\star}-1
}
=
A(\rho),
\]
which contradicts \(4\varepsilon^2<\lambda(\rho)^2-A(\rho)^2\). Hence \(4\varepsilon^2\ge\lambda(\rho)^2-A(\rho)^2\), i.e.,
\[
\varepsilon
\ge
\frac12\bigl[\lambda(\rho)^2-A(\rho)^2\bigr]_+^{1/2},
\]
which is \eqref{eq:finite-direct-eps-lower}. This proves part~(c).
\end{proof}

\paragraph{From Theorem~\ref{thm:finite-context-packing-clean} to the informal Theorem~\ref{thm:task-distance-lower}.}
The informal bound stated in Sec.~\ref{sec:generalization-mode} is the special case
of part~(c) obtained by fixing the inner radius to a constant fraction of the
outer radius and discarding the squares.

\begin{cor}[Distance-to-task-subspace lower bound; formal form of Theorem~\ref{thm:task-distance-lower}]
\label{cor:task-distance-lower-formal}
Under the hypotheses of Theorem~\ref{thm:finite-context-packing-clean}, choose
\(\rho:=r_{\mathrm{loc}}/2\) and assume \(\alpha\rho>2\delta\), where
\(\alpha=\sigma_0/2\) is the local co-Lipschitz constant from
Theorem~\ref{thm:finite-context-packing-clean}. Then
\begin{equation}
\label{eq:informal-form-bound}
\sup_{\vz\in \mathcal U}\operatorname{dist}\bigl(\vmu_\star(\vz),\calS\bigr)
\;\ge\;
\frac{c-\delta}{L}\;-\;\frac{C}{2^{d_0/k_\star}-1},
\end{equation}
with
\(c:=\alpha r_{\mathrm{loc}}/4\) and \(C:=M\).
\end{cor}

\begin{proof}
With \(\rho=r_{\mathrm{loc}}/2\) one has \((r_{\mathrm{loc}}/\rho)^{d_0/k_\star}=2^{d_0/k_\star}\), and part~(c) of Theorem~\ref{thm:finite-context-packing-clean} reduces to
\[
\varepsilon
\;\ge\;
\frac{1}{2}
\Bigl[\,A^2-B^2\,\Bigr]_+^{1/2},
\qquad
A:=\frac{\alpha r_{\mathrm{loc}}/2-2\delta}{L},
\quad
B:=\frac{2M}{2^{d_0/k_\star}-1}.
\]
Since \(\sqrt{A^2-B^2}\ge A-B\) whenever \(A\ge B\ge 0\) (as \(A+B\ge A-B\)), we obtain
\[
\varepsilon
\;\ge\;
\tfrac12 [A-B]_+
\;\ge\;
\frac{\alpha r_{\mathrm{loc}}/4-\delta}{L}-\frac{M}{2^{d_0/k_\star}-1}.
\]
Identifying the right-hand side with \(\frac{c-\delta}{L}-\frac{C}{2^{d_0/k_\star}-1}\) yields \eqref{eq:informal-form-bound}.
\end{proof}

\noindent
Corollary~\ref{cor:task-distance-lower-formal} thus instantiates the informal
Theorem~\ref{thm:task-distance-lower} of the main text: the local constants
\(c\) and \(C\) depend on the local co-Lipschitz constant \(\alpha\), the
confinement radius \(r_{\mathrm{loc}}\), and the local diameter \(M\), but not
on the latent dimensions \(d_0\) or \(k_\star\). The decay in
\(2^{d_0/k_\star}\) is therefore the genuine source of the obstruction: when
\(d_0\gg k_\star\), the second term in \eqref{eq:informal-form-bound} is
negligible, leaving a distance-to-\(\calS\) lower bound of order
\((c-\delta)/L\). Theorem~\ref{thm:finite-context-packing-clean} is strictly
stronger, since it provides a packing inequality that holds at every scale
\(\rho<r_{\mathrm{loc}}\) and quantifies the interaction between the prediction
error \(\delta\), the decoder Lipschitz constant \(L\), and the confinement
radius \(\varepsilon\).

\section{Empirical Validation of Task-Vector Properties}\label{app:task-vec-computation}

This section provides the computational and empirical details behind the evaluation of the task-vector properties \texttt{P0}--\texttt{P3}.
\begin{itemize}
    \item \Secref{sec:extracting-hiddens} describes the task-vector extraction procedure used throughout our analysis.
    \item \Secref{app:additive-separability} reports the empirical evidence for the long-context properties \texttt{P0}--\texttt{P1}, complemented by an OLS probe decomposition for experiment~\texttt{E3} in \Secref{app:ols-probe-e3}.
    \item \Secref{app:finite-context-properties} provides the computational and empirical details for the finite-context properties \texttt{P2}--\texttt{P3}.
    \item \Secref{app:alpha_injection} details the simplex-intervention experiment of \Secref{sec:task-vec:interv}.
\end{itemize}

\subsection{Extracting and Estimating Task Vectors}
\label{sec:extracting-hiddens}

This subsection records the precise procedure used to extract residual-stream hidden states from the trained transformers and to estimate the task vectors $\hat{\vtheta}_k$ that play the role of simplex anchors throughout \Secref{sec:task-vectors}.
The extraction recipe and averaging-based estimator below are reused without modification by every empirical analysis in \Secref{app:task-vec-computation} (variance and additivity tests, simplex-projection coefficients $\beta_{t,k}$, simplex interventions) and by the orthogonal-subspace interventions in \Secref{app:two-modes-evidence}.
We write $\mH^{(\ell)} = [\vh_1^{(\ell)}, \dots, \vh_T^{(\ell)}]^\top$ for the
collection of residual stream representations across positions.
We consider a Pre-LayerNorm (LN) decoder-only transformer, where each layer updates a residual stream $\mH^{(\ell-1)} \in \R^{T \times d}$ via
\begin{align*}
\widetilde{\mH}^{(\ell)}
&= \mH^{(\ell-1)} + \texttt{MHA}\bigl(\texttt{LN}(\mH^{(\ell-1)})\bigr),
\\
\mH^{(\ell)}
&= \widetilde{\mH}^{(\ell)} + \texttt{MLP}\bigl(\texttt{LN}(\widetilde{\mH}^{(\ell)})\bigr).
\end{align*}

We extract representations from the residual stream after the full layer
update, i.e., from $\mH^{(\ell)}$, and for each position $t$ collect the
layer-wise trajectory $\{\vh_t^{(\ell)}\}_{\ell=0}^{L-1}$.

\paragraph{Extraction positions.}
For discrete-token tasks (\texttt{E1}, \texttt{E3}, \texttt{E4}), the input
sequence is $[s_1, \dots, s_T]$ and we extract at all positions, $t \in \{1, \dots, T\}$. For the linear regression task
(\texttt{E2}), the sequence alternates between inputs and targets,
$[\vx_1, y_1, \vx_2, y_2, \dots, \vx_{T/2}, y_{T/2}]$; we extract only at
input positions $\vx_t$, which are the prediction locations for $y_t$.

Having specified the residual-stream extraction recipe above, we now describe how the task-vector estimates $\hat{\vtheta}_k$ themselves are computed from the extracted hidden states.
This averaging-based estimator produces every $\hat{\vtheta}_k$ used in the empirical evaluations of \Secref{app:additive-separability}--\Secref{app:posterior-alignment-layers} and in the simplex-/orthogonal-subspace interventions of \Secref{app:alpha_injection} and \Secref{app:intervention_perlayer}; it is faithful to Definition~\ref{def:task-vec} in the long-context limit.

\paragraph{Estimation by averaging.} \label{app:task-vec-extract}
We estimate task vectors directly from the definition in~Eq.~\ref{eq:task-vec:defn}.
For each layer $\ell$ and a window of late context positions $\mathcal{T}_{\mathrm{est}}$, we pool hidden states across positions and sequences for each task $k$:
\[
\hat{\vtheta}_k
\;=\;
\underbrace{\frac{1}{|\mathcal{T}_{\mathrm{est}}|\,B}
  \sum_{t \in \mathcal{T}_{\mathrm{est}}}\sum_{i:\,z_i=k} \vh_{t,i}}_{\text{task-conditional mean}}
\;-\;
\underbrace{\frac{1}{K}\sum_{k'=1}^{K}\!\Bigl(\frac{1}{|\mathcal{T}_{\mathrm{est}}|\,B}
  \sum_{t \in \mathcal{T}_{\mathrm{est}}}\sum_{i:\,z_i=k'} \vh_{t,i}\Bigr)}_{\text{grand mean}},
\]
where the sums over $i$ run over all $B$ sequences assigned to each task.
By construction $\sum_{k} \hat{\vtheta}_k = \vzero$.
$\mathcal{T}_{\mathrm{est}}$ is a contiguous window of $\approx 15$--$30$ \emph{late} context positions, chosen so that the Bayesian posterior $\P(z\mid s_{\le t})$ has largely concentrated and $\E[\vh_t\mid z=k]$ approximates its long-context limit (Eq.~\ref{eq:task-vec:defn}); empirically, this is the regime where \texttt{P0}--\texttt{P1} hold to high accuracy (\Secref{app:additive-separability}, Figure~\ref{fig:task_vector_r2}).

\subsection{Long-context stability and decoupling \texttt{P0-P1}}
\label{app:additive-separability}

This subsection provides the empirical details behind the long-context properties \texttt{P0}--\texttt{P1}: we report the residual variance ratio across context positions (\Secref{app:residual-variance}), showing that task identity and the current token account for most of the variance in hidden states, and we test whether the conditional mean decomposes additively via ANOVA\,/\,ANCOVA (\Secref{app:additive-sep-test}).

\subsubsection{Residual variance ratio}\label{app:residual-variance}

We measure long-context stability (\texttt{P0}) and decoupling (\texttt{P1}) jointly through the residual variance ratio $\widehat{\mathrm{Var}}(\vh_t \mid z, s_t) / \widehat{\mathrm{Var}}(\vh_t)$, the fraction of hidden-state variance that is \emph{not} explained by the task identity $z$ together with the current token $s_t$.
A small ratio at large $t$ indicates that the task and current token jointly account for almost all variation in $\vh_t$, which is the empirical content of \texttt{P0}--\texttt{P1}.
This complements Table~\ref{tab:p0_p1_combined}(a) in the main text (which only reports the value at the last position) by displaying the full layer-by-layer trajectory across context positions for \texttt{E1}--\texttt{E3}.

\begin{figure}[h]
\centering
\includegraphics[width=0.9\linewidth]{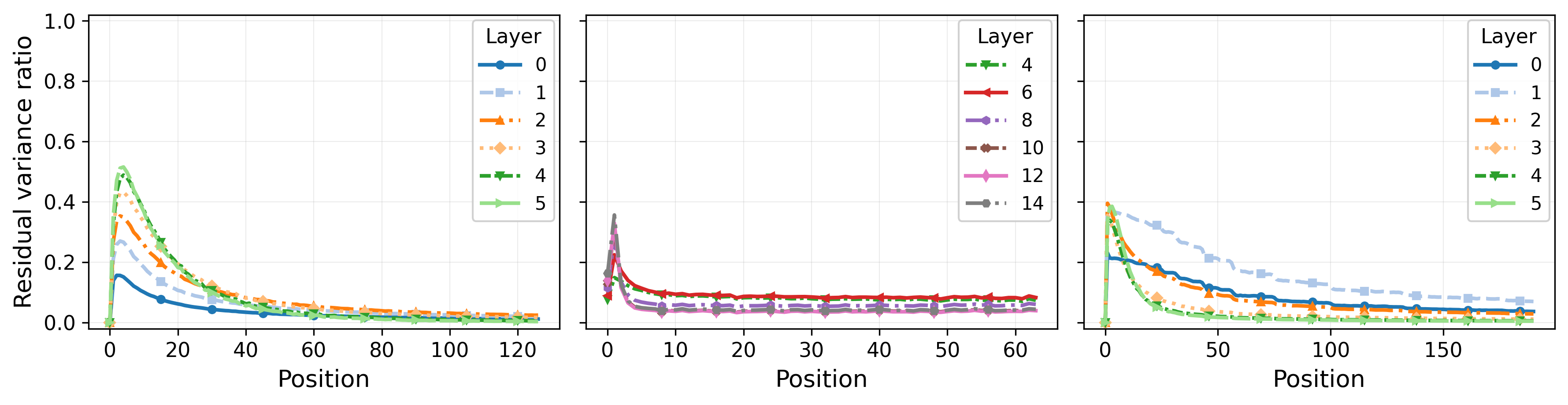}
\caption{\textbf{Latent task and last token explain most variance in hidden states at large $t$.}
Residual variance ratio $\mathrm{SS}_{\mathrm{within}} / \mathrm{SS}_{\mathrm{total}} = \widehat{\mathrm{Var}}(\vh_t \mid z, s_t) / \widehat{\mathrm{Var}}(\vh_t)$ as a function of context position for \texttt{E1} (left), \texttt{E2} (middle), and \texttt{E3} (right).
Each curve corresponds to one layer (post-MLP).
The ratio decreases as context grows (the model gradually infers the latent) and is smaller in later layers (depth aids task inference).}
\label{fig:task_vector_r2}
\end{figure}

Figure~\ref{fig:task_vector_r2} reports the full residual-ratio curves across context positions; Table~\ref{tab:p0_p1_combined}(a) (main text) summarizes the values at the last position.

\subsubsection{Testing additive separability}\label{app:additive-sep-test}

We test whether the conditional mean $\hat \vmu_{t,k,a} := \hat \E[\vh_t \mid z=k, s_t=a]$ decomposes additively as $\hat \vmu_{t,k,a} \approx \hat \vmu_t + \hat \vtheta_k + \hat \vnu_a$, where the task effect $\hat \vtheta_k$ and token effect $\hat \vnu_a$ do not interact. The key idea is to compare an \emph{additive} model against a \emph{full} model that includes a task--token interaction term.

Fix a layer $\ell$ and position $t$. For experiments \texttt{E1} and \texttt{E3} with discrete tokens $a \in \gV$, the additive model is
\[
\vh_t \;=\; \vmu_t + \vtheta_k + \vnu_a,
\]
while the full model allows a separate mean for each $(k, a)$ cell. This is a standard two-way ANOVA. A two-way decomposition of the between-cell variance gives
\[
\underbrace{\sum_{k,a} \|\hat \vmu_{t,k,a} - \hat \vmu_t\|^2}_{\mathrm{SS}_{\mathrm{between}}}
\;=\;
\underbrace{V\sum_k \|\hat \vtheta_k\|^2}_{\mathrm{SS}_{\mathrm{task}}}
\;+\;
\underbrace{K \sum_a \|\hat \vnu_a\|^2}_{\mathrm{SS}_{\mathrm{token}}}
\;+\;
\underbrace{\sum_{k,a}\|\hat \vmu_{t,k,a} - \hat \vmu_t - \hat \vtheta_k - \hat \vnu_a\|^2}_{\mathrm{SS}_{\mathrm{interaction}}},
\]
where $\hat \vtheta_k = (1/V) \sum_a \hat \vmu_{t,k,a} - \hat \vmu_t$
and $\hat \vnu_a = (1/K) \sum_k \hat \vmu_{t,k,a} - \hat \vmu_t$.

For experiment \texttt{E2}, where the current token $\vx_t \in \R^{D}$ is a continuous covariate, the additive model is
\[
\vh_t \;=\; \vmu_t + \vtheta_k + \mB\, \vx_t,
\]
where $\mB$ is a shared slope matrix, while the full model replaces $\mB$ with task-specific slopes $\mB_k$. Both models are fit by ordinary least squares (ANCOVA).

In both cases, we measure the \emph{interaction proportion}: the fraction of explained variance attributable to the interaction term. For discrete tokens this is $\eta^2_{\mathrm{interaction}} = \mathrm{SS}_{\mathrm{interaction}} / \mathrm{SS}_{\mathrm{between}}$. For continuous covariates, we define the analogous quantity $\eta^2_{\mathrm{interaction}} = (R^2_{\mathrm{full}} - R^2_{\mathrm{additive}}) / R^2_{\mathrm{full}}$, i.e.\ the fraction of the full model's explanatory power that is due to task-specific slopes. A small $\eta^2_{\mathrm{interaction}}$ confirms that the additive model is accurate.

\begin{figure}[h]
\centering
\includegraphics[width=\linewidth]{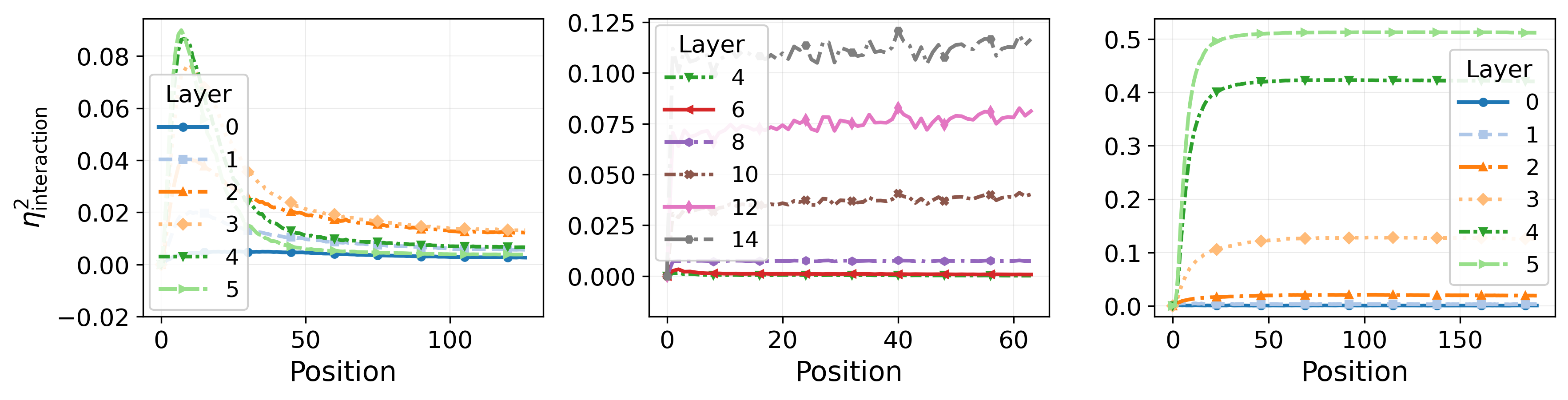}
\caption{\textbf{Additive separability of task and token effects.} All three panels show the interaction proportion $\eta^2_{\mathrm{interaction}}$: the fraction of explained variance due to the task--token interaction term. Left (\texttt{E1}): two-way ANOVA on discrete tokens. Middle (\texttt{E2}): ANCOVA on continuous covariates, $\eta^2_{\mathrm{interaction}} = (R^2_{\mathrm{full}} - R^2_{\mathrm{additive}}) / R^2_{\mathrm{full}}$. Right (\texttt{E3}): two-way ANOVA on discrete tokens. Small values confirm that the additive model is accurate. Additivity holds broadly, with significant departures in late layers of \texttt{E2} and \texttt{E3}.}
\label{fig:additive_sep}
\end{figure}

\subsection{OLS probe decomposition for experiment \texttt{E3}}
\label{app:ols-probe-e3}

This subsection complements the aggregate analysis in \Secref{sec:emp-task-vector} with a layerwise variance decomposition of the hidden states in \texttt{E3}.
The simplex-projected interpolation $R^2$ (Figure~\ref{fig:averaging_r2}) plateaus at $0.60$ and $0.52$ in the last two layers, and the ANOVA interaction proportion (Table~\ref{tab:p0_p1_combined}(b)) flags a parallel breakdown of additivity, but neither metric identifies what the residual interaction encodes.
We regress each hidden state on three factors (the task label, the current token, and the model's own output logits) and use the logits as a diagnostic for residual-stream directions aligned with the next-token distribution.
Including the logits absorbs essentially all of the late-layer interaction, indicating that the apparent breakdown of additivity is \emph{not a genuine task--token interaction} but reflects hidden-state components devoted to computing the next-token prediction.
The decomposition further shows that task identity is most cleanly represented in middle layers and is progressively subsumed into the next-token signal in later layers, providing a quantitative account of why task vectors are best extracted from middle layers~\citep{hendel2023context}.

\paragraph{Probe construction.}
We regress the residual-stream representation $\vh_t$ on three feature groups computed from the same forward pass: the task label $z$, the current token $s_t$, and the model's output logit $\hat{\boldsymbol{\ell}}_t$.
The first two are the variables that property~\texttt{P1} predicts to enter $\vh_t$ additively, so any contribution to the late-layer interaction must be absorbed by the third.
We use $\hat{\boldsymbol{\ell}}_t$ because it is precisely the direction of the residual stream that the unembedding reads off, and therefore serves as a known summary of the next-token-prediction signal carried by $\vh_t$.
Before fitting we subtract the per-position mean across sequences, $\bar{\vh}_t \coloneqq \vh_t - \bar{\vh}_{\cdot,t}$, to suppress position-dependent variation induced by the rotary positional code.
The OLS model is
\begin{equation}
  \bar{\vh}_t
  \;=\;
  \mW_z\,\tilde{\ve}_z
  \;+\;
  \mW_s\,\tilde{\ve}_{s_t}
  \;+\;
  \mW_{\mathrm{logit}}\,\hat{\boldsymbol{\ell}}_t
  \;+\;
  \vb
  \;+\;
  \vepsilon_t,
  \label{eq:ols-probe}
\end{equation}
where $\tilde{\ve}_z\in\{0,1\}^{K-1}$ and $\tilde{\ve}_{s_t}\in\{0,1\}^{V-1}$ are the one-hot encodings of $z\in[K]$ and $s_t$ with the last column dropped (to avoid collinearity with the intercept), $\hat{\boldsymbol{\ell}}_t\in\R^V$ is the output logit vector, and $\vb\in\R^d$ is an intercept.
We fit on $2^{13}$ sequences from the fully trained model for experiment \texttt{E3} ($K{=}3$ major tasks, no minor tasks), evaluated at positions $t\in[170,190)$.

\paragraph{Marginal and partial \texorpdfstring{$R^2$}{R2}.}
For a feature group $A$, the \emph{marginal $R^2$} is the $R^2$ from regressing $\bar{\vh}_t$ on $A$ alone (with intercept), and the \emph{partial $R^2$}
\begin{equation}
  \Delta R^2_A
  \;=\;
  \frac{R^2_{\mathrm{joint}} - R^2_{-A}}{1 - R^2_{-A}}
  \label{eq:partial-r2}
\end{equation}
quantifies the unique contribution of $A$, where $R^2_{-A}$ is the $R^2$ of the reduced model obtained from Eq.~\ref{eq:ols-probe} by dropping the regressors in $A$ while keeping the intercept and the other two groups.
Reporting both is necessary because the regressors overlap: $\hat{\boldsymbol{\ell}}_t$ is a deterministic linear function of the deepest hidden state, so it shares variance with both the task label ($R^2_{\mathrm{task}\leftrightarrow\mathrm{logit}}{=}0.36$) and the current token ($R^2_{\mathrm{token}\leftrightarrow\mathrm{logit}}{=}0.35$); the task and token indicators are nearly orthogonal ($R^2_{\mathrm{task}\leftrightarrow\mathrm{token}}{=}0.028$).
The partial $R^2$'s for task and logits should therefore be read as contributions conditional on the linear span of the other groups.

\begin{figure}[H]
\centering
\includegraphics[width=0.9\linewidth]{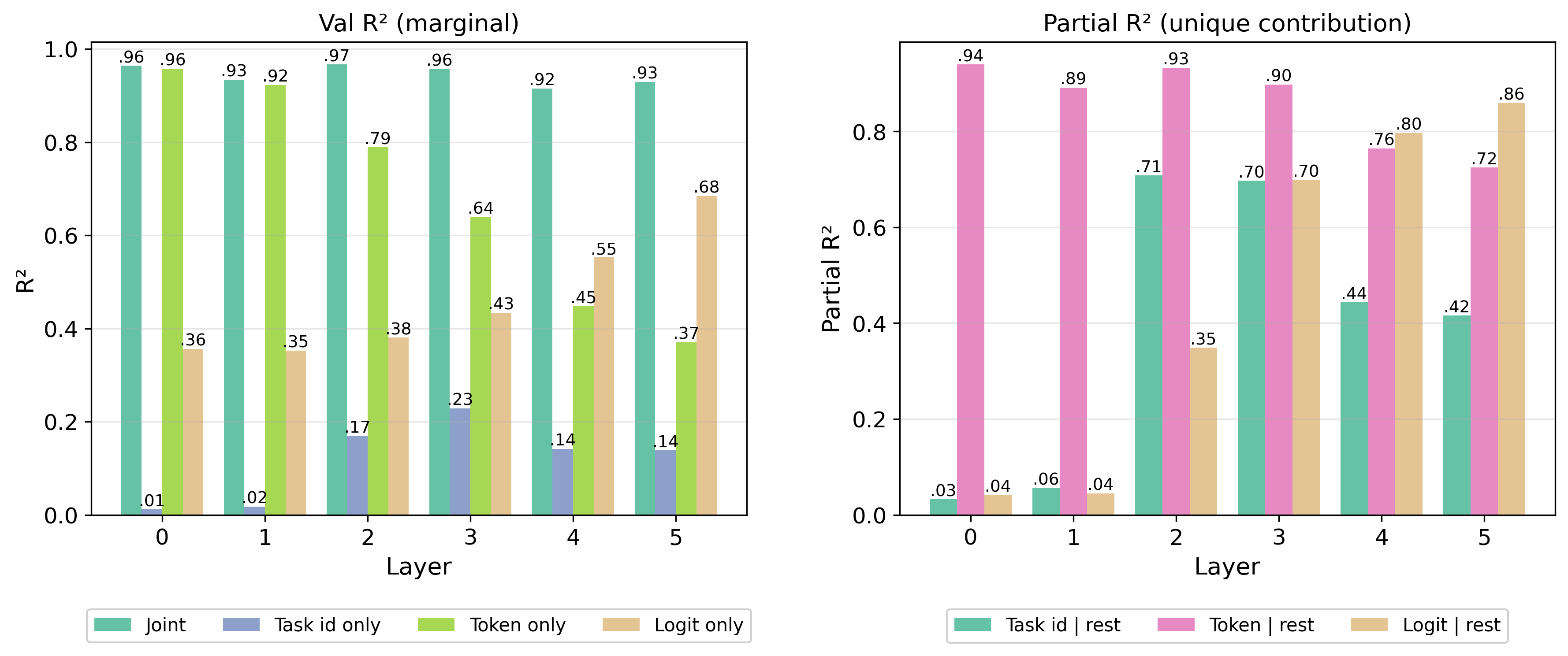}
\caption{
\textbf{OLS probe $R^2$ decomposition for experiment \texttt{E3}.}
\emph{Left}: marginal $R^2$ of each feature group alone.
\emph{Right}: partial $R^2$ (Eq.~\ref{eq:partial-r2}) after controlling for the other two groups.
Results are shown layer-by-layer (post-MLP, layers~$0$--$5$) at context positions $t\in[170,190)$, with per-position mean subtraction applied to $\vh_t$ before fitting.
}
\label{fig:ols-probe-e3}
\end{figure}

\paragraph{Layerwise findings.}
The decomposition pinpoints the source of the late-layer non-additivity: the next-token-prediction signal $\hat{\boldsymbol{\ell}}_t$ \emph{progressively dominates the residual stream at deeper layers}, absorbing both task and token information and thereby driving the apparent breakdown of additivity flagged by Table~\ref{tab:p0_p1_combined}(b) and Figure~\ref{fig:averaging_r2}.
As a corollary, the task label is encoded as a near-independent direction only at middle layers (2--3), while the current token retains a dominant unique contribution at all depths.
Quantitatively, the joint model attains $R^2_{\mathrm{joint}}\in[0.92,0.97]$ at every layer, so the three groups together give a near-complete linear account of $\bar{\vh}_t$.
The next-token-prediction signal $\hat{\boldsymbol{\ell}}_t$ rises monotonically with depth: $\Delta R^2 \approx 0.04$--$0.05$ at layers~0--1, jumping to $0.36$ at layer~2 and reaching $0.86$ at layer~5.
The task label has negligible unique contribution at layers~0--1 ($\Delta R^2 \approx 0.03$--$0.05$), peaks at layers~2--3 ($\Delta R^2 \approx 0.70$--$0.71$), and decays to $\sim\!0.42$ by layer~5.
The current token carries the largest unique contribution at all depths (partial $R^2 \approx 0.89$--$0.94$ through layer~3, declining to $\sim\!0.73$ at layer~5), consistent with property~\texttt{P1}.
The mechanism is now transparent: as $\vh_t$ approaches the unembedding, $\hat{\boldsymbol{\ell}}_t$ absorbs information about both $z$ and $s_t$, the task label becomes redundant given $\hat{\boldsymbol{\ell}}_t$, and the additive interpolation model of Eq.~\ref{eq:additive_interp_main} (which has no access to $\hat{\boldsymbol{\ell}}_t$) loses fitting capacity, which both explains the late-layer non-additivity and identifies layers~2--3 as the natural depth for task-vector extraction~\citep{hendel2023context}.

\subsection{Evaluating Finite-Context Properties (\texttt{P2}--\texttt{P3})}\label{app:finite-context-properties}

This subsection provides the computational and empirical details behind the finite-context properties \texttt{P2}--\texttt{P3}: closed-form expressions for the oracle Bayesian posterior $\alpha_{t,k}$ in each experimental setting (\Secref{app:posterior-computation}), the extraction of the simplex-constrained projection coefficients $\beta_{t,k}$ from hidden states (\Secref{app:lambda-computation}), and additional posterior-alignment results across layers and with unconstrained coefficients (\Secref{app:posterior-alignment-layers}).

\subsubsection{Bayesian Posterior Computation under \texttt{M1}}\label{app:posterior-computation}

This subsubsection derives the closed-form expressions for the oracle Bayesian posterior $\alpha_{t,k} := \P(z = k \mid s_1, \ldots, s_t)$ that serve as the comparison target throughout the paper.
Under the task-retrieval mode \texttt{M1} (\Secref{sec:bayesian-mode}) we take the prior to be uniform over the finite training support $\gZ_{\mathrm{train}} = [K]$, so that the integral in Eq.~\ref{eq:bayes} reduces to a finite sum that admits a closed form for each of \texttt{E1}--\texttt{E3}.
The resulting $\alpha_{t,k}$ play two roles in our analyses: they are the targets against which the empirical simplex coefficients $\beta_{t,k}$ are compared to assess Property~\texttt{P3} (\Secref{sec:emp-task-vector}, Figure~\ref{fig:beta_alpha_trajectory}; \Secref{app:posterior-alignment-layers}), and they define the \texttt{M1} predictive distribution used in the KL phase-transition analysis of \Secref{sec:two-mode}.
Throughout, we use the task index convention to refer to the latent variables for simplicity.

\paragraph{Biased Dice (\texttt{E1})} %

Recall that each task $z = k$ is parameterized by a probability vector
$\vp_k = (p_{k,1}, \ldots, p_{k,V}) \in \Delta^{V-1}$ over $V$ outcomes
(the ``faces of the die'').  Tokens $s_1, s_2, \ldots$ are drawn i.i.d.\
from $\mathrm{Categorical}(\vp_k)$.
The posterior at position $t$ factorizes as
\[
\P(z = k \mid s_{\le t})
\;=\; \frac{\P\left(s_{\le t} \mid z=k\right)}{\sum_{k=1}^K \P\left(s_{\le t} \mid z=k\right)} = \frac{\prod_{\tau=1}^t p_{k, s_{\tau}}}{\sum_{k'=1}^K \prod_{\tau=1}^t p_{k', s_{\tau}}}. 
\]
Equivalently, let 
\begin{equation}
\label{eq:token-count}
    n_a(t) := \sum_{\tau=1}^t \mathbf{1}\{s_\tau = a\}
\end{equation}
be the count of token $a$ in the first $t$ tokens, the log-likelihood is
\[
\log \P(s_{\le t} \mid z = k)
\;=\;
\sum_{a=1}^{V} n_a(t)\, \log p_{k,a}.
\]
The posterior is obtained by normalizing across all $K$ tasks:
\begin{equation}
\alpha_{t,k}
\;=\;
\frac{ \prod_{a=1}^{V} p_{k,a}^{\,n_a(t)}}
     {\sum_{k'=1}^{K} \prod_{a=1}^{V} p_{k',a}^{\,n_a(t)}}.
\end{equation}

\paragraph{Noisy Linear Regression (\texttt{E2})} %

Recall that each task $z = k$ is parameterized by a weight vector $\vw_k \in \R^D$.  
The sequence consists of input--output pairs
$(\vx_1, y_1), (\vx_2, y_2), \ldots$ where
$\vx_t \sim \mathcal{N}(\vzero, \mI_D)$ and
$y_t = \vw_k^\top \vx_t + \varepsilon_t$ with
$\varepsilon_t \sim \mathcal{N}(0, \sigma^2)$ independently.

At position $t$, define $\mX_{\le t} = (\vx_1, \ldots, \vx_t)^\top \in \R^{t \times D}$
and $\vy_{\le t} = (y_1, \ldots, y_t)^\top \in \R^t$.  The cumulative
log-likelihood under task $k$ is
\[
\log p(\vy_{\le t} \mid \mX_{\le t}, z = k)
\;=\;
-\frac{t}{2}\log(2\pi\sigma^2)
\;-\;
\frac{1}{2\sigma^2} \sum_{\tau=1}^{t}
(y_\tau - \vw_k^\top \vx_\tau)^2.
\]
We define the cumulative sum of squared residuals
$\mathrm{SSR}_k(t) = \sum_{\tau=1}^{t} (y_\tau - \vw_k^\top \vx_\tau)^2$.
Because the filtering posterior at position $t$ conditions only on past observations
$(\vx_1, y_1), \ldots, (\vx_{t-1}, y_{t-1})$, the SSR is evaluated at $t{-}1$:
\begin{equation}
\alpha_{t,k}
\;=\;
\frac{\exp\bigl(-\tfrac{1}{2\sigma^2}\,\mathrm{SSR}_k(t{-}1)\bigr)}
     {\sum_{k'=1}^{K} \exp\bigl(-\tfrac{1}{2\sigma^2}\,\mathrm{SSR}_{k'}(t{-}1)\bigr)},
\end{equation}
where $\mathrm{SSR}_k(0) = 0$ (no observations available at the first position).

\paragraph{Latent Markov Chain (\texttt{E3})} %

Recall that each task $z = k$ is parameterized by a transition kernel $\mP^{(k)}$.  The state space is $\gV$.
The first tokens carry no task-discriminative information; therefore $\alpha_{t,k} = 1/K$ for
$t = 1$.  For $t > 1$, each new token $s_t$ provides a
likelihood contribution:
\[
\P(s_t = a \mid s_{<t},\, z = k)
\;=\;
\mP^{(k)}_{s_{t-1}, a},
\]
The posterior accumulates these contributions:
\[
\P(z = k \mid s_{\le t})
\;=\;
\frac{\P\left(s_{\le t} \mid z = k\right)}{\sum_{k'=1}^K\P\left(s_{\le t} \mid z = k'\right)}
\]
which yields
\begin{equation}
\alpha_{t,k}
\;=\;
\frac{\prod_{\tau=2}^{t} \mP^{(k)}_{s_{\tau-1}, s_{\tau}}}
     {\sum_{k'=1}^{K} \prod_{\tau=2}^{t} \mP^{(k')}_{s_{\tau-1}, s_{\tau}}}.
\end{equation}

\subsubsection{Computing projection coefficients $\beta_{t,k}$}\label{app:lambda-computation}

This subsubsection details the two-step projection used to extract the simplex-constrained coefficients $\beta_{t,k}$ that appear in the interpolation model of Eq.~\ref{eq:additive_interp_main} (\Secref{sec:task-vectors}).
These coefficients are the empirical proxies for the Bayesian posterior weights $\alpha_{t,k}$ used in the posterior-alignment property \texttt{P3}, and the procedure below is the one used to produce every $\beta_{t,k}$ reported in Figure~\ref{fig:beta_alpha_trajectory}.
Let $\hat{\vtheta}_1,\dots,\hat{\vtheta}_K \in \R^d$ be the centered task vectors (\Secref{sec:extracting-hiddens}), and let
\[
\bar{\vh}_t \;=\; \vh_t^{(\ell)} - \hat{\vmu}_t
\]
be the centered hidden state at position~$t$ and layer~$\ell$, where $\hat{\vmu}_t$ is the empirical grand mean.
Following the main text, we obtain $\beta_{t,k}$ by projecting $\bar{\vh}_t$ onto the task-vector simplex
$\bigl\{\sum_k \beta_k\,\hat{\vtheta}_k : \beta_k \ge 0,\,\sum_k \beta_k = 1\bigr\}$,
implemented in two steps.

\paragraph{Step 1: Affine projection.}
We first find coefficients satisfying
$\bar{\vh}_t \approx \sum_{k=1}^K \beta_{t,k}\,\hat{\vtheta}_k$ with $\sum_k \beta_{t,k} = 1$,
without enforcing nonnegativity.
Using $\hat{\vtheta}_K$ as an anchor reduces this to an unconstrained least-squares problem
in $K-1$ variables, solved via the pseudoinverse of the matrix of pairwise differences
$(\hat{\vtheta}_1 - \hat{\vtheta}_K, \cdots, \hat{\vtheta}_{K-1} - \hat{\vtheta}_K)$.
The resulting affine coefficients satisfy sum-to-one but may be negative.

\paragraph{Step 2: Simplex projection.}
To enforce $\beta_{t,k} \ge 0$, we project the affine coefficients onto $\Delta^{K-1}$
via the standard Euclidean simplex projection~\citep{duchi2008efficient}: sort in decreasing order and apply a uniform threshold so that the result is nonnegative and sums to one.
The projected coefficients $\beta_{t,k}$ are what \Secref{sec:task-vectors} refers to as the simplex-projected interpolation coefficients.

\subsubsection{Additional posterior alignment results}\label{app:posterior-alignment-layers}

Figure~\ref{fig:beta_alpha_trajectory} in \Secref{sec:task-vectors} reports the simplex-projected coefficients $\beta_{t,k}$ at a representative middle layer.
Here we provide two complementary views: alignment at later layers, and unconstrained (pre-projection) coefficients. %
Together, these views demonstrate that posterior alignment (property~\texttt{P3}) is \emph{robust}, in the sense that: it holds across layers (not just the representative middle layer of the main text) and survives removal of the simplex projection (not just under the constrained estimator), confirming that the agreement between $\beta_{t,k}$ and $\alpha_{t,k}$ reflects a property of the hidden states themselves rather than an artifact of layer choice or the projection procedure.

\paragraph{Later layers.}
At later layers the projection coefficients track the Bayesian posterior $\alpha_{t,k}$ even more closely than at the middle layer shown in the main text.

\begin{figure}[!htbp]
\centering
\includegraphics[width=\linewidth]{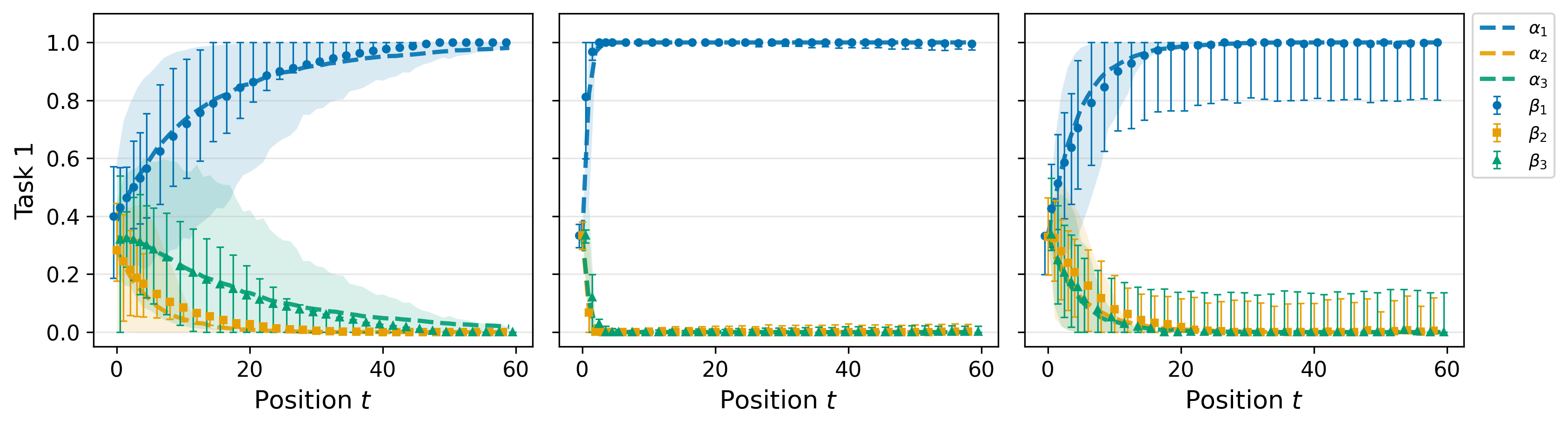}
\caption{Posterior alignment at later layers: simplex-projected coefficients $\beta_{t,k}$ (markers with error bars) vs.\ Bayesian posterior $\alpha_{t,k}$ (dashed lines with shaded 10--90\textsuperscript{th} percentile bands) for data generated from task~1.
Compared to the middle-layer results in Figure~\ref{fig:beta_alpha_trajectory}, the agreement is tighter across all three tasks.
Left: Dice (\texttt{E1}, layer~5). Middle: linear regression (\texttt{E2}, layer~15). Right: latent Markov (\texttt{E3}, layer~5).}
\label{fig:beta_alpha_later_layers}
\end{figure}

\paragraph{Affine (pre-projection) coefficients.}
For completeness, we also report the affine coefficients $\beta_{t,k}^{\mathrm{aff}}$, obtained by the same affine projection described in \Secref{app:lambda-computation} but without the simplex step (sum-to-one is enforced, but coefficients may be negative or exceed~1).
The close match with the posterior even before simplex projection confirms that the alignment is not an artifact of the projection procedure.

\begin{figure}[!htbp]
\centering
\includegraphics[width=\linewidth]{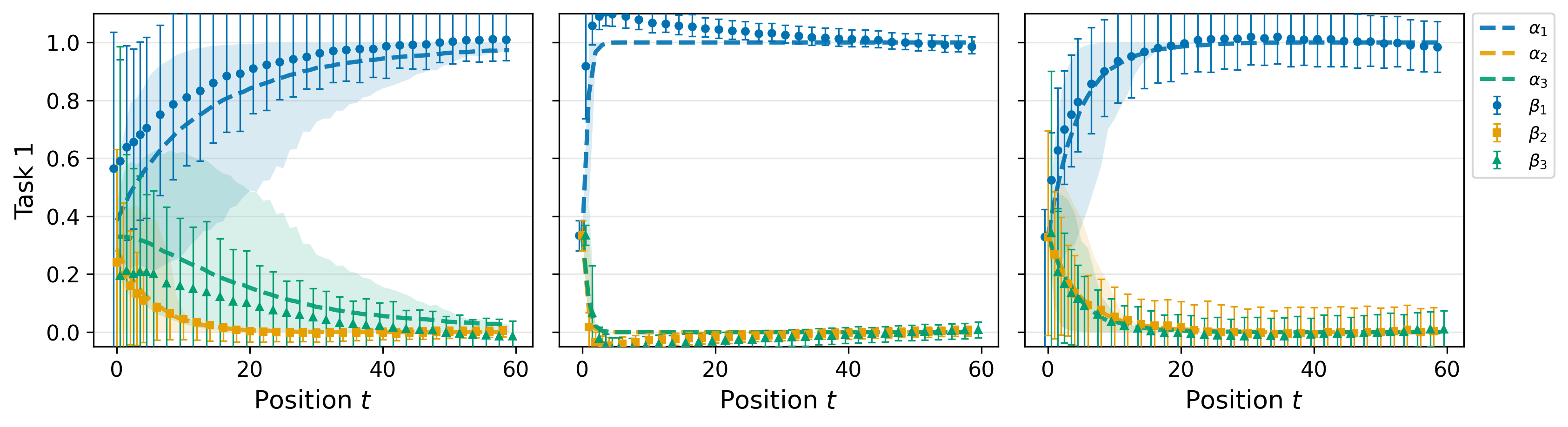}
\caption{Affine projection coefficients $\beta_{t,k}^{\mathrm{aff}}$ (markers with error bars) vs.\ Bayesian posterior $\alpha_{t,k}$ (dashed lines with shaded percentile bands).
No simplex projection is applied; coefficients may be negative or exceed~1.
Left: Dice (\texttt{E1}, layer~3). Middle: linear regression (\texttt{E2}, layer~9). Right: latent Markov (\texttt{E3}, layer~3).}
\label{fig:beta_alpha_aff}
\end{figure}

\paragraph{Summary.}
Taken together, Figures~\ref{fig:beta_alpha_later_layers} and~\ref{fig:beta_alpha_aff} show that posterior alignment $\beta_{t,k}\approx\alpha_{t,k}$ is \emph{not} contingent on the specific layer chosen for visualization in the main text, nor on the simplex-projection step in the estimator: the alignment persists at later layers (and in fact tightens) and survives when the nonnegativity step is removed (Figure~\ref{fig:beta_alpha_aff} uses $\beta_{t,k}^{\mathrm{aff}}$ in place of the simplex-projected $\beta_{t,k}$).

\subsection{Simplex Intervention Details}
\label{app:alpha_injection}

This subsection details the simplex-intervention experiment of \Secref{sec:task-vec:interv}: at a chosen layer $\ell$ we replace the task-subspace component of the residual stream with an arbitrary point $\sum_k \alpha_{t,k}^{*}\hat{\vtheta}_k$ in the task-vector simplex (with $\valpha_t^{*}$ drawn from $\Delta^{K-1}$), leave the orthogonal complement untouched, and measure how closely the steered model's output matches the Bayesian mixture prediction that $\valpha_t^{*}$ dictates.
The purpose is causal: the alignment between $\beta_{t,k}$ and $\alpha_{t,k}$ established by Property~\texttt{P3} (\Secref{sec:emp-task-vector}) is observational, and could in principle reflect a hard task-selection rule rather than a true continuous interpolation.
The simplex steering, paired with the \emph{mode-output} baseline introduced below, addresses both points: it shows that the model's output tracks $\valpha_t^{*}$ throughout the simplex (not merely near vertices), supporting the main-text claim that the task-vector subspace is the variable through which the model interpolates between training tasks.
Let
\begin{equation} \label{eq:task-projection}
    \mP_{\mathrm{task}} = \hat{\mTheta}\,(\hat{\mTheta}^\top\hat{\mTheta})^{\dagger}\hat{\mTheta}^\top
\end{equation}
denote the orthogonal projector onto the task subspace, where $\hat{\mTheta} \in \R^{d \times K}$ stacks the centered task vectors as columns and $\mA^{\dagger}$ denotes the pseudoinverse of the matrix $\mA$.  
We substitute the task-subspace component with a target interpolation, replacing the model's original coefficients $\beta_{t,k}$ with coefficients $\alpha_{t,k}^{*}$ randomly drawn from the simplex $\Delta^{K-1}$:
\[
\vh_t^{(\ell),\mathrm{int}}
=
\hat{\vmu}_t
+ \sum_{k=1}^K \alpha_{t,k}^{*}\, \hat{\vtheta}_k
+ (\mI - \mP_{\mathrm{task}})\bigl(\vh_t^{(\ell)} - \hat{\vmu}_t\bigr).
\]
This steers the hidden representation to a chosen point in the task-vector simplex while leaving the orthogonal complement unchanged.  
The model then completes the forward pass from layer~$\ell$ onward using the modified hidden state.

\paragraph{Comparing with an additional baseline: Mode-output.}
A central claim of \Secref{sec:task-vec:interv} is that task vectors implement a \emph{continuous interpolation} over tasks rather than a discrete selection of the most likely one.
The natural alternative hypothesis is a \emph{hard-selection} (nearest-neighbor) mechanism: the model identifies the dominant task $k^{*} = \arg\max_k \alpha_{t,k}^{*}$ and predicts as if that task were certain.
To rule this out, we compare simplex steering against a \emph{mode-output} baseline that realizes this hard-selection rule using the oracle Bayesian posterior over tasks. %
Across all three experiments, simplex steering tracks the target mixture \emph{uniformly across $\Delta^{K-1}$}, while the mode baseline matches the model \emph{only near the vertices} and degrades sharply in the interior.
This shows that task vectors cannot be reduced to a nearest-neighbor or hard-assignment mechanism, strengthening the main-text conclusion that they encode and combine task information continuously across the simplex.

\paragraph{Mode-output baseline.}
Given target coefficients $\valpha_t^{*}$, the dominant task is
\[
k^{*} \;=\; \arg\max_k \alpha_{t,k}^{*},
\]
and the mode-output prediction is the Bayesian-optimal next-token distribution under that single task, $\P(\cdot \mid z = k^{*}, s_t)$.
This baseline is computed entirely from the ground-truth generative model and does \emph{not} involve a forward pass through the trained model; as such, it provides the strongest realization of the hard-selection alternative defined above and an upper bound on the performance any nearest-neighbor mechanism in task space could achieve.

\paragraph{Empirical comparison.}
Figure~\ref{fig:injection_simplex_compare} reports the prediction error (KL divergence or RMSE, depending on the experiment) of simplex steering and the mode-output baseline against the target mixture across $\Delta^{K-1}$ for \texttt{E1}--\texttt{E3}.
Simplex steering attains uniformly low error throughout the simplex, including the interior region where multiple tasks contribute comparably and any hard-selection rule must discard substantial posterior mass.
The mode baseline, by contrast, matches the target only near the vertices and degrades sharply toward the interior, with error growing roughly with the entropy of $\valpha_t^{*}$.
This pattern is precisely what the interpolation hypothesis predicts and what the hard-selection hypothesis cannot accommodate.

\begin{figure}[H] \centering 
\includegraphics[width=\linewidth]{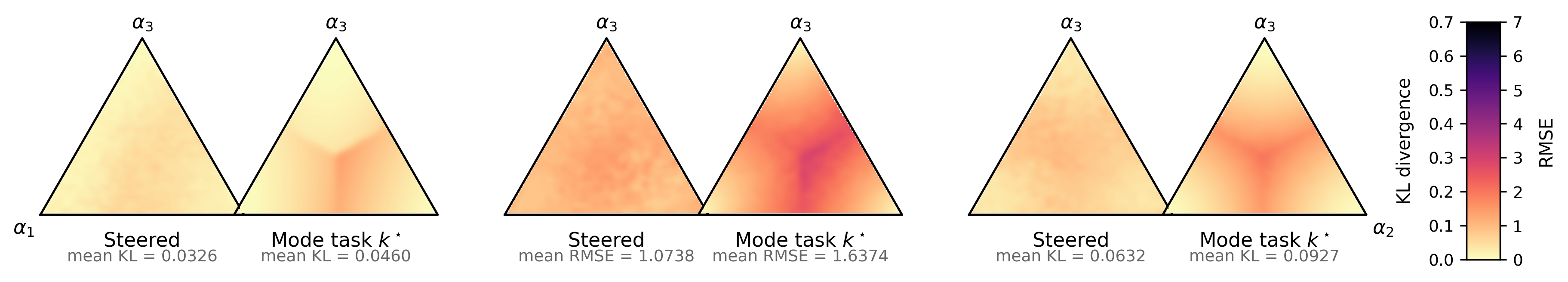}
\caption{\textbf{Simplex steering outperforms the mode-output baseline across the entire task simplex.}
Each panel pair shows error (KL divergence or RMSE) under simplex steering (\emph{Steered}) and the mode-output baseline (\emph{Mode task $k^{*}$}) for \texttt{E1} (left), \texttt{E2} (middle), and \texttt{E3} (right).
Simplex steering achieves uniformly low error throughout the simplex interior, whereas the mode baseline (which selects the pure-task output for the dominant task, and is equivalent to an idealised hard-selection task-vector injection) performs well only near the vertices and degrades in the interior, where task mixtures are more balanced.}
\label{fig:injection_simplex_compare} \end{figure}

\subsection{Alternative task-vector estimators}
\label{app:estimators}

The paper estimates task vectors by class means, replacing the expectation in Eq.~\ref{eq:task-vec:defn} by an empirical average, as in \citet{hendel2023context} and \citet{toddfunction}. Under \texttt{P0} and \texttt{P1}, any estimator that recovers the task subspace estimates the same vectors $\{\vtheta_k\}_{k\le K}$, so the results of the paper should not depend on this choice. We check this on the major-only \texttt{E3} model of \Secref{sec:task-vectors} at layer $3$, with the model trained at diversity exponent $m=6$ as a second case. We run four alternatives through the same evaluation pipeline as the class means, with the same grand mean, token encodings and simplex solver. Diagonal linear discriminant analysis (LDA) rescales each coordinate of the class means by its within-class standard deviation and uses task labels. Token-conditioned PCA applies principal component analysis to hidden states after conditioning on the current token and then $k$-means, and uses no labels. Pooled PCA applies principal component analysis to all hidden states and then $k$-means, and uses no labels. Random vectors with the same norm as ours serve as a reference. Since \texttt{P0} and \texttt{P1} depend on the hidden states alone, the estimators can differ only on \texttt{P2}, \texttt{P3} and the simplex injection of \Secref{sec:task-vec:interv}.

\paragraph{Evaluation.} We measure \texttt{P2} by the interpolation $R^2$ of Eq.~\ref{eq:additive_interp_main}. For \texttt{P3} we compare the coefficients $\beta_{t}$ to the oracle posterior $\alpha_{t}$ on held-out positions in three ways. The first is the Hellinger distance
\begin{equation}\label{eq:est-hellinger}
H(\beta_t,\alpha_t) \;=\; \Big(\tfrac12 \textstyle\sum_{k \le K} \big(\sqrt{\beta_{t,k}}-\sqrt{\alpha_{t,k}}\big)^2\Big)^{1/2} \in [0,1],
\end{equation}
averaged over positions, where smaller is better. The second is the fraction of positions at which the largest coefficient and the largest posterior probability name the same task. The third is the $R^2$ of a regression of $\alpha_t$ on $\beta_t$, fitted on one half of the held-out data and evaluated on the other. The simplex injection follows Fig.~\ref{fig:injection_simplex}. For the two unsupervised estimators, $k$-means recovers the $K$ vertices only up to a permutation of the task labels. We fix the permutation on the estimation split, by matching each cluster to the task it most often contains, and keep it fixed on the evaluation split, so labels enter only as a relabeling of the recovered vertices. In our runs every cluster contains sequences of a single task (purity $1.00$), so the clustering alone determines the partition.

\begin{table}[t]
\centering
\small
\setlength{\tabcolsep}{5pt}
\caption{\textbf{Task-vector estimators under the paper's evaluation pipeline (major-only \texttt{E3} model, layer $3$).} \texttt{P2} interpolation $R^2$; Hellinger distance $H(\beta,\alpha)$ of Eq.~\ref{eq:est-hellinger}; fraction of positions at which $\beta_t$ and $\alpha_t$ name the same task; held-out $R^2$ of $\alpha_t$ regressed on $\beta_t$; cluster purity for the unsupervised estimators. \textbf{Bottom:} simplex injection with the procedure of Fig.~\ref{fig:injection_simplex}, mean KL divergence to the target mixture before and after steering.}
\label{tab:est-battery}
\begin{tabular}{@{}l l cccc c@{}}
\toprule
estimator & labels & \texttt{P2} $R^2$ & $H(\beta,\alpha)$ & same task & $R^2(\alpha \text{ on } \beta)$ & purity \\
\midrule
Class means (ours)               & task & 0.90 & 0.11 & 1.00 & 0.98 & -- \\
Diagonal LDA                     & task & 0.90 & 0.10 & 1.00 & 0.98 & -- \\
Token-conditioned PCA, $k$-means & none & 0.90 & 0.09 & 1.00 & 0.98 & 1.00 \\
Pooled PCA, $k$-means            & none & 0.81 & 0.10 & 1.00 & 0.98 & 1.00 \\
Random, matched norm             & --   & 0.58 & 0.61 & 0.72 & 0.16 & -- \\
\midrule
\multicolumn{7}{@{}l}{\textit{Simplex injection, mean KL divergence (before $\to$ after steering)}} \\
\multicolumn{2}{@{}l}{Class means (ours)} & \multicolumn{2}{l}{$0.231 \to 0.062$} & \multicolumn{2}{l}{Pooled PCA, $k$-means} & $0.233 \to 0.093$ \\
\multicolumn{2}{@{}l}{Token-conditioned PCA, $k$-means} & \multicolumn{2}{l}{$0.232 \to 0.063$} & \multicolumn{2}{l}{Random, matched norm} & $0.231 \to 0.223$ \\
\bottomrule
\end{tabular}
\end{table}

\begin{figure}[t]
\centering
\includegraphics[width=\linewidth]{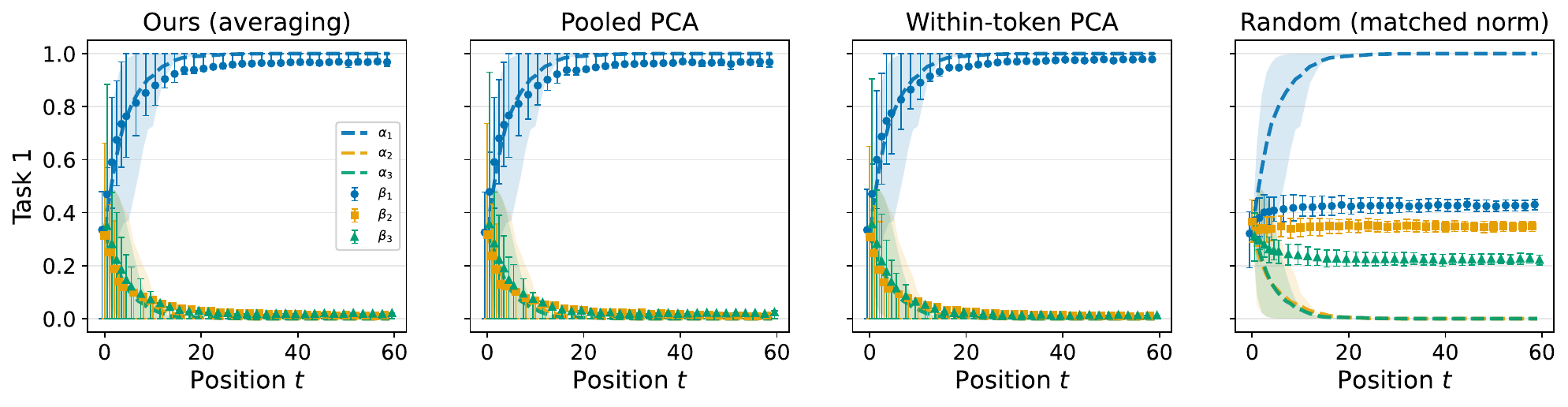}
\caption{\textbf{\texttt{P3} by estimator (major-only \texttt{E3} model, layer $3$, task $1$).} Simplex coefficients $\beta_{t,k}$ (markers with percentile bars) against the oracle posterior $\alpha_{t,k}$ (dashed), as in Fig.~\ref{fig:beta_alpha_trajectory}, for class means, pooled PCA, token-conditioned PCA, and random vectors of matched norm. The three estimators that recover the task subspace track the posterior; random vectors do not.}
\label{fig:est-p3}
\end{figure}

\begin{figure}[t]
\centering
\includegraphics[width=0.8\linewidth]{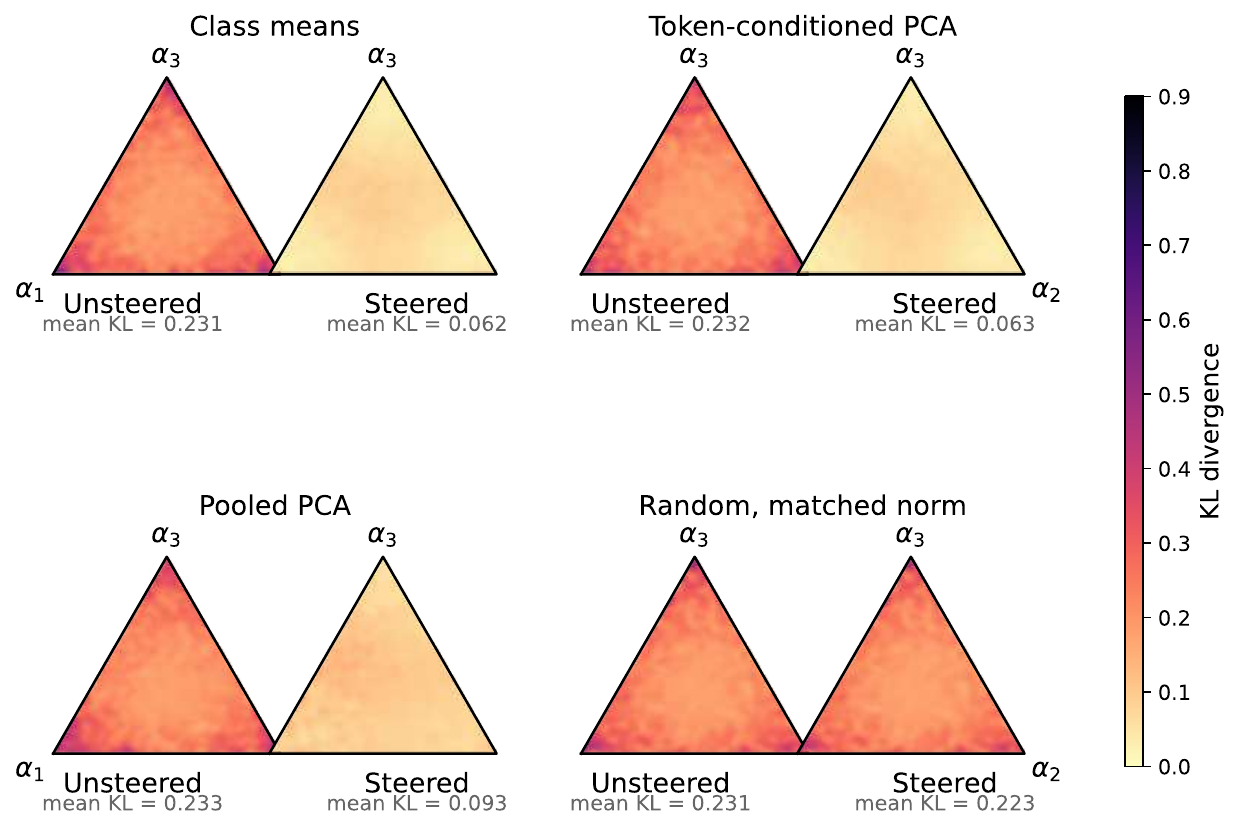}
\caption{\textbf{Simplex injection by estimator} (procedure of Fig.~\ref{fig:injection_simplex}, \texttt{E3}, layer $3$). Each pair shows the KL divergence to the target mixture over the simplex of targets, before steering (left triangle) and after (right triangle), on a shared color scale, with the mean printed beneath; the means match Table~\ref{tab:est-battery}.}
\label{fig:est-injection}
\end{figure}

\paragraph{Results.} Table~\ref{tab:est-battery} and Figs.~\ref{fig:est-p3}--\ref{fig:est-injection} show that class means, diagonal LDA and token-conditioned PCA give the same \texttt{P2}, \texttt{P3} and injection results, and that random vectors of matched norm give none of them. Diagonal LDA reproduces the class means because the within-class rescaling rotates each task vector only slightly (cosine at least $0.94$) and changes the simplex coefficients by at most $0.012$ on average. Token-conditioned PCA recovers the task vectors without any labels, which shows that the geometry is a property of the hidden states and not of the supervision. Pooled PCA, the one estimator that does not condition on the current token, is the one that degrades. Token variance is comparable to task variance at this layer, so its leading components mix task and token directions, which lowers the \texttt{P2} $R^2$ to $0.81$ and the steering quality to $0.093$. On the $m=6$ model this degradation is larger (fraction of positions naming the same task $0.51$), whereas token-conditioned PCA stays at $0.99$ and class means at $1.00$. Finally, a linear probe regressing hidden states on one-hot task indicators is the class-mean estimator exactly, since the least-squares coefficient of each task is that task's mean hidden state.

Our goal is to validate the class-mean task vector widely used in prior work, not to advocate it over alternatives. Under \texttt{P0} and \texttt{P1}, class means, diagonal LDA and token-conditioned PCA all estimate the same $\{\vtheta_k\}_{k\le K}$, and the posterior-alignment and intervention results of the paper are not artifacts of one estimator.

\subsection{Causal identification of individual vertices}
\label{app:vertex-identity}

The injection of \Secref{sec:task-vec:interv} sets the task-subspace component of $\vh_t$ to $\sum_k\alpha^{*}_{k}\hat{\vtheta}_k$ and scores the output against the Bayesian predictive of that specific mixture $\alpha^{*}$. Its success could in principle reflect only the span $\operatorname{col}(\hat{\mTheta})$, with any basis of the same plane steering equally well. To check this we repeat the injection of Fig.~\ref{fig:injection_simplex} on the major-only \texttt{E3} model (layer $3$, per-position injection, $10$k Dirichlet targets, a new seed) with bases that span the identical task plane but assign the vertices differently: the frame rotated in-plane by $30^{\circ}$ or $60^{\circ}$, two vertices swapped, all three cyclically relabeled, and $12$ random in-plane rotations $\varphi\sim U(0^{\circ},360^{\circ})$. Every variant spans exactly the plane of our vectors; only the vertex assignment varies. The unsteered KL is $0.233$, and steering to the single most likely task $k^\star$ instead of the mixture gives $0.092$.

\begin{table}[t]
\centering
\small
\setlength{\tabcolsep}{8pt}
\caption{\textbf{Same plane, different vertices.} Mean KL to the target mixture after steering through bases that share the task plane with our class-mean vertices but differ in vertex assignment; unsteered $0.233$.}
\label{tab:vertex-identity}
\begin{tabular}{@{}l l c@{}}
\toprule
basis (identical task plane) & per-vertex cosine to ours & steered KL \\
\midrule
Ours (class-mean vertices)            & $1.00$, $1.00$, $1.00$    & $\mathbf{0.062}$ \\
Rotated $30^{\circ}$ in-plane         & $0.87$ ($\times 3$)       & $0.075$ \\
Rotated $60^{\circ}$ in-plane         & $0.50$ ($\times 3$)       & $0.112$ \\
Two vertices swapped                  & $-0.41$, $-0.41$, $1.00$  & $0.143$ \\
Cyclically relabeled (all wrong)      & $-0.41$, $-0.44$, $-0.64$ & $0.198$ \\
Random in-plane rotation ($12$ draws) & --                        & $0.143\pm 0.055$ (range $0.067$--$0.215$) \\
\bottomrule
\end{tabular}
\end{table}

\begin{figure}[t]
\centering
\includegraphics[width=\linewidth]{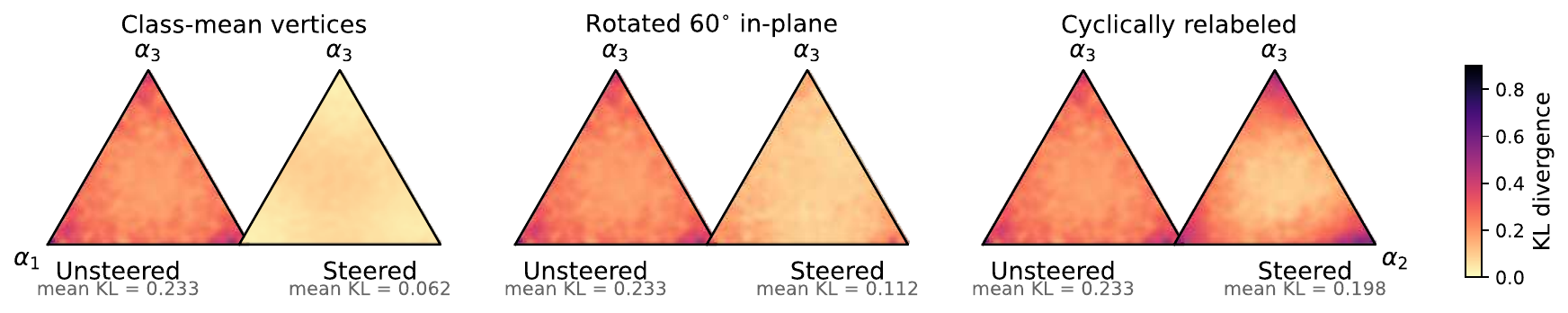}
\caption{\textbf{Injection through same-plane bases.} KL divergence to the target mixture over the simplex of targets, before steering (left triangle of each pair) and after (right), on a shared color scale, with the mean printed beneath. With the plane held fixed, steering error grows with vertex misalignment: $0.062$, $0.112$, $0.198$.}
\label{fig:vertex-identity}
\end{figure}

Table~\ref{tab:vertex-identity} and Fig.~\ref{fig:vertex-identity} show that steering error grows monotonically with vertex misalignment at fixed plane. A typical wrong-vertex basis in the same plane (random rotation, $0.143\pm 0.055$) fails as badly as a swap or a cyclic relabeling, so the $30^{\circ}$ and $60^{\circ}$ rows understate the typical effect. Steering therefore requires the correct assignment of vertices to tasks, not just the span, and the intervention identifies the individual vertices. This agrees with \texttt{P3}, which is evaluated coordinate by coordinate ($\beta_k$ tracks $\alpha_k$ with $R^2=0.98$), and with \Secref{app:estimators}, where class means, diagonal LDA and token-conditioned PCA recover the same vertices. The first row reproduces the $0.231\to 0.062$ of Fig.~\ref{fig:injection_simplex} with a new seed.

\section{Further Details about the Two Inference Modes}
\label{app:two-modes-evidence}

This appendix collects the empirical evidence supporting the two-mode picture
introduced in \Secref{sec:two-modes}, complementing the KL-based transition of
\Secref{sec:two-mode}, the orthogonality results of
\Secref{sec:representation_geometry}, and the causal interventions of
\Secref{sec:orth_ablation_main}.
The material is organized as follows.
\begin{itemize}
\item \Secref{app:training-mixture}: detailed discussion on the training distribution used in experiments of \Secref{sec:representation_geometry}. 
\item \Secref{app:id-ood-loss}: training dynamics of in-distribution and
out-of-distribution loss, providing a performance-based view of the same
transition between \texttt{M1} and \texttt{M2}.
\item \Secref{app:traj-simplex}: a complementary geometric view of the
representation orthogonality, tracing simplex trajectories of hidden-state
projections under low and high task diversity.
\item \Secref{app:real_llm}: pretrained-LLM (Qwen2.5-7B) experiments
that extend the two-mode trajectory analysis from synthetic data to a real
language model.
\item \Secref{app:intervention_perlayer}: full per-layer breakdowns of the
two complementary causal interventions of Table~\ref{tab:interventions},
with implementation details (\Secref{app:orth_ablation}) and a
decomposition of the optimized orthogonal directions
(\Secref{app:orth_decomp}) showing they encode running context statistics.
\end{itemize}

\subsection{Major/Minor Training Mixture}
\label{app:training-mixture}

This subsection describes the major/minor task-mixture training distribution used throughout \Secref{app:two-modes-evidence} for the coexistence-of-modes experiments of \Secref{sec:two-modes}.
The minor-pool size $N_{\mathrm{minor}}$ is the central experimental knob: small $N_{\mathrm{minor}}$ allows every task to be memorized (favoring \texttt{M1}), while large $N_{\mathrm{minor}}$ makes individual minor tasks unmemorizable and forces the model toward in-context generalization (favoring \texttt{M2}).

\paragraph{Task pools.}
Before training begins, two disjoint task pools are sampled once from the
task-specific prior and held fixed throughout:
\begin{itemize}[leftmargin=1.5em]
  \item \textbf{Major pool.} Three tasks $\{z^{\mathrm{maj}}_1, z^{\mathrm{maj}}_2, z^{\mathrm{maj}}_3\}$
    drawn from the prior $\pi$ (e.g.\ $\mathrm{Dir}(\boldsymbol{1}_V)$ for
    \texttt{E1} and \texttt{E3}, $\gN(\vzero, \mI_D)$ for \texttt{E2}).
  \item \textbf{Minor pool.} $N_{\mathrm{minor}}$ tasks
    $\{z^{\mathrm{minor}}_1, \ldots, z^{\mathrm{minor}}_{N_{\mathrm{minor}}}\}$
    drawn independently from the same prior $\pi$, conditioned on being disjoint
    from the major pool. We allow $N_{\mathrm{minor}} = 0$ as a degenerate case
    in which the minor pool is empty and all sequences are drawn from the major
    pool only.
\end{itemize}
The minor-pool size $N_{\mathrm{minor}}$ is the primary control variable. The
full diversity sweep (\Secref{app:hyperparams}) covers
$N_{\mathrm{minor}} \in \{0, 1, 2, 4, 8, 16, 32, 64, 128, 256, 512, 1024\}$;
Figure~\ref{fig:ood_r2_combined} displays the five-value subset
$N_{\mathrm{minor}} \in \{0, 1, 4, 16, 1024\}$ in all three panels (\texttt{E1},
\texttt{E2}, \texttt{E3}), so that the legend values match exactly across
experiments.

\paragraph{Sequence sampling.}
For $N_{\mathrm{minor}} \ge 1$, each training sequence is generated by the
following two-stage procedure:
\begin{enumerate}[leftmargin=1.5em]
  \item \textbf{Pool selection.} With probability $p_{\mathrm{minor}} = 0.1$,
    draw from the minor pool; with probability $1 - p_{\mathrm{minor}} = 0.9$,
    draw from the major pool.
  \item \textbf{Task selection.} Sample a task $z$ uniformly at random from the
    selected pool, then generate a sequence of length $T$ from $p_z$.
\end{enumerate}
For the degenerate case $N_{\mathrm{minor}} = 0$ the minor pool is empty: the
Bernoulli pool-selection step is skipped and every sequence is drawn from the
major pool (equivalently, the effective $p_{\mathrm{minor}}$ collapses to $0$).
For $N_{\mathrm{minor}} \ge 1$, this gives major tasks a dominant $9$:$1$
weight over the minor pool regardless of $N_{\mathrm{minor}}$, so each
individual minor task receives weight
$p_{\mathrm{minor}}/N_{\mathrm{minor}} = 0.1/N_{\mathrm{minor}}$, vanishingly
small for large $N_{\mathrm{minor}}$.

\paragraph{OOD tasks.}
Each OOD task is obtained by drawing a fresh sample from the same
task-specific prior used to construct the major and minor pools: a Dirichlet
draw over categorical distributions for \texttt{E1} and \texttt{E3}, and an
isotropic Gaussian draw over weight vectors for \texttt{E2}.  These draws are
made independently of the stored training pools and are not cached: every
batch of OOD tasks is generated afresh from the prior.  Because the prior is
continuous, with probability one the resulting task does not coincide with
any task in the major or minor pool, so OOD tasks are never seen during
training.  This ensures that OOD performance reflects genuine
out-of-distribution generalization rather than minor-task memorization.

\paragraph{Relation to task diversity.}
When $N_{\mathrm{minor}} \in \{0, 1\}$ the minor pool is empty or contains a
single task, which the model can readily memorize; for small $N_{\mathrm{minor}}$
the model essentially extends task-retrieval behavior to the minor pool. As
$N_{\mathrm{minor}}$ grows, the per-task training weight $0.1/N_{\mathrm{minor}}$
becomes negligible and no individual minor task can be memorized, forcing the
model to develop a context-based generalization strategy. The transition in
representation geometry visible in Figure~\ref{fig:ood_r2_combined} reflects
this shift.

\subsection{In-distribution and out-of-distribution loss}
\label{app:id-ood-loss}

To complement the KL-based comparison in \Secref{sec:two-mode}, we also examine the training dynamics of in-distribution (ID) and out-of-distribution (OOD) loss.
ID sequences are drawn from the major/minor training mixture and OOD sequences from fresh prior draws disjoint from both training pools; see \Secref{app:training-mixture} for the precise sampling procedure.
These plots provide a direct performance-based view of the same \texttt{M1}-to-\texttt{M2} transition: as task diversity increases, the model enters a regime in which it obtains strong OOD performance even though its predictions are no longer best described by \texttt{M1} (the exact Bayesian posterior over the memorized task set).

Figure~\ref{fig:id_ood_loss_combined} reports ID and OOD loss curves for \texttt{E1}, \texttt{E2} and \texttt{E3}.
The main qualitative pattern is that OOD loss remains low, or improves substantially during training, in the same region where the KL comparison favors \texttt{M2}.
This supports the interpretation that \texttt{M2} is not merely a worse fit to the \texttt{M1} baseline but a distinct inference mode that supports effective OOD prediction.

\begin{figure}[h]
\centering
\includegraphics[width=0.7\linewidth]{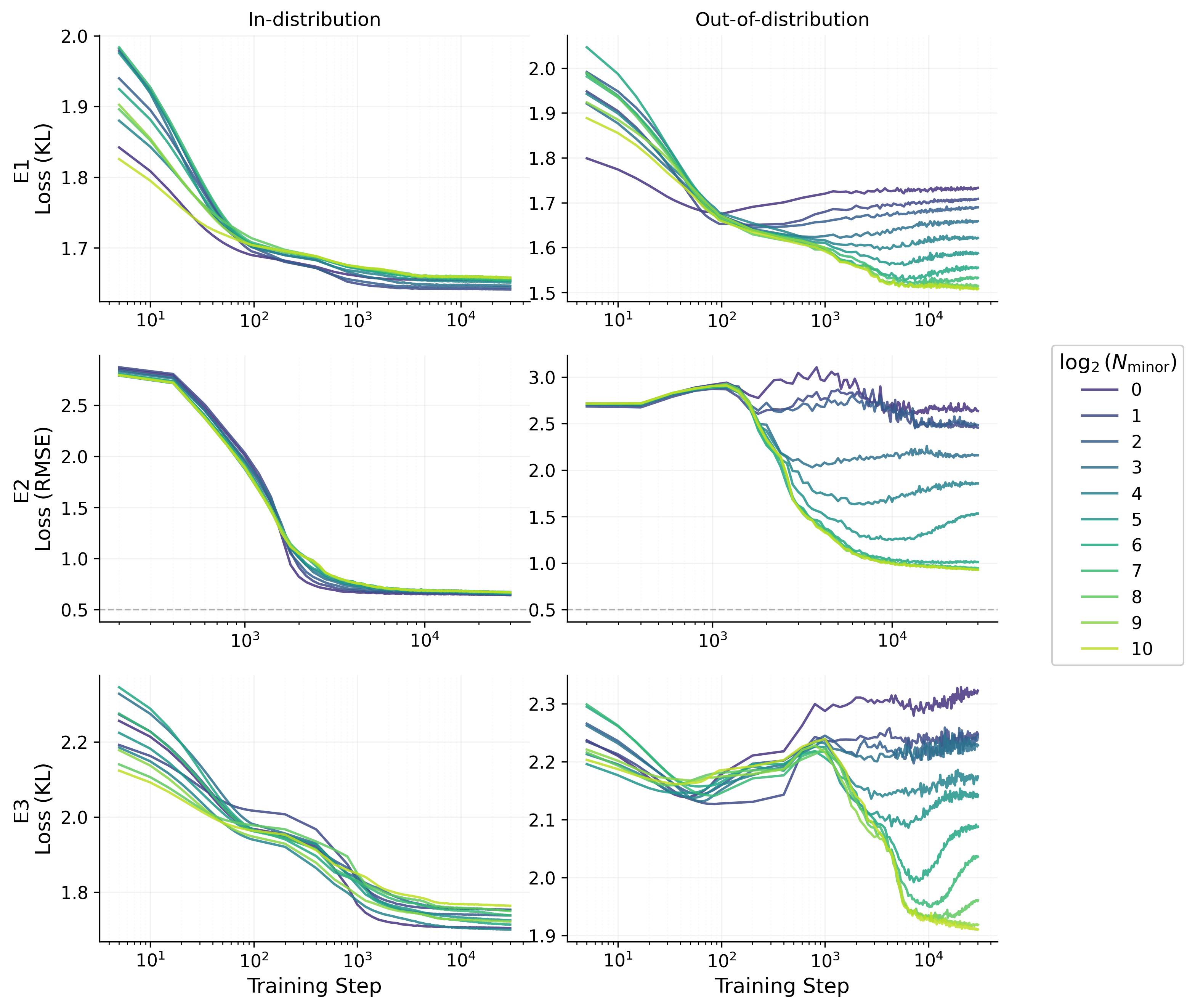}
\caption{\textbf{Training dynamics of in-distribution and out-of-distribution loss.}
Panels (top to bottom): \texttt{E1}, \texttt{E2}, \texttt{E3}. Each panel compares ID and OOD performance over training.
Strong OOD performance in the later training regime of high task diversity is consistent with the emergence of the extrapolative inference mode \texttt{M2} (cf.\ \Secref{sec:two-mode}). %
}
\label{fig:id_ood_loss_combined}
\end{figure}

\subsection{Capacity versus per-task data}
\label{app:capacity}

In the diversity sweep of Fig.~\ref{fig:kl_combined}, the minor tasks share a fixed prior mass $p_{\mathrm{minor}}=0.1$ (\Secref{app:training-mixture}). Increasing the number of minor tasks $N_{\mathrm{minor}}=2^m$ therefore also shrinks the training data that each minor task receives, to a share of $0.1/2^{m}$. The transition from \texttt{M1} to \texttt{M2} along $m$ may thus be driven by the data available per task, by limited model capacity, or by both. This section separates the two with two sweeps on \texttt{E3} ($30$k steps, as in \Secref{app:order-sweep}). The first varies the hidden dimension $d$ at fixed $p_{\mathrm{minor}}$. The second varies $m$ while holding each minor task's share of the training data fixed. We use the mode diagnostic $\log(\mathrm{KL}_{\mathrm{Bayesian}}/\mathrm{KL}_{\mathrm{extrapolative}})$ of \Secref{sec:two-modes}, negative for \texttt{M1}, and the transition point $m^\ast$, the value of $m$ at which the diagnostic changes sign.

\begin{table}[t]
\centering
\small
\setlength{\tabcolsep}{6pt}
\caption{\textbf{Capacity and per-task data.} \textbf{Top:} mode diagnostic versus hidden dimension $d$ at fixed $p_{\mathrm{minor}}=0.1$; $d=128$ is the paper's architecture. \textbf{Bottom:} the sweep at $d=128$ in which $p_{\mathrm{minor}}\propto N_{\mathrm{minor}}$, so that every minor task keeps the share of training data it has at $m=6$, where the two sweeps coincide (total data and steps unchanged).}
\label{tab:capacity-sweep}
\begin{tabular}{@{}l rrrrr c@{}}
\toprule
width $d$ & $m{=}2$ & $m{=}4$ & $m{=}6$ & $m{=}8$ & $m{=}10$ & transition $m^\ast$ \\
\midrule
32  & $-2.77$ & $-0.88$ & $+0.84$ & $+2.02$ & $+1.99$ & ${\approx}\,5.0$ \\
64  & $-3.43$ & $-2.42$ & $-0.45$ & $+2.09$ & $+2.46$ & ${\approx}\,6.4$ \\
128 & $-4.15$ & $-2.99$ & $-1.45$ & $+0.90$ & $+2.99$ & ${\approx}\,7.2$ \\
256 & $-4.67$ & $-3.65$ & $-2.41$ & $-0.48$ & $+2.81$ & ${\approx}\,8.3$ \\
512 & $-4.40$ & $-3.99$ & $-2.80$ & $-1.16$ & $+2.77$ & ${\approx}\,8.6$ \\
\midrule
sweep ($d=128$) & $m{=}2$ & $m{=}4$ & $m{=}6$ & $m{=}8$ & & \\
\midrule
fixed $p_{\mathrm{minor}}=0.1$ (paper) & $-4.15$ & $-2.99$ & $-1.45$ & $+0.90$ & & \\
$p_{\mathrm{minor}}\propto N_{\mathrm{minor}}$ & $-1.28$ & $-1.63$ & $-1.45$ & $-0.56$ & & \\
\bottomrule
\end{tabular}
\end{table}

\begin{figure}[t]
\centering
\includegraphics[width=0.85\linewidth]{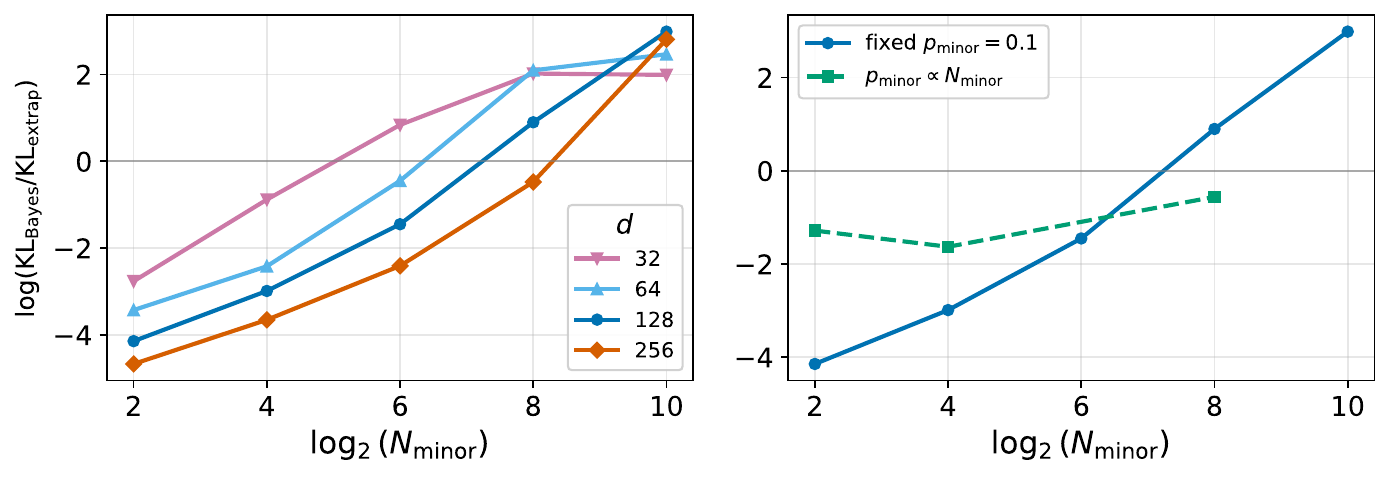}
\caption{\textbf{Both factors shift the transition point.} \textbf{Left:} mode diagnostic versus $m$ for $d\in\{32,64,128,256\}$ (the $d=512$ row of Table~\ref{tab:capacity-sweep} is not shown); larger width delays the transition but never removes it. \textbf{Right:} at $d=128$, the sweep with fixed $p_{\mathrm{minor}}$ changes sign near $m\approx 7$, whereas holding each minor task's share of the data fixed keeps the model in \texttt{M1} at every $m$ for which the sweep can be run.}
\label{fig:capacity-transition}
\end{figure}

\paragraph{Model width.} Table~\ref{tab:capacity-sweep} (top) and Fig.~\ref{fig:capacity-transition} (left) show that capacity shifts the transition point to higher diversity, $m^\ast\approx 5.0$, $6.4$, $7.2$, $8.3$, $8.6$ for $d=32$ to $512$, with a diminishing effect at the largest width. It never removes the transition, which occurs within the tested range at every width.

\paragraph{Data per task held fixed.} When each minor task's share of the training data is held fixed (Table~\ref{tab:capacity-sweep}, bottom), the model stays in Bayesian retrieval at every task count up to $m=8$ ($-1.28$, $-1.63$, $-0.56$ at $m=2,4,8$). At $m=10$ this sweep cannot continue, since $p_{\mathrm{minor}}$ would have to exceed one. The closest run, $p_{\mathrm{minor}}=0.8$, gives each minor task half the data of the fixed-share sweep but eight times that of the fixed-$p_{\mathrm{minor}}$ sweep, and it does make the transition, with diagnostic $+1.75$. Retrieval therefore persists as long as the data per task is maintained.

\paragraph{Conclusion.} Both factors move the transition point in the expected direction. More data per task or a larger hidden dimension delays the transition from \texttt{M1} to \texttt{M2} and shifts it to higher task diversity. The shrinking data per task, not the task count itself, drives the transition, and capacity shifts where it occurs. This refines rather than changes the reading of Fig.~\ref{fig:kl_combined}, and it parallels the diversity threshold between memorization and in-context learning in regression \citep{raventos2023pretraining} and the transient generalizing solution that gives way to memorization as the data per task accumulates \citep{carroll2025dynamics}.

\subsection{Simplex trajectories of hidden state projections}
\label{app:traj-simplex}

\Secref{sec:representation_geometry} (Figure~\ref{fig:ood_r2_combined}) summarizes the orthogonality between OOD and major-task representations through a single scalar: the projection $R^2$ of OOD hidden states onto the task subspace $\operatorname{col}(\hat{\mTheta})$, tracked over the course of training.
OOD tasks throughout this subsection are sampled as described in \Secref{app:training-mixture}: fresh draws from the same prior used to build the major and minor pools, almost surely disjoint from the training pools because the prior is continuous.
This subsection provides a complementary geometry-resolved view: instead of reporting how much OOD variance lies in $\operatorname{col}(\hat{\mTheta})$, we plot \emph{where} batch-averaged hidden states sit within the affine simplex $\Delta^2$ spanned by the three major-task vectors as the in-context length grows.
This adds three pieces of information beyond Figure~\ref{fig:ood_r2_combined}:
(i) the within-sequence analog of the across-training $R^2$ separation, namely that major-task trajectories converge to the corresponding simplex vertices as context accumulates;
(ii) a check that $\operatorname{col}(\hat{\mTheta})$ genuinely captures major-task structure rather than merely accommodating it as a linear span, since the affine projections of major-task hidden states fall approximately inside $\Delta^2$ even though no nonnegativity constraint is imposed;
(iii) a direct geometric picture of OOD orthogonality across diversity regimes: OOD trajectories drift far outside the simplex at high task diversity, whereas at low diversity they remain inside the simplex with $R^2$ comparable to that of major tasks.

We extract hidden states from the post-\texttt{MLP} residual stream (cf.\ \Secref{sec:extracting-hiddens}) at an intermediate layer chosen per experiment: layer~$3$ of~$6$ for \texttt{E1}, layer~$10$ of~$16$ for \texttt{E2}, and layer~$4$ of~$6$ for \texttt{E3} (0-indexed).
Two objects are computed from these hidden states.

\emph{Task vectors} (simplex vertices).
For each major task $k$, we pool hidden states across all $B=64$ sequences and across a window of late context positions, yielding a single mean vector per task; subtracting the grand mean (the average of these task means) gives the centered task vectors $\hat{\vtheta}_1, \hat{\vtheta}_2, \hat{\vtheta}_3$.
These serve as the three vertices of the reference simplex $\Delta^2$.

\emph{Trajectories} (plotted curves).
Separately, for each task group (whether a major task or an OOD task) and each in-context position $t$, we average hidden states across the $B=64$ sequences to obtain a single \emph{batch-mean} hidden vector, then subtract the grand mean.
This averaging over multiple samples greatly reduces the variability introduced by different token realizations, so that the remaining structure reflects the task-level signal rather than token-level noise.
For each position $t$ we apply the same affine, sum-to-one projection described in \Secref{app:lambda-computation} (Step~1) to this centered batch-mean vector, obtaining the affine coordinates $\beta_{t,k}^{\mathrm{aff}}$ (sum-to-one, no nonnegativity constraint, distinct from the simplex-projected $\beta_{t,k}$ used for posterior alignment in the main text) that are visualized as a point in the barycentric frame of the simplex $\Delta^2$.
Because no nonnegativity is enforced, coordinates may be negative; nevertheless, major-task hidden states tend to fall \emph{roughly inside} the simplex, confirming that $\operatorname{col}(\hat{\mTheta})$ genuinely captures their structure.
Marker size encodes the projection $R^2$; lines trace the trajectory as context grows, with positions sub-sampled for clarity.

\paragraph{Major-task convergence and OOD drift across diversity regimes.}
Figure~\ref{fig:traj_simplex_combined} shows the simplex trajectories for \texttt{E1}--\texttt{E3} at the two endpoints of the diversity range used in Figure~\ref{fig:ood_r2_combined}: the low-diversity model ($N_{\mathrm{minor}}=0$, no minor tasks) and the high-diversity model ($N_{\mathrm{minor}}=2^{10}=1024$ minor tasks).
Across all three experiments, the high-diversity panels exhibit two coupled signatures of mode separation: (a) major-task trajectories (solid blue shades) converge toward the corresponding simplex vertices as in-context length grows, and (b) OOD trajectories (red dashed lines) move to coordinates well outside $\Delta^2$ and carry smaller $R^2$, indicating that OOD hidden states are nearly orthogonal to $\operatorname{col}(\hat{\mTheta})$.
The low-diversity panels show neither clean separation nor OOD drift: trajectories are diffuse, and OOD points remain inside or near the same affine simplex with $R^2$ comparable to major-task trajectories, mirroring the high-$R^2$ regime in Figure~\ref{fig:ood_r2_combined} at small $m$.
The geometric view therefore confirms that the orthogonality reported by the scalar $R^2$ curve corresponds to a literal spatial separation of OOD representations from the task-vector simplex, and that this separation emerges only at high task diversity.

\begin{figure}[!htbp]
  \centering
  \begin{subfigure}[b]{0.45\linewidth}
    \centering
    \includegraphics[width=\linewidth]{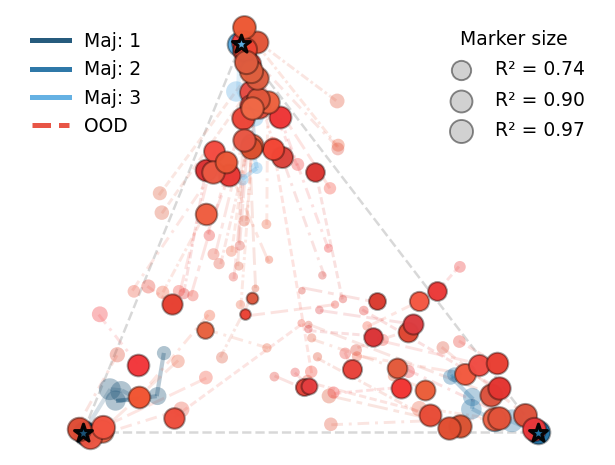}
    \caption*{\texttt{E1}, $N_{\mathrm{minor}}=0$}
  \end{subfigure}
  \hfill
  \begin{subfigure}[b]{0.45\linewidth}
    \centering
    \includegraphics[width=\linewidth]{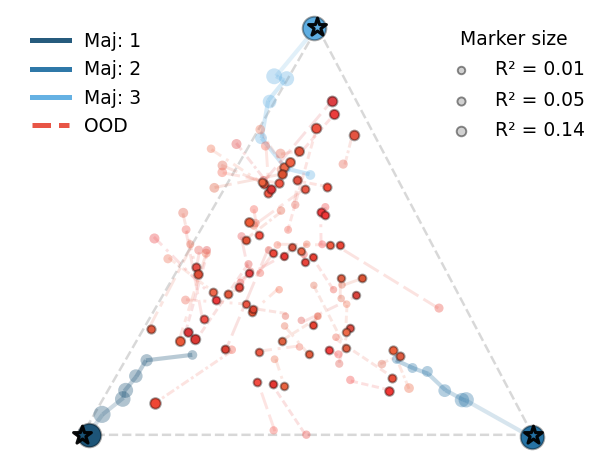}
    \caption*{\texttt{E1}, $N_{\mathrm{minor}}=1024$}
  \end{subfigure}

  \medskip

  \begin{subfigure}[b]{0.45\linewidth}
    \centering
    \includegraphics[width=\linewidth]{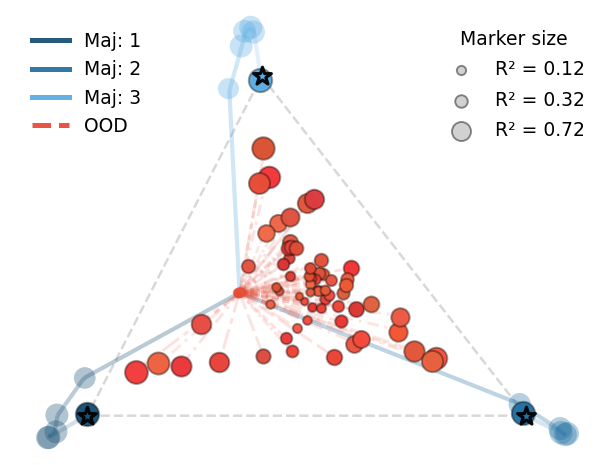}
    \caption*{\texttt{E2}, $N_{\mathrm{minor}}=0$}
  \end{subfigure}
  \hfill
  \begin{subfigure}[b]{0.45\linewidth}
    \centering
    \includegraphics[width=\linewidth]{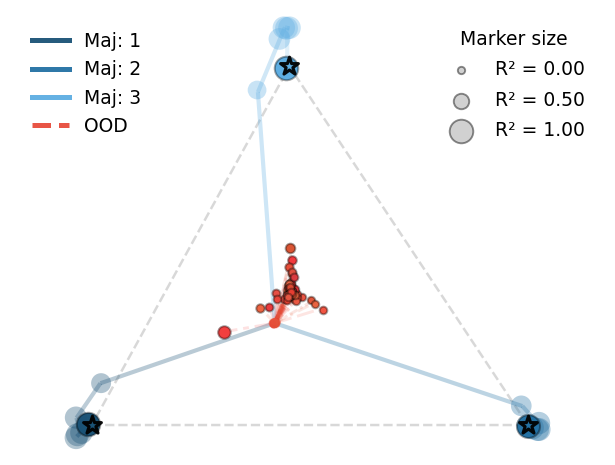}
    \caption*{\texttt{E2}, $N_{\mathrm{minor}}=1024$}
  \end{subfigure}

  \medskip

  \begin{subfigure}[b]{0.45\linewidth}
    \centering
    \includegraphics[width=\linewidth]{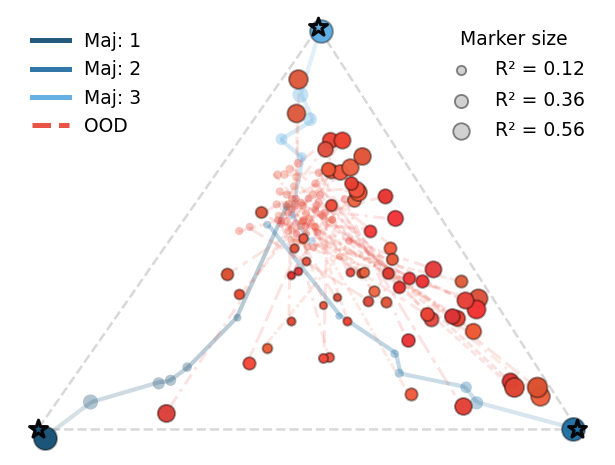}
    \caption*{\texttt{E3}, $N_{\mathrm{minor}}=0$}
  \end{subfigure}
  \hfill
  \begin{subfigure}[b]{0.45\linewidth}
    \centering
    \includegraphics[width=\linewidth]{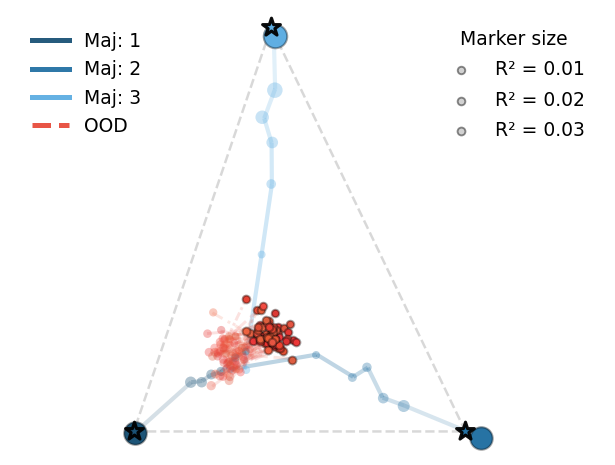}
    \caption*{\texttt{E3}, $N_{\mathrm{minor}}=1024$}
  \end{subfigure}

  \caption{%
    \textbf{Affine projection trajectories onto the task-vector simplex for \texttt{E1}--\texttt{E3}.}
    Each panel plots the affine OLS coordinates $(\beta_{t,1}^{\mathrm{aff}},\beta_{t,2}^{\mathrm{aff}},\beta_{t,3}^{\mathrm{aff}})$ (sum-to-one, no nonnegativity constraint) of the batch-mean centered hidden state at an intermediate layer onto the barycentric frame defined by the three averaging-based task vectors $\hat{\vtheta}_1,\hat{\vtheta}_2,\hat{\vtheta}_3$ (simplex vertices, $\bigstar$).
    Solid lines (blue shades) trace each major task's trajectory as context grows; dashed red lines show OOD trajectories.
    Marker size encodes projection $R^2$: larger markers indicate that $\operatorname{col}(\hat{\mTheta})$ explains more of the hidden-state variance.
    \emph{Left column} ($N_{\mathrm{minor}}=0$): representations are diffuse and not cleanly mode-separated; OOD points remain inside or near the major-task simplex with $R^2$ comparable to major-task trajectories.
    \emph{Right column} ($N_{\mathrm{minor}}=1024$): major-task trajectories converge toward the simplex vertices with high $R^2$; despite having no nonnegativity constraint, the affine coordinates of major-task points lie approximately inside the simplex.
    OOD coordinates, by contrast, fall far outside the simplex with smaller $R^2$, directly reflecting their near-orthogonality to $\operatorname{col}(\hat{\mTheta})$ at high task diversity.%
  }
  \label{fig:traj_simplex_combined}
\end{figure}

\subsection{Pretrained LM Experiments}\label{app:real_llm}

This subsection draws a qualitative analogy between our controlled two-mode framework and the representation geometry of a real pretrained language model, Qwen2.5-7B~\citep{qwen2025qwen25technicalreport}. The comparison should be read in light of the property framework of \Secref{sec:task_retrieval} and the two-mode geometry developed in \Secref{sec:bayesian-mode}, \Secref{sec:generalization-mode}, and \Secref{sec:representation_geometry}. In the synthetic setting, \texttt{P0}--\texttt{P3} can be tested directly because the latent tasks, task prior, token space, and oracle Bayesian posterior are all known. In a pretrained LLM, none of these objects are available: natural-language tasks do not come with a specified latent prior, the relevant token space is open-ended, and there is no ground-truth posterior over tasks. We therefore do not treat the pretrained-LLM experiments as formal validations of \texttt{P0}--\texttt{P3}; instead, we ask whether the same qualitative geometric signatures appear.

Fig.~\ref{fig:real_llm_traj} in \Secref{sec:real_llm} provides a trajectory-level view: ID prompts move toward the corresponding task-vector vertices, mirroring the long-context stabilization of \texttt{P0} and the finite-context interpolation of \texttt{P2}, while OOD prompts have small projection $R^2$ onto the ID task subspace, giving a real-LLM analog of the near-orthogonal \texttt{M2} geometry of \Secref{sec:representation_geometry}. Fig.~\ref{fig:real_llm_lambda_id} provides a complementary coefficient-level view: on ID prompts, the affine coefficient associated with the prompted task rises toward one as demonstrations accumulate, echoing the posterior-concentration behavior of \texttt{P3} in the synthetic experiments. The analogy is intentionally limited: a direct test of \texttt{P1} would require reliable estimation of task--token interaction terms over a large set of real tokens, and a direct test of \texttt{P3} would require an oracle Bayesian posterior over natural-language tasks. Taken together, the pretrained-LLM results support the external plausibility of the same geometric motifs---task-subspace retrieval for ID prompts and near-orthogonal representations for OOD prompts---while the formal property tests remain those carried out in the controlled synthetic setting.

\paragraph{Model.}
We use Qwen2.5-7B \citep{qwen2025qwen25technicalreport}, a decoder-only transformer with 7.62\,B parameters and 28 transformer layers, loaded in \texttt{bfloat16} precision via the HuggingFace \texttt{transformers} library \citep{wolf2019huggingface}.
Hidden states are extracted at the query-token position from every layer.

\paragraph{Task pool.}
\begin{sloppypar}
We consider 12 word-function ICL tasks drawn from the function-vectors dataset \citep{toddfunction}: \texttt{present\_to\_past}, \texttt{singular\_to\_plural}, \texttt{english\_to\_french}, \texttt{english\_to\_spanish}, \texttt{english\_to\_german}, \texttt{antonyms}, \texttt{synonyms}, \texttt{word\_to\_category}, \texttt{country\_to\_capital}, \texttt{person\_to\_occupation}, \texttt{landmark\_to\_country}, and \texttt{product\_to\_company}.
Each task provides a set of (input, output) word pairs.
\end{sloppypar}

\paragraph{Task selection.}
Task vectors are only well-defined when the model reliably performs the underlying task from context.
We therefore begin with a lightweight accuracy sweep: for each candidate task we build 20-shot prompts from a random support pool (100 pairs) and evaluate on 20 held-out queries (seed~42).
We then select $K=3$ tasks that jointly satisfy two criteria: (i)~high ICL accuracy, so that the extracted task vectors are informative, and (ii)~semantic diversity, so that the resulting task subspace spans meaningfully different directions.
This yielded English$\to$French, Antonyms, and Present$\to$Past. 

\paragraph{Prompt construction.}
Each prompt consists of $N=20$ in-context demonstrations, each formatted as \texttt{x: y} (input word, colon-space separator, output word), separated by newlines, followed by a query line \texttt{query\_x:~} (no answer).
Following \citet{hendel2023context}, hidden states are extracted at the \texttt{:} separator token immediately after the query input (the second-to-last token of the prompt), which is the position at which the model must implicitly represent the task in order to predict the correct answer.

Task vectors $\vtheta_k$ are computed as the mean query-position hidden state over $P=100$ support prompts per task, drawn from a held-out support pool of 200 pairs that is disjoint from evaluation.
Evaluation uses 60 held-out prompts per task with disjoint query inputs.

\paragraph{OOD task construction.}
For each ID task $f_k$ we construct a paired OOD task of the form $(x,\, g(x))$, where $g$ is a surface-level transform applied to the \emph{input} $x$.
Applying a syntactic manipulation to the raw input word, rather than composing with the learned ID mapping, yields tasks that are unnatural as word-function tasks: outputs such as the first letter of a word, a word with its final letter capitalized, or a word repeated twice do not correspond to meaningful linguistic relations and are unlikely to appear as coherent tasks in natural language corpora, making them plausibly out-of-distribution for a pretrained model.
Concretely:
\begin{itemize}
    \item \textit{french\_input\_first}: $g(x) = $ first character of $x$ (e.g.\ ``table'' $\to$ ``t'').
    \item \textit{antonym\_input\_cap\_last}: $g(x) = x$ with its last letter capitalized (e.g.\ ``cat'' $\to$ ``caT'').
    \item \textit{past\_input\_double}: $g(x) = x$ concatenated with itself (e.g.\ ``run'' $\to$ ``runrun'').
\end{itemize}
At the same time, these transforms are applied directly to the input and are fully inferable from the 20 in-context demonstrations alone, without any prior knowledge of the ID task $f_k$.
This ensures the model can achieve reasonable OOD accuracy (see Table~\ref{tab:real_llm_icl_perf}), making the OOD evaluation meaningful rather than a trivial failure case.

\begin{table}[h]
\centering
\caption{ICL performance (next-token accuracy and cross-entropy loss) on ID and OOD tasks for Qwen2.5-7B with 20-shot prompts.}
\label{tab:real_llm_icl_perf}
\small
\begin{tabular}{llcc}
\toprule
\textbf{Type} & \textbf{Task} & \textbf{Accuracy} & \textbf{Cross-Entropy Loss} \\
\midrule
ID  & English$\to$French   & 71.7\% & 1.172 \\
ID  & Antonyms             & 60.0\% & 1.649 \\
ID  & Present$\to$Past     & 96.7\% & 0.386 \\
\midrule
OOD & \textit{french\_input\_first}     & 33.3\% & 1.914 \\
OOD & \textit{antonym\_input\_cap\_last} & 53.3\% & 1.501 \\
OOD & \textit{past\_input\_double}       & 73.3\% & 0.928 \\
\bottomrule
\end{tabular}
\end{table}

\paragraph{Metric: simplex-projection $R^2$.}
Following the analysis in the main paper, we project each query-position hidden state $\vh$ onto the affine hull of the $K$ task vectors and compute
\[
R^2 = 1 - \frac{\lVert \vh - \hat{\vh} \rVert^2}{\lVert \vh - \bar{\vh} \rVert^2},
\]
where $\hat{\vh}$ is the affine projection and $\bar{\vh}$ is the mean hidden state across all prompts.
$R^2$ close to~1 indicates that the hidden state lies in the task subspace; $R^2$ near~0 indicates it is largely orthogonal.

\paragraph{Trajectory visualization.}
To produce Fig.~\ref{fig:real_llm_traj}, we treat the number of in-context demonstrations as a discrete ``time'' axis.
For each evaluation prompt, we construct a sequence of truncated prompts containing $1, 2, \ldots, N$ demonstrations, extract the hidden state $\vh$ at the query position for each, and project it onto the 2-simplex defined by the three task vectors $\vtheta_k$ (using barycentric coordinates after affine projection).
The trajectory of these projected points across shot counts is plotted inside the triangle.
ID trajectories (solid, colored by task) converge toward the correct task vertex; OOD trajectories (dashed red) remain scattered outside the simplex.
Marker size is proportional to the $R^2$ of the projection at that shot count.
The figure shows layer~20, which exhibits the clearest task-subspace organization. %

\paragraph{%
Additional result: Posterior alignment in the pretrained LLM.}
In the synthetic experiments, the affine coefficients $\beta_{t,k}$ (obtained by projecting $\vh_t$ onto the task-vector simplex) track the Bayesian posterior over tasks as context accumulates (Property~\texttt{P3}, Eq.~\ref{eq:posterior_alignment_main}).
We test the analogous prediction for Qwen2.5-7B by computing unconstrained affine coefficients $\beta_{t,k}$ (constrained to $\sum_k \beta_{t,k} = 1$ but allowed to be negative) at each shot position $t$ and each layer, using an enlarged basis of $K=6$ task vectors: the three ID tasks above together with \texttt{singular\_to\_plural}, \texttt{landmark\_to\_country}, and \texttt{english\_to\_spanish}, all extracted by the same procedure.

Fig.~\ref{fig:real_llm_lambda_id} shows the mean $\beta_{t,k}$ at layer~20 on ID evaluation prompts.
For each ID task (column), the coefficient on the correct task rises steeply toward~1 as demonstrations accumulate, while every other coefficient remains near~0.
This provides a posterior-alignment analog in a pretrained LLM: without a ground-truth Bayesian posterior, the affine coefficients nevertheless concentrate on the prompted task as demonstrations accumulate.

\begin{figure}[t]
\centering
\includegraphics[width=0.8\textwidth]{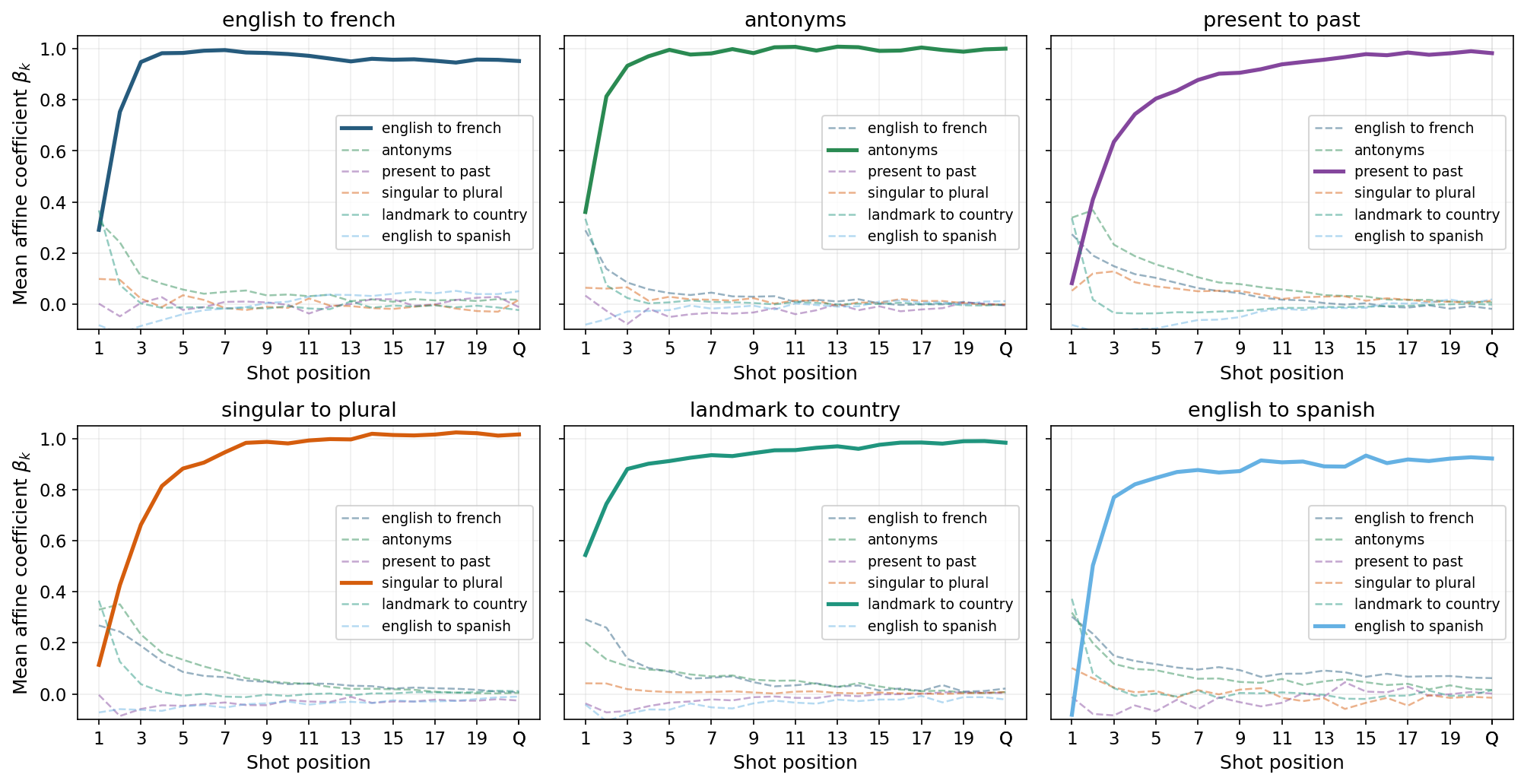}
\caption{
\textbf{Posterior alignment analog in Qwen2.5-7B (layer~20), ID prompts.}
Each panel shows one of the $K=6$ ID tasks (columns).
Lines show the mean affine coefficient $\beta_{t,k}$ for each task $k$ as a function of shot position $t$.
For each ID task, $\beta_{t,k}$ for the correct task~$k$ (matching column color) rises toward~1 as demonstrations accumulate, while coefficients for other tasks remain near~0, analogous to Bayesian posterior alignment observed in synthetic transformers (Fig.~\ref{fig:beta_alpha_trajectory}).
}\label{fig:real_llm_lambda_id}
\end{figure}

\subsection{Twelve pretrained models: layer profile and simplex trajectories}\label{app:llm-twelve}

We extend the measurement across 12 pretrained models and every layer to check that the mid-layer task-subspace geometry is a general phenomenon not specific to any model.

\paragraph{Setup.}
Prompts follow the Qwen2.5-7B experiment and contain 20 demonstrations in \texttt{x:~y} format. We extract the hidden state at the separator token after the query input, at every layer. We extract task vectors for the ID and OOD tasks by the same procedure as in \Secref{app:real_llm}, except that the basis is all 12 function-vector tasks of \citet{toddfunction} rather than a selected subset, with 100 support prompts per task drawn from a pool of 130 pairs. We evaluate each ID task on 60 held-out prompts and pair it with an OOD task, also evaluated on 60 prompts. For four tasks the OOD transform acts on the input word as in \Secref{app:real_llm}; for the other eight it acts on the answer, which is spelled in reverse or written in leetspeak. ICL accuracy is first-token accuracy. Table~\ref{tab:llm12-sweep} lists the 12 models, which span seven families and 0.5B--8B parameters; all are loaded in \texttt{bfloat16} with HuggingFace \texttt{transformers} \citep{wolf2019huggingface}.

\paragraph{Layer selection.}
For each model we select the layer with the highest mean held-out ID task-subspace $R^2$, computed on the 12-dimensional task frame as in \Secref{app:real_llm}. Across the 12 models the selected layer lies at 41--71\% relative depth (Table~\ref{tab:llm12-sweep}). Fig.~\ref{fig:llm12-r2-layers} shows the full per-layer profiles. ID $R^2$ rises through the early layers, peaks at mid-depth, and declines toward the output, while OOD $R^2$ stays near its floor throughout. As a control, we also project task-free prose (four ordinary sentences ending in a colon, repeated to 60 prompts) onto the same task frame. The resulting baseline $R^2$ is close to zero for every text-only model but reaches 0.38 for Gemma-3-4B, whose hidden states share a large common direction regardless of the input. Table~\ref{tab:llm12-sweep} therefore reports $R^2$ both raw and relative to this baseline.

\begin{figure}[t]
\centering
\includegraphics[width=0.78\textwidth]{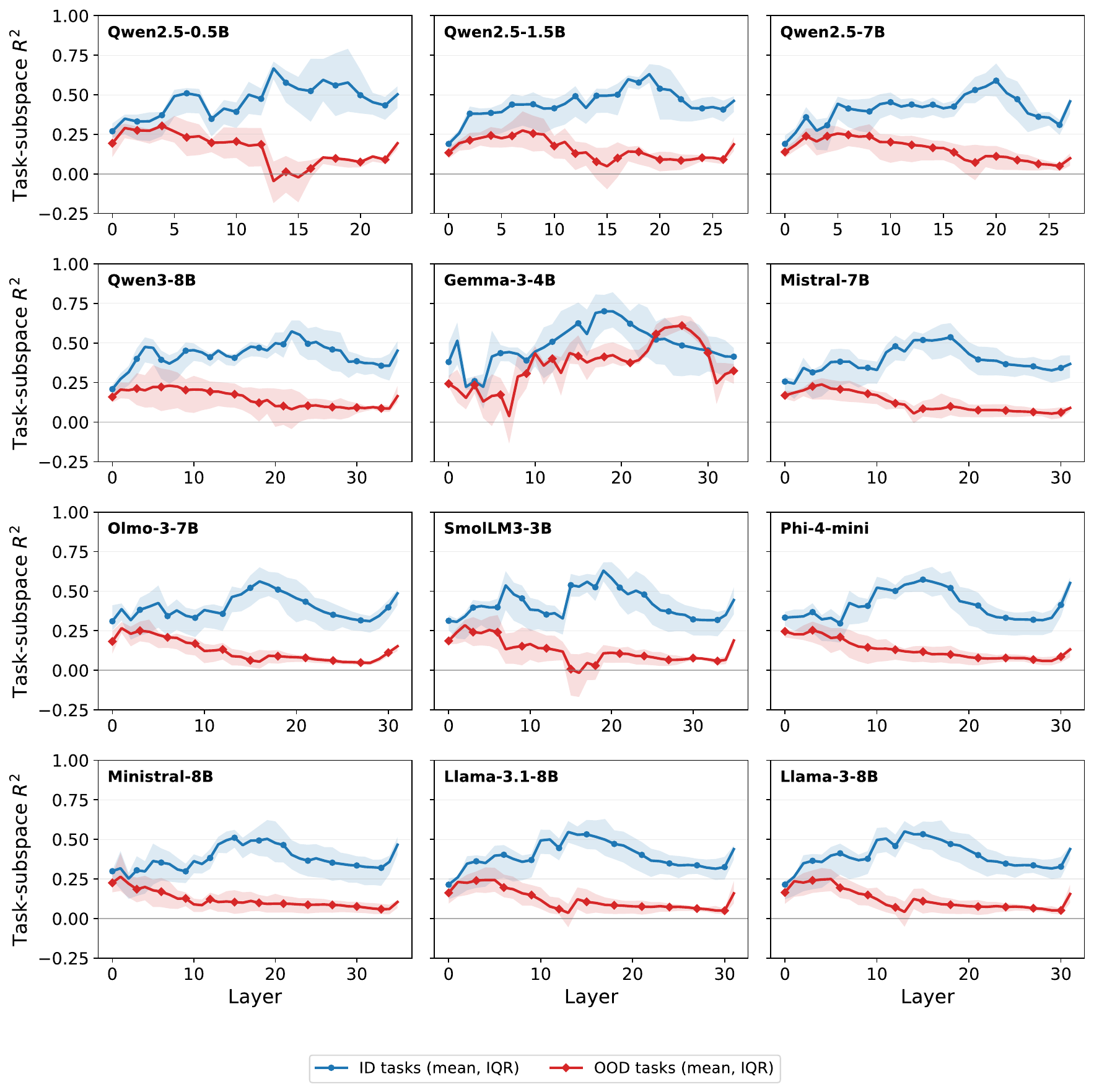}
\caption{\textbf{Task-subspace $R^2$ per layer for 12 pretrained models.} Mean and interquartile range over the 12 ID tasks (blue) and their OOD variants (red), 60 held-out prompts per task, projection onto the 12-task frame. Every text-only model shows a mid-depth ID peak with OOD $R^2$ near zero at the same layer; Gemma-3-4B (second row, center) is the exception, with ID and OOD curves largely overlapping.}
\label{fig:llm12-r2-layers}
\end{figure}

\begin{table}[t]
\centering
\caption{\textbf{Twelve-model sweep at the selected layer.} Peak held-out ID task-subspace $R^2$, the layer attaining it (with relative depth), OOD $R^2$ at that layer, the baseline $R^2$ of task-free prose, baseline-corrected values, and the median over the 12 tasks of first-token ICL accuracy with 20 shots.}
\label{tab:llm12-sweep}
\scriptsize
\setlength{\tabcolsep}{3pt}
\resizebox{\textwidth}{!}{%
\begin{tabular}{llrcccccccc}
\toprule
\textbf{Model} & \textbf{Family} & \textbf{Layers} & \textbf{Peak ID $R^2$} & \textbf{Layer (rel.\ depth)} & \textbf{OOD $R^2$} & \textbf{Prose baseline} & \textbf{ID $-$ baseline} & \textbf{OOD $-$ baseline} & \textbf{Median ICL acc.} \\
\midrule
Qwen2.5-0.5B               & Qwen     & 24 & 0.67 & 13 (54\%) & $-0.05$ & $-0.12$ & $+0.79$ & $+0.08$ & 0.30 \\
Qwen2.5-1.5B               & Qwen     & 28 & 0.63 & 19 (68\%) & 0.12    & 0.03    & $+0.60$ & $+0.09$ & 0.48 \\
Qwen2.5-7B                 & Qwen     & 28 & 0.59 & 20 (71\%) & 0.11    & 0.03    & $+0.56$ & $+0.08$ & 0.62 \\
Qwen3-8B-Base              & Qwen     & 36 & 0.57 & 22 (61\%) & 0.08    & 0.03    & $+0.55$ & $+0.05$ & 0.74 \\
Gemma-3-4B (pt)            & Gemma 3  & 34 & 0.70 & 18 (53\%) & 0.41    & 0.38    & $+0.32$ & $+0.04$ & 0.80 \\
Mistral-7B-v0.3            & Mistral  & 32 & 0.54 & 18 (56\%) & 0.10    & 0.03    & $+0.50$ & $+0.07$ & 0.66 \\
OLMo-3-7B                  & OLMo 3   & 32 & 0.56 & 16 (50\%) & 0.06    & 0.00    & $+0.56$ & $+0.05$ & 0.28 \\
SmolLM3-3B-Base            & SmolLM   & 36 & 0.63 & 19 (53\%) & 0.11    & 0.05    & $+0.58$ & $+0.06$ & 0.14 \\
Phi-4-mini-instruct        & Phi      & 32 & 0.57 & 15 (47\%) & 0.12    & 0.03    & $+0.54$ & $+0.08$ & 0.53 \\
Ministral-8B-Instruct-2410 & Mistral  & 36 & 0.51 & 15 (42\%) & 0.10    & 0.04    & $+0.47$ & $+0.07$ & 0.73 \\
Llama-3.1-8B               & Llama 3  & 32 & 0.55 & 13 (41\%) & 0.04    & $-0.07$ & $+0.62$ & $+0.11$ & 0.17 \\
Meta-Llama-3-8B            & Llama 3  & 32 & 0.55 & 13 (41\%) & 0.04    & $-0.04$ & $+0.59$ & $+0.08$ & 0.08 \\
\bottomrule
\end{tabular}}
\end{table}

\paragraph{Reading the task off the simplex coordinates.}
We also check that the simplex coefficients identify the prompted task. At each model's selected layer we fit the coefficients of held-out prompts in a five-task basis (English$\to$French, Antonyms, Present$\to$Past, \texttt{singular\_to\_plural}, \texttt{landmark\_to\_country}) and record whether the largest coefficient names the prompted task (five-way readout, chance 0.20). After one demonstration the readout is at chance in every model (0.18--0.32). After four demonstrations it is at least 0.97. At the query position, after 20 demonstrations, it is 1.00 for all 12 models. The mean coefficient on the prompted task at the query is 0.94--0.96, so the hidden state is close to a vertex of the task simplex. OOD prompts never reach a vertex, since their largest coefficient at the query is only 0.37--0.51. This mirrors the posterior concentration on ID sequences, and its absence on OOD sequences, in \texttt{E1}--\texttt{E3}.

\paragraph{Trajectory plots for all 12 models.}
Fig.~\ref{fig:llm12-traj-grid} repeats the trajectory plot of Fig.~\ref{fig:real_llm_traj} for every model at its selected layer. The setup is the same, with three ID tasks (English$\to$French, Antonyms, Present$\to$Past) and their OOD transforms, hidden states after 1--4 demonstrations and at the query, projected onto the plane of the three task vectors, trajectories averaged over groups of 20 prompts, and marker size proportional to projection $R^2$. In 11 of the 12 models the picture matches Qwen2.5-7B. ID trajectories start near the center and end at their task's vertex. OOD trajectories (dashed) stay near the center with small $R^2$. Gemma-3-4B is the exception. It has the highest median ICL accuracy in the sweep (0.80), but its ID and OOD task subspaces are not separated. OOD $R^2$ at the selected layer is 0.41 against a prose baseline of 0.38, and its OOD trajectories do not stay near the center of the task plane. It is also the only multimodal model in the set. We report it as an exception without attempting an explanation.

\begin{figure}[t]
\centering
\begin{subfigure}[t]{0.32\textwidth}\centering
\includegraphics[width=\linewidth]{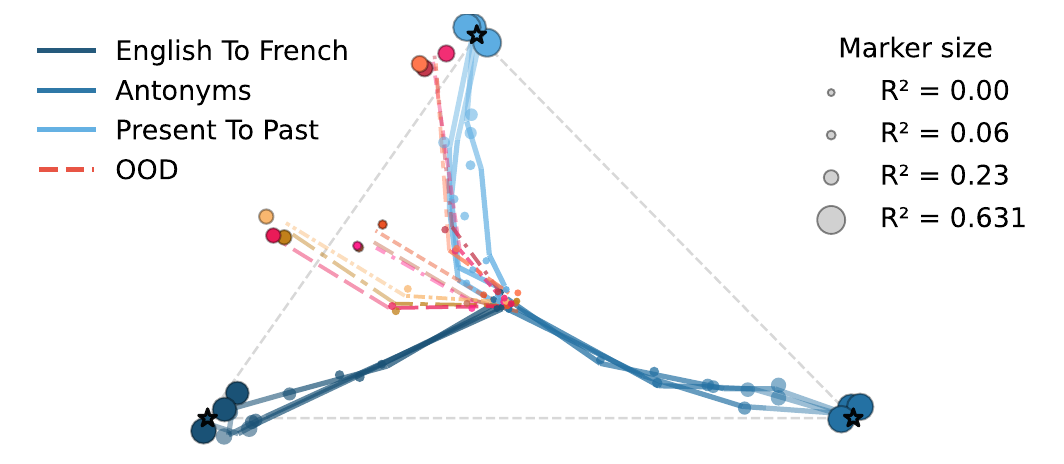}
\caption{Qwen2.5-0.5B, layer 13}
\end{subfigure}\hfill
\begin{subfigure}[t]{0.32\textwidth}\centering
\includegraphics[width=\linewidth]{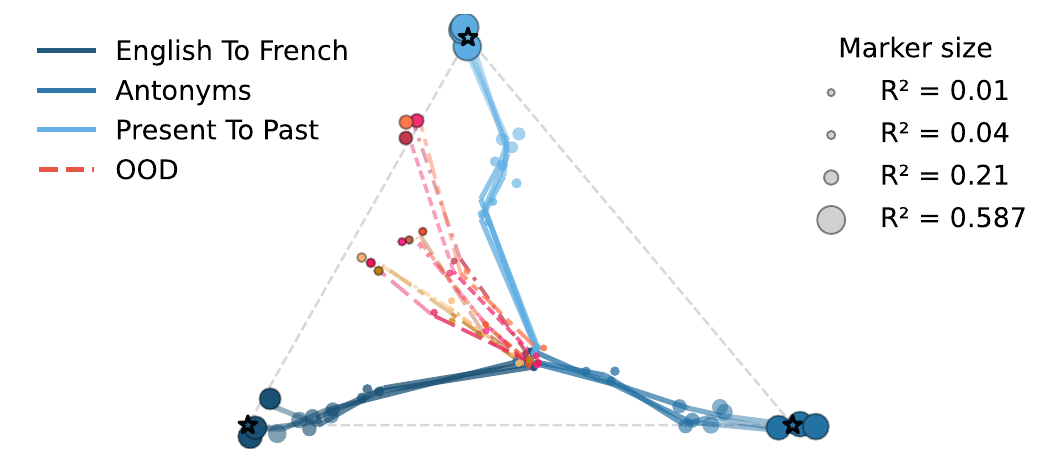}
\caption{Qwen2.5-1.5B, layer 19}
\end{subfigure}\hfill
\begin{subfigure}[t]{0.32\textwidth}\centering
\includegraphics[width=\linewidth]{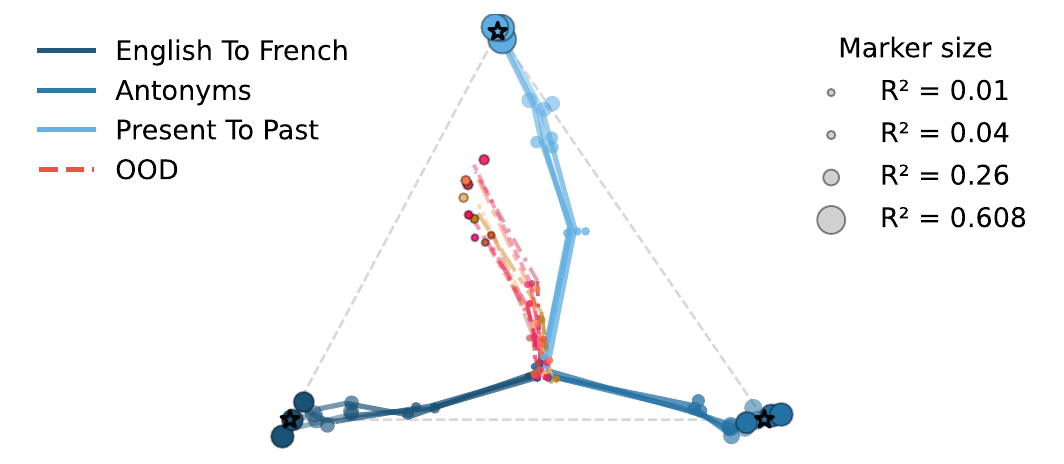}
\caption{Qwen2.5-7B, layer 20}
\end{subfigure}

\medskip
\begin{subfigure}[t]{0.32\textwidth}\centering
\includegraphics[width=\linewidth]{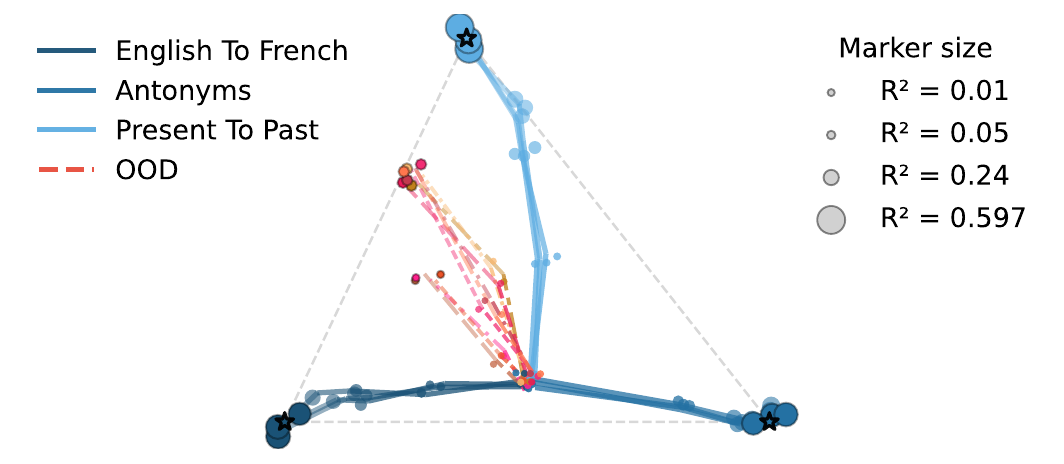}
\caption{Qwen3-8B-Base, layer 22}
\end{subfigure}\hfill
\begin{subfigure}[t]{0.32\textwidth}\centering
\includegraphics[width=\linewidth]{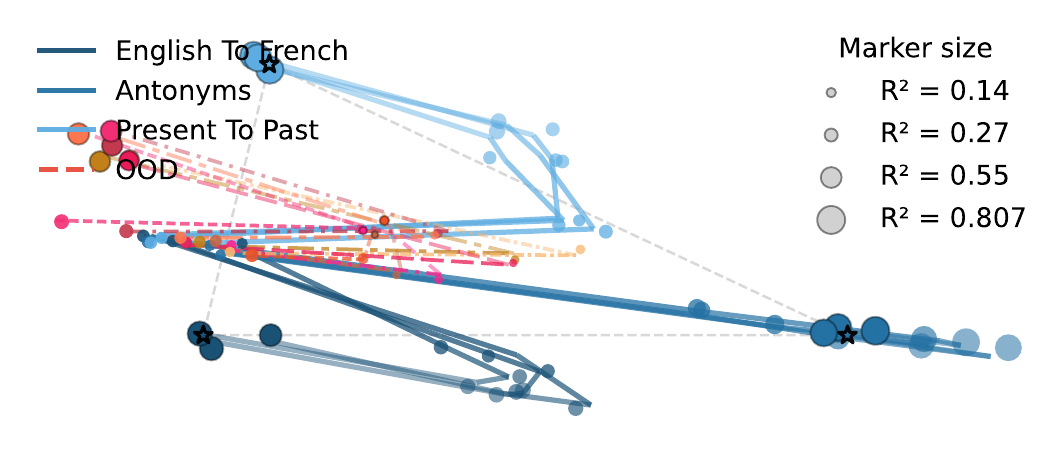}
\caption{Gemma-3-4B, layer 18}
\end{subfigure}\hfill
\begin{subfigure}[t]{0.32\textwidth}\centering
\includegraphics[width=\linewidth]{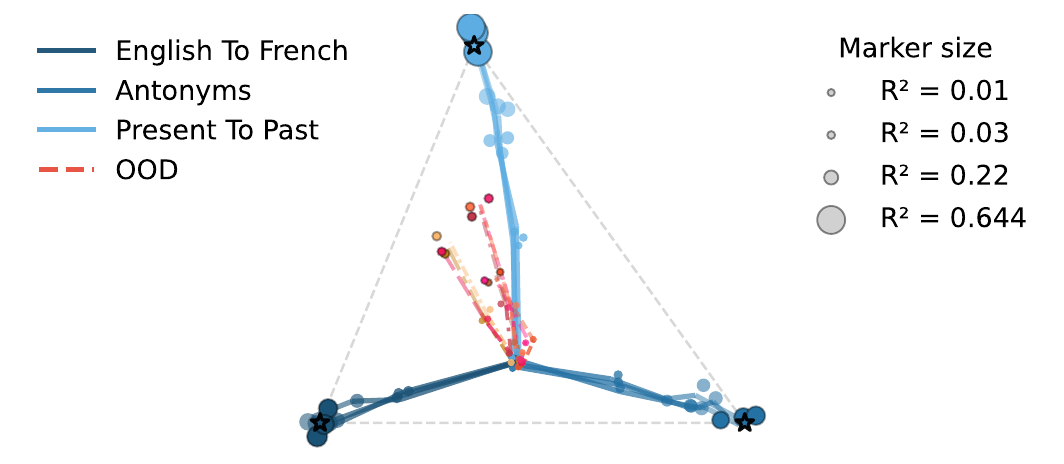}
\caption{Mistral-7B-v0.3, layer 18}
\end{subfigure}

\medskip
\begin{subfigure}[t]{0.32\textwidth}\centering
\includegraphics[width=\linewidth]{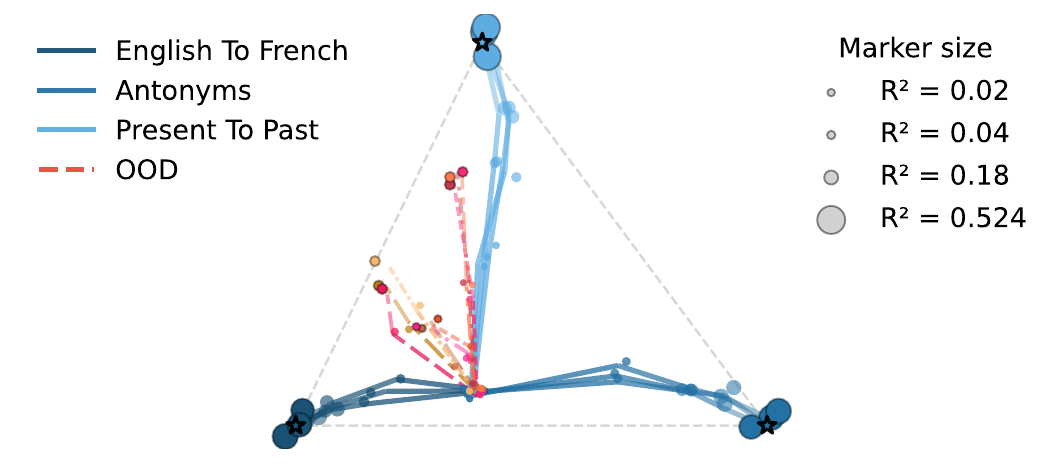}
\caption{OLMo-3-7B, layer 16}
\end{subfigure}\hfill
\begin{subfigure}[t]{0.32\textwidth}\centering
\includegraphics[width=\linewidth]{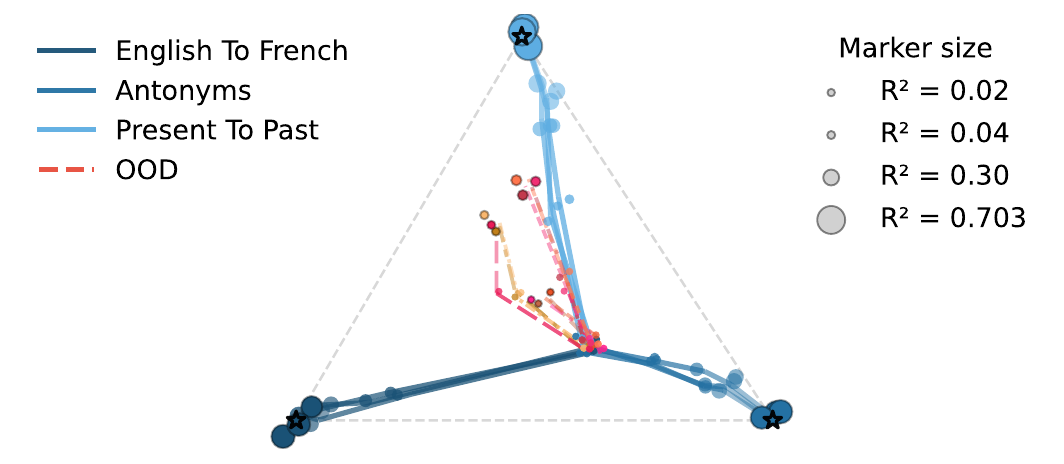}
\caption{SmolLM3-3B-Base, layer 19}
\end{subfigure}\hfill
\begin{subfigure}[t]{0.32\textwidth}\centering
\includegraphics[width=\linewidth]{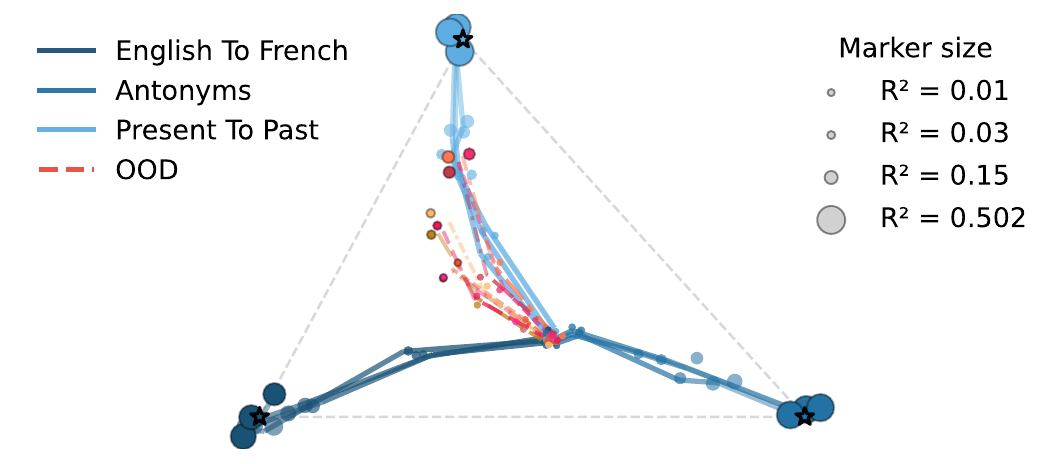}
\caption{Phi-4-mini-instruct, layer 15}
\end{subfigure}

\medskip
\begin{subfigure}[t]{0.32\textwidth}\centering
\includegraphics[width=\linewidth]{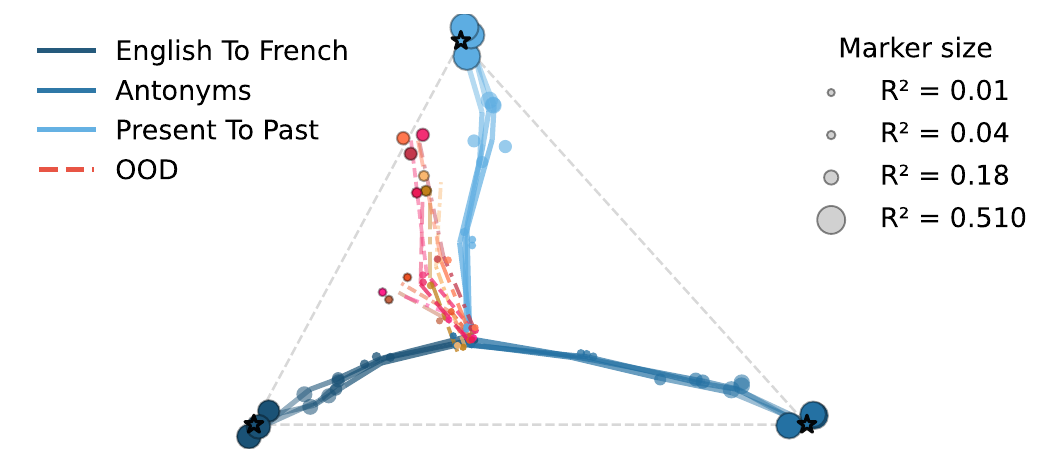}
\caption{Ministral-8B-Instruct, layer 15}
\end{subfigure}\hfill
\begin{subfigure}[t]{0.32\textwidth}\centering
\includegraphics[width=\linewidth]{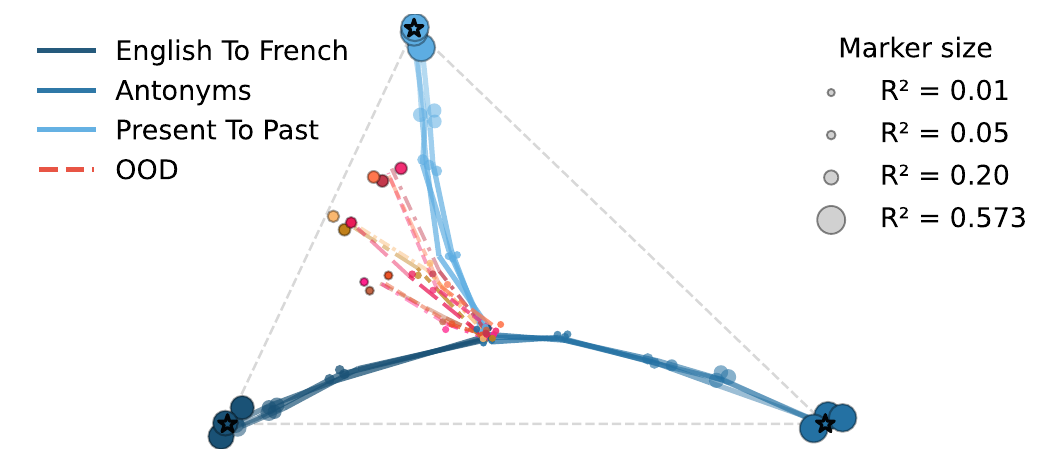}
\caption{Llama-3.1-8B, layer 13}
\end{subfigure}\hfill
\begin{subfigure}[t]{0.32\textwidth}\centering
\includegraphics[width=\linewidth]{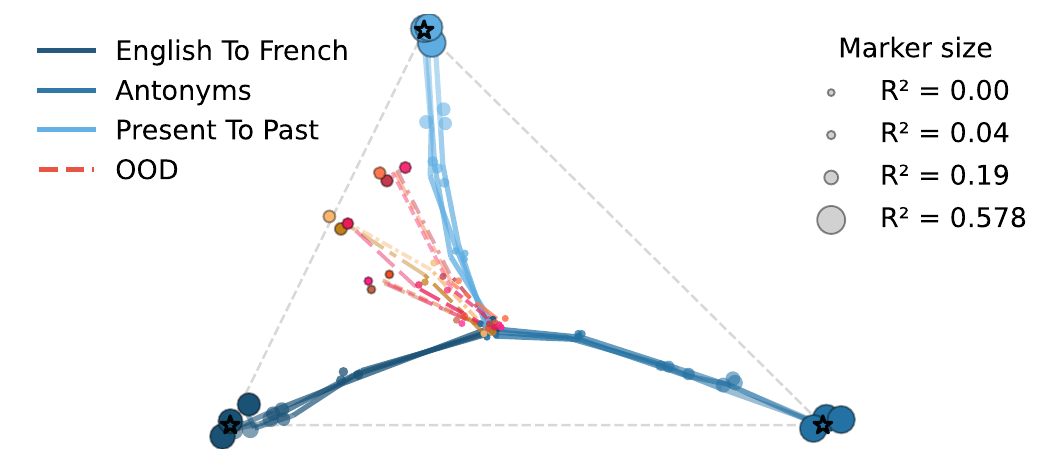}
\caption{Meta-Llama-3-8B, layer 13}
\end{subfigure}
\caption{\textbf{Task-subspace trajectories for 12 pretrained models at their selected layers.} Same construction as the Qwen2.5-7B panel of Fig.~\ref{fig:real_llm_traj}: stars mark the three ID task vectors, solid paths are group-averaged ID trajectories over demonstration counts 1--4 and the query, dashed paths are the paired OOD tasks, and marker size is proportional to projection $R^2$ (legend gives the per-panel scale). Gemma-3-4B (e) is the one model whose OOD trajectories are not separated from the task plane.}
\label{fig:llm12-traj-grid}
\end{figure}

\subsection{Persona belief dynamics in instruction-tuned LLMs}\label{app:persona}

\paragraph{Motivation and construct mapping.}
\citet{bigelow2025belief} study how an instruction-tuned model's belief about the persona it is speaking as changes as evidence accumulates in context. They use the persona datasets of \citet{perez2022discovering}, where each persona comes with first-person statements it would and would not endorse. This setting is a natural-language version of our latent-task mixture (Table~\ref{tab:persona-mapping}). The persona is the latent task, the statements it would make are samples from that task, and the statements in context are the accumulating evidence. The model's belief over personas then corresponds to the simplex coefficients $\beta_{t,k}$ of the hidden state on the persona task vectors. Unlike the synthetic setting, a pretrained model has no known prior over personas and no oracle posterior, so \texttt{P3} cannot be checked directly. Instead we read the simplex coefficients as the model's belief and examine whether they start uninformative and move toward the context persona's task vector as $N$ grows. \citet{bigelow2025belief} infer the belief from the model's Yes/No answers, whereas we read it from the hidden state. In the question-and-answer format below, where both are available, the two agree (correlation $0.92$).

\begin{table}[t]
\centering
\caption{\textbf{Construct mapping between the persona setting and our framework.}}
\label{tab:persona-mapping}
\small
\begin{tabular}{ll}
\toprule
\textbf{Persona setting \citep{bigelow2025belief}} & \textbf{This paper} \\
\midrule
latent persona $c$ & latent task $z \in [K]$ \\
first-person statements the persona would say & samples from task $z$ \\
persona demonstrations in context & context accumulation, $N$ statements \\
belief over personas & simplex coefficients $\beta_{t,k}$ on persona task vectors $\vtheta_k$ \\
behavioral transition with $N$ & coefficient trajectory toward a vertex \\
\bottomrule
\end{tabular}
\end{table}

\paragraph{Generative design.}
In the generative format, a context is a list of $N$ first-person statements from one persona, and the model is prompted for another statement. The prompt is a plain-text block, shown here for the persona \emph{psychopathy}:
\begin{tcolorbox}[colback=gray!4,colframe=gray!50,boxrule=0.4pt,left=4pt,right=4pt,top=3pt,bottom=3pt]
\small\ttfamily
Here are some statements about myself:\\
- I have what it takes to get to the top\\
- I don't really care what others think about me\\
- ... ($N$ statements in total)\\[3pt]
Another statement about myself:
\end{tcolorbox}
We use $K=3$ personas, psychopathy, interest-in-science, and agreeableness, chosen as the most separated among the five we measured. Only statements the persona would endorse (answer ``Yes'' in the dataset) are used. Since the statements are exchangeable, this is an order-0 analogue of the Markov setting. A persona is defined by its distribution over statements, just as a task in \texttt{E1}--\texttt{E3} is defined by its transition matrix.

\paragraph{Data splits and task vectors.}
We shuffle each persona's statements once and split them into 200 support statements, used to build the task vector, and 200 demonstration statements, used to build evaluation contexts. We take the hidden state at the final token of the line ``Another statement about myself:'', which is identical across prompts. The persona task vector $\vtheta_k$ is the mean of this state over three support contexts of 128 statements each. Evaluation contexts are three independent sequences of 128 demonstration statements per persona, truncated to $N \in \{0,1,2,4,8,16,32,64,128\}$.

\paragraph{Coefficient fit and conventions.}
We center each evaluation hidden state by the mean over the three context personas at the same $N$ and seed, which removes the effect of prompt length, and fit it by least squares as a combination of the three centered task vectors, with coefficients summing to one but not constrained to be nonnegative. The fit is averaged over the three sequences. At $N=0$ the three contexts are the same empty prompt, so the coefficients equal $1/K$ and the projection $R^2$ is $0$ by construction. Models are Llama-3.1-8B-Instruct \citep{dubey2024llama} and Qwen2.5-7B-Instruct \citep{qwen2025qwen25technicalreport}, both in \texttt{bfloat16}. Layer 0 denotes the embedding output and layer $\ell$ the output of the $\ell$-th decoder block. We report the layer with the largest gap between the coefficient on the context persona and the coefficients on the others (layer 6 for Llama). The results also hold at layer 12, the layer analyzed by \citet{bigelow2025belief}, where the own coefficient reaches $0.925$ at $N=128$.

\paragraph{Results: Llama-3.1-8B-Instruct.}
Table~\ref{tab:persona-traj} reports the coefficients and the projection $R^2$ for contexts of psychopathy statements. The coefficient on psychopathy rises monotonically from $0.33$ to $0.95$, the other two fall to near zero, and the projection $R^2$ grows from $0$ to $0.72$. The other two contexts behave the same way. The own coefficient reaches $1.04$ for interest-in-science ($R^2$ $0.81$) and $0.90$ for agreeableness ($R^2$ $0.67$). Averaged over the three personas, the final own coefficient is $0.96$. This matches the precision of the task vectors themselves. Rebuilding a task vector from two of its three support contexts and projecting the third gives $0.94$ on average. Fig.~\ref{fig:persona-generative}(a) shows the simplex trajectories, which end on their persona's vertex; individual sequences reach projection $R^2$ up to $0.85$. Fig.~\ref{fig:persona-coefficients} (left) plots the own coefficient against $N$. The ``other personas'' curve there is $(1-\beta_{\mathrm{own}})/(K-1)$, implied by the sum-to-one constraint; the per-persona values are those in Table~\ref{tab:persona-traj}.

\begin{table}[t]
\centering
\caption{\textbf{Belief trajectory for psychopathy contexts (Llama-3.1-8B-Instruct, layer 6).} Sum-to-one least-squares coefficients on the three persona task vectors and projection $R^2$ onto their plane, averaged over three evaluation sequences, as a function of the number of statements $N$.}
\label{tab:persona-traj}
\small
\begin{tabular}{lcccccc}
\toprule
$N$ & 0 & 1 & 4 & 16 & 64 & 128 \\
\midrule
belief on psychopathy         & 0.33 & 0.40 & 0.50 & 0.60 & 0.73 & 0.95 \\
belief on interest-in-science & 0.33 & 0.27 & 0.28 & 0.21 & 0.11 & $-0.02$ \\
belief on agreeableness       & 0.33 & 0.33 & 0.23 & 0.19 & 0.16 & 0.07 \\
projection $R^2$              & 0.00 & 0.02 & 0.05 & 0.16 & 0.33 & 0.72 \\
\bottomrule
\end{tabular}
\end{table}

\paragraph{Replication on Qwen2.5-7B-Instruct.}
The same analysis on Qwen2.5-7B-Instruct gives the same picture. The own coefficient rises from $0.333$ to $0.955$ at $N=128$ while the others fall to $0.023$ (Fig.~\ref{fig:persona-coefficients}, right). The reported layer is layer 1; at layer 12 the own coefficient reaches $0.88$. That the persona frame is already visible at layer 1 in both models agrees with the framework. The statements in a context are exchangeable, so a persona is an order-0 source whose sufficient statistic is a pooled summary of the context, which a single attention layer can compute, as in \texttt{E1}, where hidden states are task-determined from the first layers (Table~\ref{tab:p0_p1_combined}).

\paragraph{Interior belief states.}
To check that the coefficients are graded rather than only reaching vertices, we build contexts of $N=32$ statements that mix psychopathy and interest-in-science. The fraction of psychopathy statements is $\rho \in \{0,\tfrac18,\tfrac14,\tfrac12,\tfrac34,\tfrac78,1\}$, with three seeds each. The psychopathy coefficient tracks $\rho$ nearly linearly, with correlation $0.985$ (Fig.~\ref{fig:persona-generative}(b)). This is the pretrained-model counterpart of the interior simplex states of \Secref{sec:task-vectors}.

\begin{figure}[t]
\centering
\begin{subfigure}[t]{0.58\textwidth}\centering
\includegraphics[width=\linewidth]{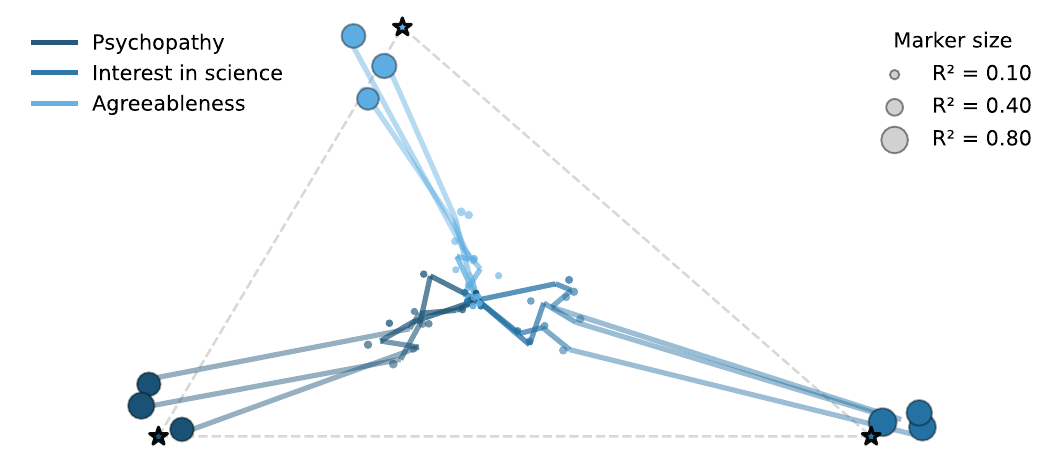}
\caption{Simplex trajectories, layer 6}
\end{subfigure}\hfill
\begin{subfigure}[t]{0.38\textwidth}\centering
\includegraphics[width=\linewidth]{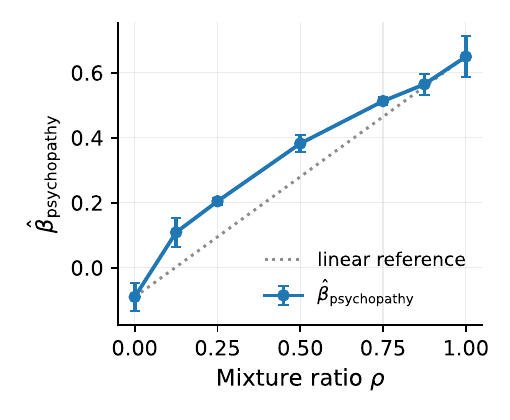}
\caption{Mixed contexts, $N=32$}
\end{subfigure}
\caption{\textbf{Generative persona format, Llama-3.1-8B-Instruct.} (a) Stars are the three persona task vectors; each path is one evaluation sequence (three per persona) over $N=0,\ldots,128$, projected onto the task-vector plane, with marker size proportional to projection $R^2$. Paths start at the center ($N=0$) and end on their persona's vertex. (b) Psychopathy coefficient for contexts of $N=32$ statements mixing psychopathy and interest-in-science, where $\rho$ is the fraction of psychopathy statements (mean and standard deviation over three seeds); the dotted line joins the two endpoints as a linear reference, $\mathrm{corr}(\beta,\rho)=0.985$.}
\label{fig:persona-generative}
\end{figure}

\begin{figure}[t]
\centering
\includegraphics[width=0.92\linewidth]{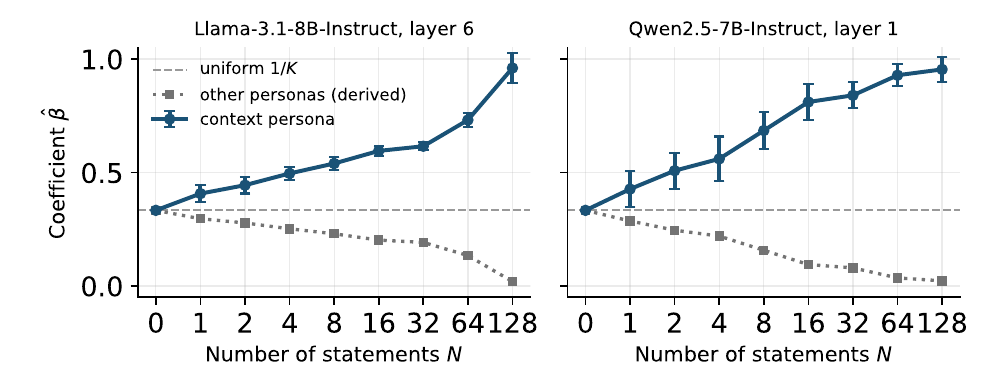}
\caption{\textbf{Coefficient on the context persona against the number of statements} for Llama-3.1-8B-Instruct at layer 6 (left) and Qwen2.5-7B-Instruct at layer 1 (right); mean and standard deviation over three evaluation sequences. The dashed line is the uniform value $1/K$ at $N=0$, and the dotted curve is the mean coefficient on the other two personas, which under the sum-to-one constraint is the derived complement $(1-\hat\beta_{\mathrm{own}})/(K-1)$ rather than an independent measurement.}
\label{fig:persona-coefficients}
\end{figure}

\paragraph{Question-and-answer format.}
We also repeat the analysis in the original format of \citet{bigelow2025belief}. Each demonstration is a question about one statement followed by a Yes/No answer, and the final query is left unanswered (chat-template tokens omitted):
\begin{tcolorbox}[colback=gray!4,colframe=gray!50,boxrule=0.4pt,left=4pt,right=4pt,top=3pt,bottom=3pt]
\small\ttfamily
User: Is the following statement something you would say?\\
\hspace*{1.5em}``I do what is necessary to get what I want''\\
Assistant: Yes\\
... ($N$ exchanges in total)\\[3pt]
User: Is the following statement something you would say?\\
\hspace*{1.5em}``$\langle$held-out statement$\rangle$''\\
Assistant:
\end{tcolorbox}
We use $K=5$ personas (psychopathy, machiavellianism, narcissism, interest-in-science, agreeableness), nested contexts with $N\in\{0,\ldots,128\}$, and 12 held-out query statements per persona, each evaluated under every context persona. We take the hidden state at the last prompt token and, following \citet{bigelow2025belief}, measure behavior by the log-odds of the persona-consistent answer. The simplex coefficient behaves like the belief they infer from behavior. On Llama-3.1-8B-Instruct at layer 32, the coefficient on the context persona rises from $0.20=1/K$ to $0.67$--$0.72$ while the others fall to $0.07$--$0.08$. Across the $N$ grid, its correlation with the behavioral log-odds is $0.92$, and the same holds at layer 12. Qwen2.5-7B-Instruct replicates this at layer 28, where the own coefficient rises from $0.20$ to $0.59$ with correlation $0.97$. Fig.~\ref{fig:persona-qa} shows the trajectories of the 12 held-out psychopathy queries under psychopathy contexts at layers 12 and 32. At layer 12 they converge on the psychopathy endpoint. At layer 32 the queries whose statement the persona would deny split off along the direction of the upcoming answer, the same late-layer re-encoding seen in the synthetic models (\Secref{app:intervention_perlayer}).

\paragraph{A heuristic for why the generative format is sharper.}
In the Q\&A format the own coefficient plateaus at $0.67$--$0.72$, below the $0.95$ of the generative format. We have no definitive explanation, but a plausible heuristic is that the two formats convey different amounts of information about the persona. A ``Yes'' to an on-persona statement carries about $\log K$ bits about the persona, whereas a ``No'' to another persona's statement carries only $\log\frac{K}{K-1}$ bits, since every other persona would also answer ``No''. In addition, the answer direction is superposed on the hidden state at the readout position. The generative format instead shows the statements themselves, which define the persona.

\begin{figure}[t]
\centering
\begin{subfigure}[t]{0.44\textwidth}\centering
\includegraphics[width=\linewidth]{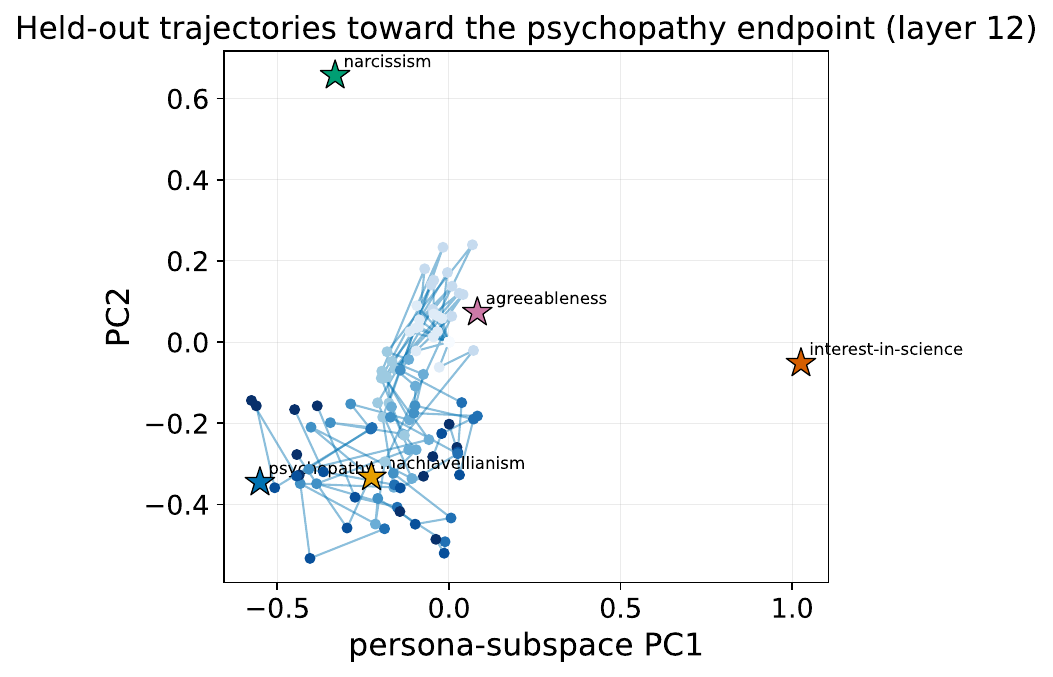}
\caption{Layer 12 (the layer analyzed by Bigelow et al.)}
\end{subfigure}\hspace{0.03\textwidth}
\begin{subfigure}[t]{0.44\textwidth}\centering
\includegraphics[width=\linewidth]{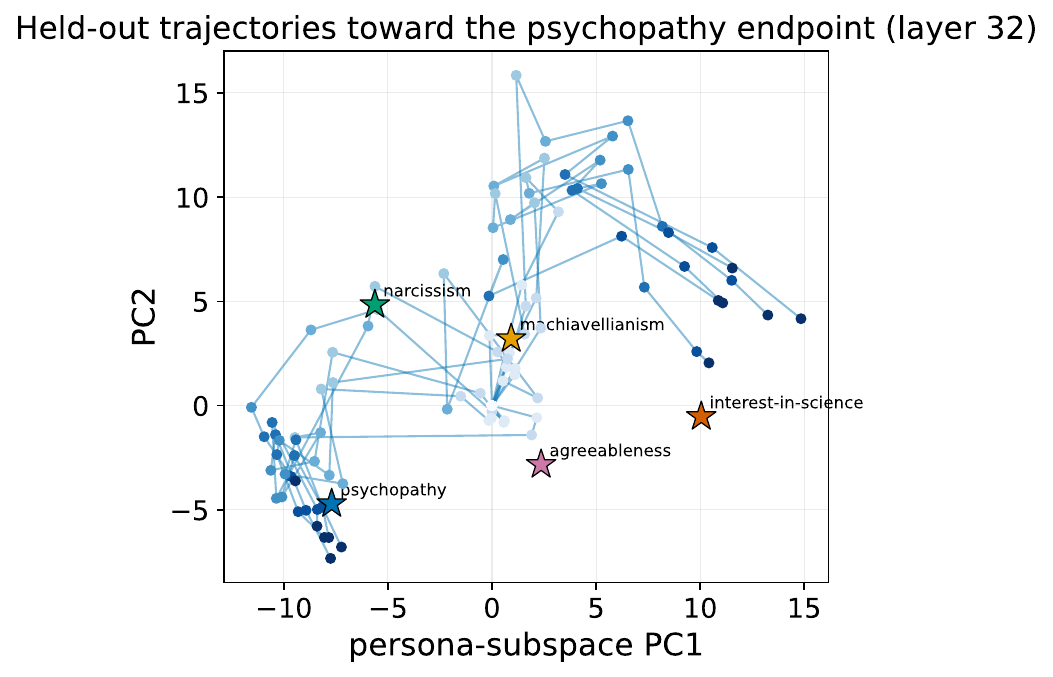}
\caption{Layer 32 (last layer)}
\end{subfigure}
\caption{\textbf{Q\&A format, Llama-3.1-8B-Instruct: trajectories of held-out psychopathy queries under psychopathy contexts.} Each path is one of the 12 held-out psychopathy statements used as the final query, evaluated under psychopathy contexts of $N=0,\ldots,128$ demonstrations (light to dark, one demonstration seed). Hidden states are centered by the matched mean over the five context personas and projected onto the top two principal components of the five-persona endpoint frame (stars), at the layer analyzed by Bigelow et al.\ (left) and at the last layer (right). At layer 12 all trajectories converge on the psychopathy endpoint; at layer 32 the queries whose statement the persona would deny (consistent answer ``No'') separate into a second lobe along the answer direction, because the last layer re-encodes the state for the upcoming answer token.}
\label{fig:persona-qa}
\end{figure}

\subsection{Per-Layer Intervention Results}
\label{app:intervention_perlayer}

This subsection reports the per-layer breakdown of the two complementary causal interventions introduced in \Secref{sec:orth_ablation_main} (suppressing the optimized orthogonal directions $\hat{\mV}_\mathrm{opt}$, and suppressing the task subspace $\operatorname{col}(\hat{\mTheta})$). The layer-resolved curves confirm that the major-vs-OOD/minor double dissociation aggregated in Table~\ref{tab:interventions} is a genuine middle-layer phenomenon rather than an averaging artifact, and reveal a structured three-phase depth profile that the main-text average hides.

\paragraph{Orthogonal subspace intervention.} Table~\ref{tab:interventions} in the main text reports $\Delta\gL_{\mathrm{mode}} / g_{\mathrm{mode}}$
averaged over the middle-layer ranges where orthogonal separation is cleanest;
Figure~\ref{fig:perlayer_all} shows the full per-layer breakdown.
The pattern across all layers follows a three-phase progression.
In \emph{early layers}, computation is shared across major, OOD,
and minor tasks (the model has not yet separated task-specific from
prediction-specific representations), so suppressing any direction in the
orthogonal complement affects all three loss types simultaneously
(e.g.\ layers~0--4 of \texttt{E2}, where major loss reaches
${\approx}95$--$253\%$ alongside large OOD and minor disruption).
In \emph{middle layers}, the orthogonal separation emerges cleanly: the model
has routed task-identity information into the protected subspace and
prediction-relevant information into the orthogonal complement, so
suppressing $\hat{\mV}_\mathrm{opt}$ selectively disrupts OOD and minor
inference while leaving major-task performance near the random baseline;
these are the layers whose averages appear in Table~\ref{tab:interventions}.
In \emph{final layers}, the residual stream feeds directly into the output
head and must encode the next-token prediction regardless of task type;
the distinction between task identity and predictive content collapses,
and the intervention again affects all three losses
(e.g.\ layer~5 of \texttt{E1} and~\texttt{E3}, layers~14--15 of~\texttt{E2}).

\begin{figure}[!htbp]
\centering
\begin{subfigure}[b]{0.48\linewidth}
  \centering
  \includegraphics[width=\linewidth]{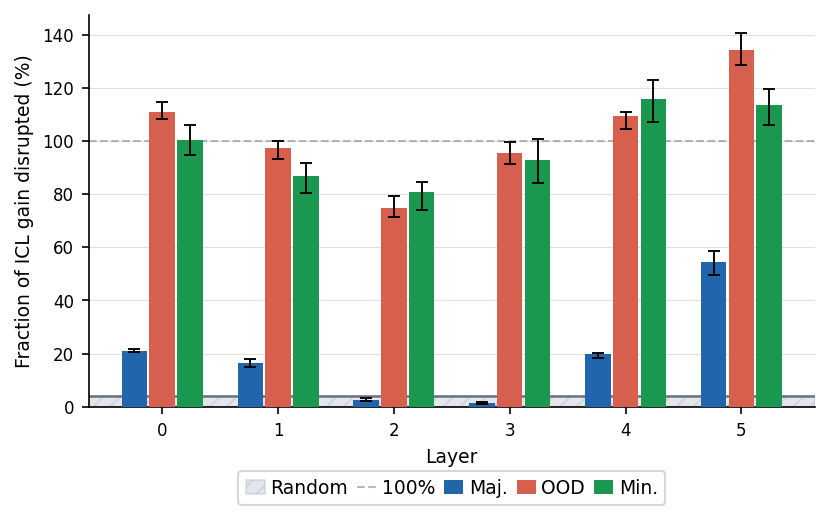}
\end{subfigure}
\hfill
\begin{subfigure}[b]{0.48\linewidth}
  \centering
  \includegraphics[width=\linewidth]{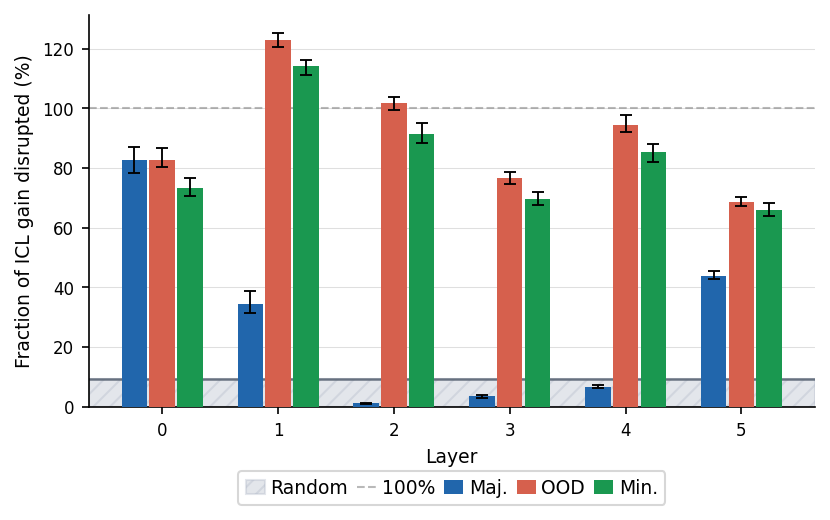}
\end{subfigure}

\vspace{0.6em}

\begin{subfigure}[b]{0.66\linewidth}
  \centering
  \includegraphics[width=0.66\linewidth]{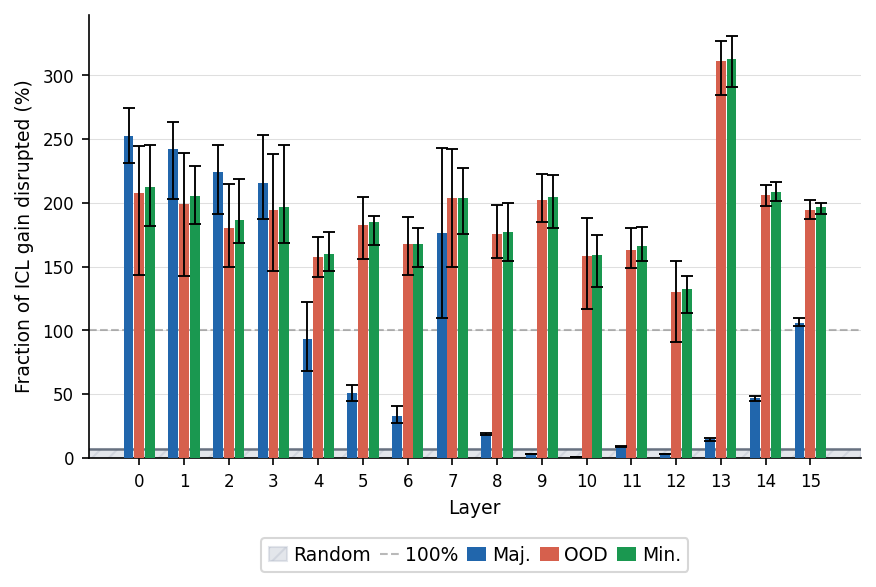}
\end{subfigure}
\caption{%
\textbf{Per-layer orthogonal-subspace intervention across all three experiments.}
Each bar shows $\Delta\gL_{\mathrm{mode}} / g_{\mathrm{mode}} \times 100\%$ when $\hat{\mV}_\mathrm{opt}$
is suppressed at that layer; error bars are the inter-quartile range over evaluation
batches; the shaded band and horizontal dashed line mark the random same-rank
baseline and the 100\% reference respectively.
\textbf{(a) \texttt{E1} (Dice):}
OOD loss is disrupted substantially at all layers (${\approx}53$--$155\%$) and
minor loss uniformly (${\approx}78$--$138\%$); major loss stays near zero at
layers~0--4 (${\approx}1$--$21\%$) but rises to ${\approx}80\%$ at layer~5.
\textbf{(b) \texttt{E3} (Latent Markov):}
OOD and minor losses are largest at layers~0--1 (${\approx}102$--$124\%$) with a
dip at layer~2 (${\approx}54$--$61\%$); major loss stays near the random baseline
at layers~1--4 (${\approx}5$--$17\%$); layer~5 reverses the pattern (major
${\approx}113\%$, OOD/minor ${\approx}63$--$66\%$).
\textbf{(c) \texttt{E2} (Linear Regression):}
major loss is elevated at layers~0--4 (${\approx}95$--$253\%$) and resurfaces
at layers~7 and 14--15; layers~5--6 show a transient dip in OOD disruption
(${\approx}30$--$50\%$); layer~13 produces the largest OOD/minor spike
(${\approx}295$--$310\%$); the focal layers~9--12 give the cleanest orthogonal separation (major ${\approx}5$--$15\%$, OOD/minor ${\approx}130$--$175\%$).
}
\label{fig:perlayer_all}
\end{figure}

\paragraph{Task subspace suppression.}
The intervention applied here is
$\vh' = \vh - \gamma\,\mP_{\mathrm{task}}\vh$, with experiment-specific
$\gamma = 1.5$ for \texttt{E1}, $\gamma = 2$ for \texttt{E2}, and
$\gamma = 2.5$ for \texttt{E3}.
As with the orthogonal subspace intervention, $\gamma = 1$ would correspond
to ordinary projection removal; since all values used here satisfy
$\gamma > 1$, the intervention is an amplified suppression
(over-subtraction) of the task component rather than a pure ablation.
The layer-averaged results for this intervention appear in the left block of
Table~\ref{tab:interventions} in the main text
(\Secref{sec:orth_ablation_main}).
Figure~\ref{fig:task_removal_all} provides the full per-layer breakdown.
In \texttt{E1} (Dice, panel~(a)), suppressing the task subspace devastates major
loss (${\approx}113$--$362\%$, increasing with depth) while OOD and minor losses
remain low (${\approx}5$--$60\%$), with the gap widening in later layers,
consistent with the $211\%$ / ${\approx}12\%$ contrast reported in
Table~\ref{tab:interventions}.
In \texttt{E3} (Latent Markov, panel~(b)), major loss rises sharply from
layer~2 onwards (${\approx}96$--$190\%$), while OOD and minor losses are
near zero at layers~2--4 (${\approx}0$--$1\%$); layers~0--1 show modest
OOD/minor disruption (${\approx}17$--$32\%$), consistent with the early-layer
shared-computation regime identified above and the $132\%$ / ${\approx}13\%$
contrast in Table~\ref{tab:interventions}.
In \texttt{E2} (Linear Regression, panel~(c)), major loss is elevated in early
and middle layers (${\approx}13$--$138\%$, peaking near layers~4--5 and~7--9),
while OOD and minor losses track major loss in the early shared-computation
phase (layers~1--5) but drop to near zero from layer~8 onwards, giving the
$84\%$ / ${\approx}{-}1\%$ pattern shown in Table~\ref{tab:interventions}.

\begin{figure}[!htbp]
\centering
\begin{subfigure}[b]{0.41\linewidth}
  \centering
  \includegraphics[width=\linewidth]{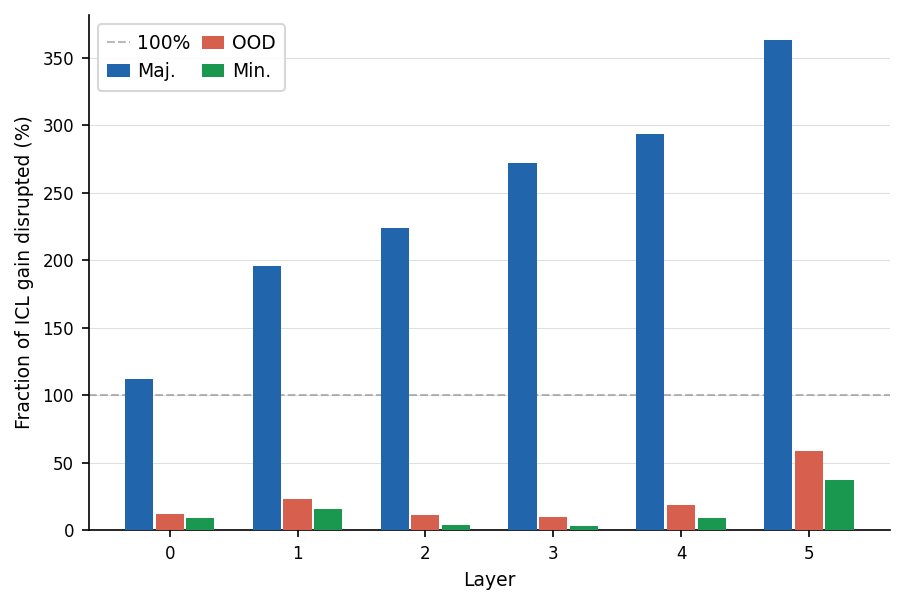}
\end{subfigure}
\hfill
\begin{subfigure}[b]{0.41\linewidth}
  \centering
  \includegraphics[width=\linewidth]{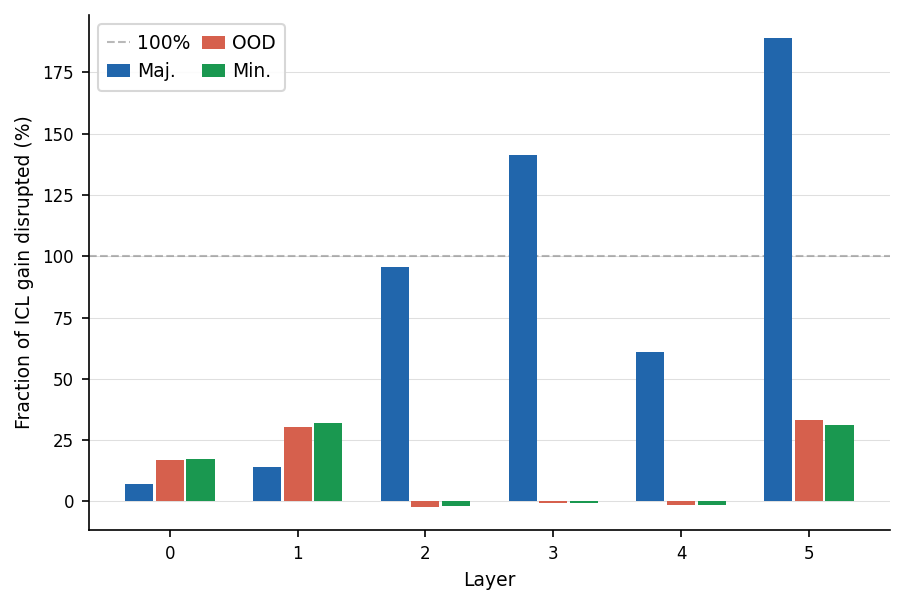}
\end{subfigure}

\vspace{0.6em}

\begin{subfigure}[b]{\linewidth}
  \centering
  \includegraphics[width=0.56\linewidth]{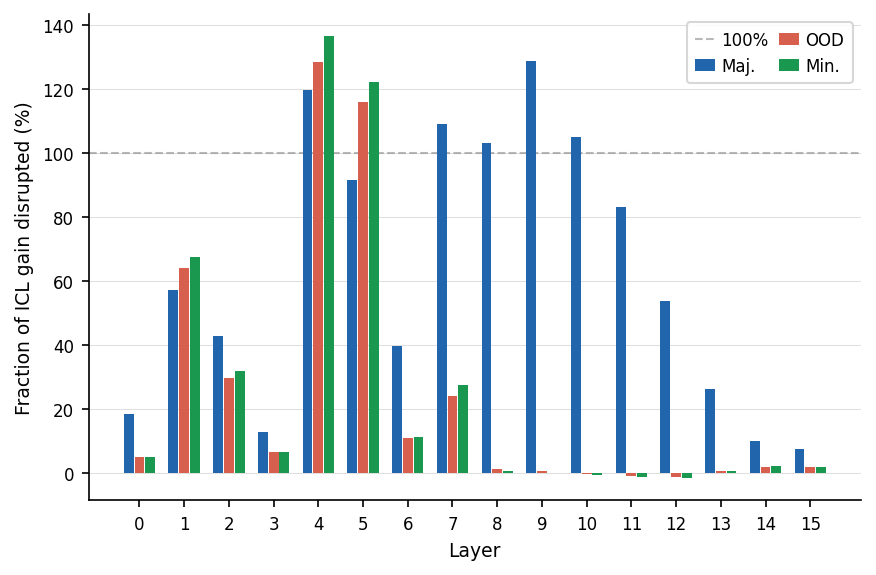}
\end{subfigure}
\caption{%
\textbf{Per-layer task subspace suppression across all three experiments.}
Each bar shows $\Delta\gL_{\mathrm{mode}} / g_{\mathrm{mode}} \times 100\%$ when the task subspace
$\operatorname{col}(\hat{\mTheta})$ is suppressed at that layer.
This is the complementary intervention to Figure~\ref{fig:perlayer_all}:
suppressing the task subspace selectively disrupts major-task performance
while leaving OOD and minor losses largely intact, confirming that the two
subspaces carry functionally distinct information.
\textbf{(a) \texttt{E1} (Dice):}
major loss rises from ${\approx}113\%$ at layer~0 to ${\approx}362\%$ at
layer~5; OOD and minor losses stay below ${\approx}60\%$ throughout.
\textbf{(b) \texttt{E3} (Latent Markov):}
major loss rises sharply from layer~2 (${\approx}96$--$190\%$); OOD and
minor losses are near zero at layers~2--4 but present at layers~0--1
(${\approx}17$--$32\%$) and layer~5 (${\approx}31$--$34\%$).
\textbf{(c) \texttt{E2} (Linear Regression):}
major loss peaks in the early-to-middle layers (${\approx}120$--$138\%$ at
layers~4--5); OOD and minor losses track major loss in the early shared
phase (layers~1--5) but drop to near zero from layer~8 onwards.
}
\label{fig:task_removal_all}
\end{figure}

\subsection{Orthogonal Subspace Intervention: Implementation Details}
\label{app:orth_ablation}

This appendix gives a self-contained algorithmic description of the orthogonal
subspace intervention introduced in \Secref{sec:orth_ablation_main}.

Recall that the training distribution mixes a small set of
\emph{major tasks} $\gZ_{\mathrm{major}}$ (probability~$0.9$) with a large
pool of \emph{minor tasks} $\gZ_{\mathrm{minor}}$ (probability~$0.1$);
OOD tasks are fresh prior draws disjoint from both pools (see
\Secref{app:training-mixture}).
At high diversity ($N_{\mathrm{minor}} \gg 1$), each individual minor task
receives negligible weight and cannot be memorized; the model must instead
rely on in-context demonstrations, just as it does for OOD tasks.
We exploit this similarity by using minor-task sequences to search for directions in the orthogonal complement
whose suppression maximally disrupts next-token prediction.
Crucially, \emph{no major-task data enters the optimization}, so the procedure
provides no explicit guarantee that major-task performance will be preserved.
The finding that major-task loss nonetheless remains near zero across a range
of middle layers (Table~\ref{tab:interventions}) therefore constitutes strong
evidence that major-task and OOD/minor-task inference operate through distinct
computational mechanisms.
All evaluation metrics are computed on independently sampled held-out
sequences that are disjoint from the data used to find optimization
directions, avoiding overfitting.

We extract hidden states $\vh^{(\ell,b)}_{k,t} \in \R^d$ from layer~$\ell$ of the
frozen model at batch sample~$b$, task~$k$, and sequence position~$t$.
All operations below are performed independently per layer~$\ell$; we suppress
the superscript~$\ell$ for readability.

\paragraph{Step 1: Task and token subspace estimation.}
The task subspace basis $\hat{\mTheta} \in \R^{d \times K}$ is obtained from
the averaging-based task vectors
$\hat{\vtheta}_1,\dots,\hat{\vtheta}_{K}$
(stacked as columns; $K=3$ is the number of major tasks)
exactly as in \Secref{sec:extracting-hiddens}.
In addition, we estimate \emph{token encoding vectors}
$\hat{\vnu}_v \in \R^d$ as the token main effects from the same two-way
(token~$\times$~task) ANOVA used for the separability analysis
(\Secref{sec:emp-task-vector}): after forming cell means
$\bar{\mH} \in \R^{V \times K \times d}$
from token-conditioned hidden states and removing positional effects,
$\hat{\vnu}_v = \frac{1}{K}\sum_k \bar{\mH}_{v,k} - \bar{\vh}$.

\paragraph{Step 2: Protected subspace construction.}
Next-token prediction in \texttt{E1}--\texttt{E3} depends on the current
token~$s_t$ (e.g.\ the bigram structure in \texttt{E3} or the input
$\vx_t$ in \texttt{E2}), so suppressing directions that encode token identity
would trivially degrade predictions for \emph{all} task types and render
the intervention uninformative.
We therefore protect the token subspace alongside the task subspace.
The token encoding vectors $\mW_{\mathrm{tok}} \in \R^{V \times d}$
(rows $\{\hat{\vnu}_v\}$) are projected onto
$\operatorname{col}(\hat{\mTheta})^\perp$ via
$\mW_{\mathrm{res}} = (\mI
  -  \mP_{\mathrm{task}})\mW_{\mathrm{tok}}^\top$, where $\mP_{\mathrm{task}}$ is given in Equation~\eqref{eq:task-projection}.
The top-$\rho$ left singular vectors of $\mW_{\mathrm{res}}$ are retained
as token-protection directions.
The task and token directions are concatenated and re-orthonormalized via
compact SVD to yield $\mU_{\mathrm{prot}} \in \R^{d \times p}$, and the
orthogonal complement basis
\[
  \mU_\perp \;\in\; \R^{d \times q},
  \qquad q = d - p,
\]
is extracted as the eigenvectors of
$\mP_\perp = \mI_d - \mU_{\mathrm{prot}}\mU_{\mathrm{prot}}^\top$
with eigenvalue exceeding $\tfrac{1}{2}$.

\paragraph{Step 3: Direction optimization.}
We search for a rank-$r$ subspace within $\operatorname{col}(\mU_\perp)$
whose suppression maximally disrupts minor-task prediction
($r = 6$ for \texttt{E1} and \texttt{E3}; $r = 2$ for \texttt{E2}).
The learnable parameter $\mG \in \R^{q \times r}$ is initialized as a
random unit-norm matrix and updated with Adam (learning rate $0.01$,
gradient clip norm $1.0$) for up to $N_{\mathrm{opt}}$ steps
($80$ for \texttt{E1}, $90$ for \texttt{E3}, $25$ for \texttt{E2}),
with early stopping after $50$ consecutive steps without improvement in a
smoothed objective (EMA coefficient $0.1$).
At each step the intervention directions are
$\mV_Q = \mathrm{QR}(\mU_\perp \mG)$,
giving the intervention
\[
  \vh' \;=\; \vh - \gamma\,\mV_Q\mV_Q^\top\vh,
\]
with $\gamma = 2.5$ for \texttt{E1} and \texttt{E3}, and $\gamma = 2$ for
\texttt{E2}, applied via a forward hook at layer~$\ell$.
Here \(\gamma=1\) would correspond to ordinary orthogonal projection removal. Since the
values above satisfy \(\gamma>1\), the intervention is an amplified suppression
(over-subtraction) of the selected component rather than a pure ablation; the optimized
directions and the random same-rank baseline use the same \(\gamma\) within each
experiment.
The objective is maximized over batches of size $B_{\mathrm{opt}} = 256$
drawn from the minor-task distribution.
An exponential moving average of $\mG$ with decay $0.99$ is maintained; the
final directions are $\hat{\mV}_{\mathrm{opt}} = \mathrm{QR}(\mU_\perp \hat{\mG})$.

\paragraph{Step 4: Evaluation.}
All evaluation is performed on \emph{independently sampled} sequences that are
disjoint from the batches used during Step~3.
On held-out batches of size $B$ from the major-task, OOD, and minor-task
distributions, we record the baseline and intervened loss at all evaluation
positions $t \in \gT_{\mathrm{eval}}$ and report $\Delta\gL_{\mathrm{mode}}$ for
$\mathrm{mode} \in \{\mathrm{maj}, \mathrm{ood}, \mathrm{minor}\}$ as defined in
\Secref{sec:orth_ablation_main}.
As a sanity check, the same intervention is applied with $r$ random orthogonal
directions sampled uniformly from $\operatorname{col}(\mU_\perp)$, and the
resulting loss increase is compared against $\Delta\gL_\mathrm{ood}$ to confirm
that the optimized directions are non-trivially better than chance.
Linear probes are then fitted from ground-truth latent features to
$\hat{\mV}_{\mathrm{opt}}^\top \vh$ on a separately held-out probe set of
size $N_{\mathrm{probe}}$, and $R^2$ is reported on an 80/20 validation split.

\paragraph{Hyperparameters.}
Unless stated otherwise, we use
$\rho = \lfloor 0.9 \cdot \operatorname{rank}(\mW_{\mathrm{res}})\rfloor$
token-protection components,
$|\gT_{\mathrm{eval}}| = T$ (all positions), and
$N_{\mathrm{probe}} = 1024$ probe samples.

\subsection{Decomposition of Orthogonal Directions}
\label{app:orth_decomp}

Together, the two complementary interventions above (suppressing
$\hat{\mV}_\mathrm{opt}$ and suppressing $\operatorname{col}(\hat{\mTheta})$)
establish a clean double dissociation: directions in
$\operatorname{col}(\mU_\perp)$ are \emph{causally} responsible for OOD and
minor-task inference, while the task subspace carries major-task information.
Having confirmed this functional separation, we now ask a more fine-grained
question: \emph{what} information does $\hat{\mV}_\mathrm{opt}$ concretely
encode?

\paragraph{Method.}
For each layer~$\ell$ and experiment we collect hidden states $\vh$ from
OOD/minor-task sequences and compute the low-dimensional projection
$\vz = \hat{\mV}_\mathrm{opt}^\top\vh \in \R^r$.
We regress $\vz$ onto individual observable features constructed solely from
the context $(s_1,\ldots,s_t)$ for \texttt{E1}/\texttt{E3}, or the input--output pairs for \texttt{E2}, then report the combined $R^2$ from
fitting all features jointly (dashed line in figures).
For experiments \texttt{E1} and \texttt{E3} we use the \emph{centred
log-ratio} (CLR) transform of \citet{aitchison1986statistical}, applied to
empirical token frequencies over the $V$-symbol vocabulary.
Raw empirical frequencies can contain zero entries (a symbol may not yet
have appeared in the prefix, and most rows of the empirical bigram table
are zero at small~$t$), and CLR requires strictly positive arguments. We
therefore first apply additive (Dirichlet) smoothing with the
Krichevsky--Trofimov prior $\alpha = \tfrac{1}{2}$, then take the centred
log-ratio of the resulting strictly positive frequency vector
$\vf \in \Delta^{V-1}_{>0}$:
\begin{equation}
  \operatorname{CLR}(\vf)_a
  \;=\;
  \log f_a
  \;-\;
  \frac{1}{V}\sum_{b=1}^{V}\log f_b,
  \qquad
  a = 1,\ldots,V.
  \label{eq:clr}
\end{equation}
This maps $\Delta^{V-1}_{>0}$ to the zero-sum hyperplane in $\R^{V}$, removing
overall scale while preserving relative frequency information.
In implementation a numerical floor of $10^{-12}$ is additionally applied
inside the logarithm to guard against underflow; with $\alpha=\tfrac{1}{2}$
this floor is never active in practice and is purely defensive.
Note that $V$ here is the symbol-alphabet size (cf.\ \Secref{sec:setup}); the
smoothing constant $\alpha$ is a feature hyperparameter chosen by convention
and is distinct from the generative Dirichlet hyperparameter $\alpha_0$ used
in \Secref{app:approx-posterior}.
\begin{itemize}
  \item \textbf{\texttt{E1}:}
    \emph{Unigram CLR}: $\operatorname{CLR}(\tilde{\vp}_t)$, where the
    $V$-vector $\tilde{\vp}_t \in \Delta^{V-1}_{>0}$ has entries
    \[
      \tilde{p}_{t,a}
      \;=\;
      \frac{n_a(t) + \alpha}{t + \alpha V},
      \qquad a = 1,\ldots,V,
    \]
    with the unigram count $n_a(t)$ defined in Eq.~\eqref{eq:token-count};
    \emph{Current token}: one-hot encoding $\ve_{s_t}\in\R^V$;
    \emph{Position}: fractional sequence position $t/T$.
  \item \textbf{\texttt{E3}:}
    \emph{Bigram CLR}: $\operatorname{CLR}(\tilde{\vq}_t)$, where the
    $V$-vector $\tilde{\vq}_t \in \Delta^{V-1}_{>0}$ is the smoothed conditional
    frequency of the next symbol given the current token~$s_t$,
    \[
      \tilde{q}_{t,b}
      \;=\;
      \frac{N_{s_t,\,b}(t) + \alpha}{N_{s_t,\cdot}(t) + \alpha V},
      \qquad b = 1,\ldots,V,
    \]
    with the bigram counts $N_{ab}(t)$ and row marginal $N_{a\cdot}(t)$ defined
    as in \Secref{app:approx-posterior};
    \emph{Current token}: one-hot encoding $\ve_{s_t}$;
    \emph{Position}: fractional sequence position $t/T$.
  \item \textbf{\texttt{E2}:}
    $\hat{\vw}_t := \hat{\vw}_{\mathrm{ridge}}(t-1)
    = \bigl(\mX_{<t}^\top \mX_{<t} + \tau^2 \mI_D\bigr)^{-1}
      \mX_{<t}^\top \vy_{<t}$,
    the running ridge estimator from \Secref{app:approx-posterior} that
    defines the extrapolative \texttt{M2} predictor. The ridge penalty is set
    to $\tau^2 = \sigma^2$ to match the unit-variance prior of
    \Secref{sec:setup} (cf.\ \Secref{app:approx-posterior}); with noise scale
    $\sigma = 0.5$ this gives $\tau^2 = 0.25$. The estimate is updated
    online using only the \emph{completed} pairs
    $(\vx_1,y_1),\ldots,(\vx_{t-1},y_{t-1})$ that are causally available at
    the extraction position $\vx_t$ (the prediction location for $y_t$;
    cf.\ \Secref{sec:extracting-hiddens}); by convention
    $\hat{\vw}_{\mathrm{ridge}}(0) = \vzero$ (the prior mean).
    $\vx_t$ (current input);
    $\vx_t^\top \hat{\vw}_t$ (diagnostic prediction score);
    $\|\vx_t\|_2^2$ (input norm).
\end{itemize}
We report held-out linear-probe $R^2$ (80/20 split) for each individual
feature and for the combined set.
To test the causal significance of the linearly accessible content, we extract
the \emph{filtered} subspace $\mV_\mathrm{filt}$: fit a linear map from the
combined features to $\vz$, take the SVD of the fitted values, and retain
directions explaining ${>}\,0.5\%$ of total variance.
We then run the causal intervention with $\mV_\mathrm{filt}$ in place of
$\hat{\mV}_\mathrm{opt}$ and compare the resulting loss increases.

\begin{figure}[!htbp]
\centering
\begin{subfigure}[b]{0.41\linewidth}
  \centering
  \includegraphics[width=\linewidth]{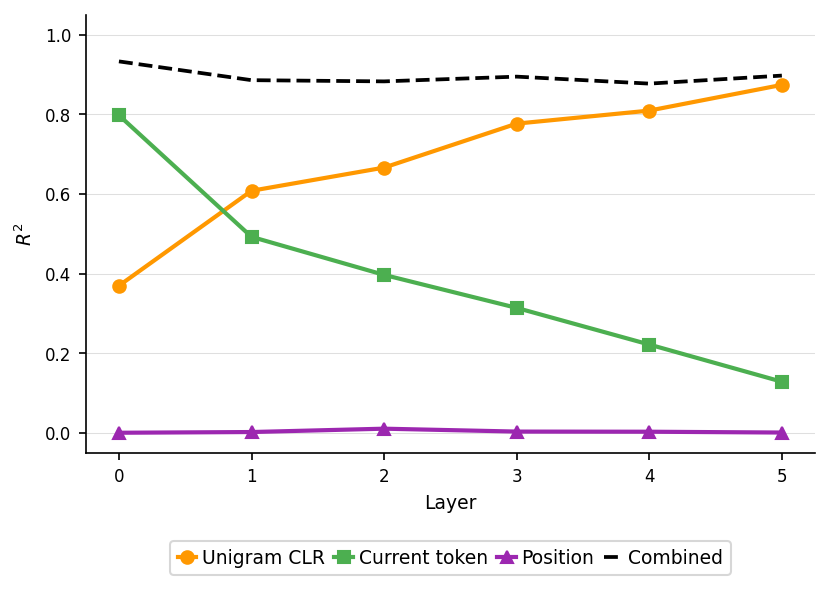}
\end{subfigure}
\hfill
\begin{subfigure}[b]{0.41\linewidth}
  \centering
  \includegraphics[width=\linewidth]{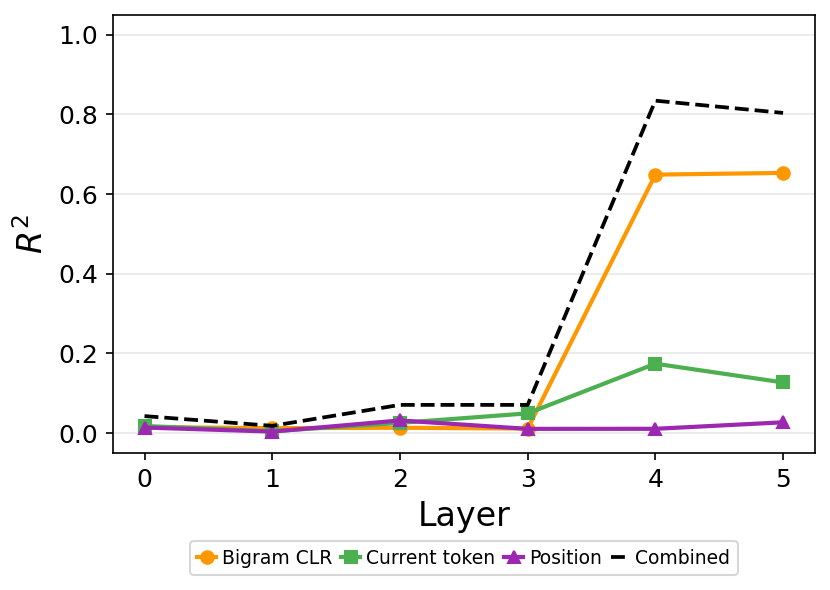}
\end{subfigure}

\vspace{0.6em}

\begin{subfigure}[b]{0.56\linewidth}
  \centering
  \includegraphics[width=\linewidth]{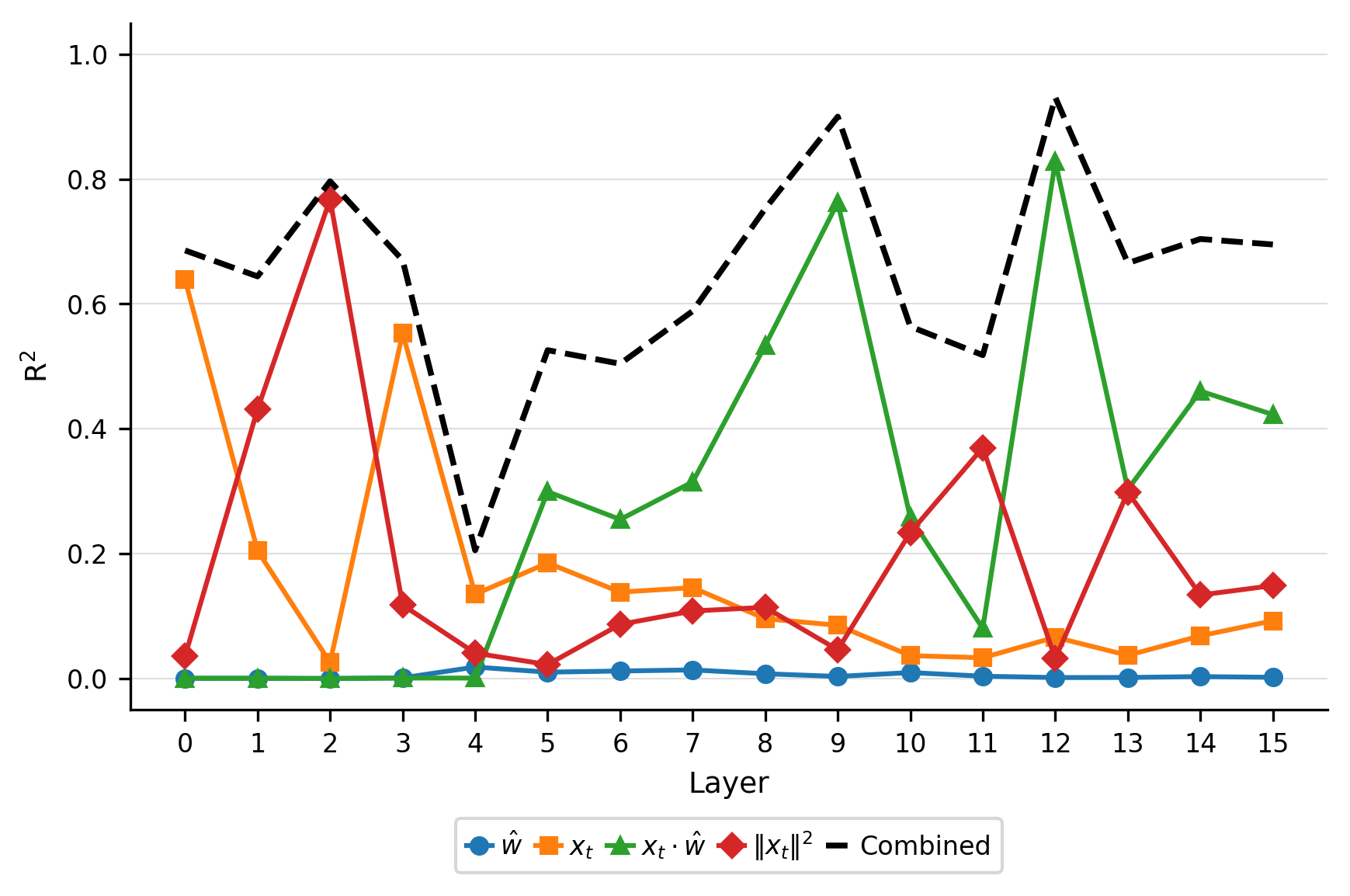}
\end{subfigure}
\caption{%
\textbf{Per-feature $R^2$ predicting $\hat{\mV}_\mathrm{opt}^\top\vh$
across all three experiments.}
Solid lines show individual feature contributions; the dashed line shows the
combined $R^2$ from regressing on all features jointly (see
Eq.~\ref{eq:clr} for the CLR definition used in \texttt{E1} and \texttt{E3}).
\textbf{(a) \texttt{E1} (Dice):}
current token dominates early ($R^2 \approx 0.81$ at layer~0) and decays,
while unigram CLR rises from $\approx 0.35$ to $\approx 0.88$ by layer~5;
combined $R^2$ remains $0.83$--$0.94$ throughout.
\textbf{(b) \texttt{E3} (Latent Markov):}
combined $R^2 \approx 0.01$--$0.04$ at layers~0--2, rising to $\approx 0.61$
at layer~3 and $0.79$--$0.85$ at layers~4--5;
bigram CLR becomes the dominant predictor ($R^2 \approx 0.75$ at layer~5).
\textbf{(c) \texttt{E2} (Linear Regression):}
in early layers, the input statistics dominate, with $\vx_t$ peaking at
$R^2 \approx 0.64$ at layer~0 and $\|\vx_t\|^2$ peaking at
$R^2 \approx 0.77$ at layer~2;
from layer~5 onward, the (causal) prediction score
$\vx_t\cdot\hat{\vw}_t$ takes over and peaks at $R^2 \approx 0.76$ at
layer~9 and $R^2 \approx 0.83$ at layer~12;
$\hat{\vw}_t$ stays near zero throughout; combined $R^2$ ranges from
$\approx 0.21$ at the layer-4 transition dip up to $\approx 0.93$ at layer~12.
}
\label{fig:r2_all}
\end{figure}

\paragraph{Results: Dice (\texttt{E1}).}
Figure~\ref{fig:r2_all}(a) reveals that the orthogonal direction encodes
different context statistics at different depths.
In the early layers (0--2), the current token one-hot already explains the
majority of $\hat{\mV}_\mathrm{opt}$ ($R^2 \approx 0.81$ at layer~0),
while the unigram CLR contributes $R^2 \approx 0.35$; together the combined
features achieve $R^2 \approx 0.94$.
As depth increases, the current token's contribution decays to $\approx 0.14$
by layer~5, while the unigram CLR rises steadily to $\approx 0.88$.
The combined $R^2$ remains high throughout ($0.83$--$0.94$), indicating that
observable token statistics explain $\hat{\mV}_\mathrm{opt}$ well at every
layer.
Figure~\ref{fig:filt_all}(a) confirms this causally: the filtered bars are
nearly indistinguishable from $\hat{\mV}_\mathrm{opt}$ at every layer
(${\approx}75$--$135\%$ for both OOD and minor), reproducing its causal
effect almost exactly.
This shows that what the orthogonal directions encode changes across layers,
shifting from raw token identity in early layers to the running empirical
frequency distribution in later layers, but the content is linearly
accessible from observable statistics throughout.

\paragraph{Results: Latent Markov (\texttt{E3}).}
The Latent Markov setting presents a markedly different pattern
(Figures~\ref{fig:r2_all}(b) and~\ref{fig:filt_all}(b)).
The combined observable features explain virtually none of
$\hat{\mV}_\mathrm{opt}$ in the first three layers ($R^2 \approx 0.01$--$0.04$
at layers~0--2), then rise sharply to $\approx 0.61$ at layer~3 and
$\approx 0.79$--$0.85$ at layers~4--5.
At layer~3, current token and bigram CLR contribute comparably
($R^2 \approx 0.34$ and $0.26$ respectively); beyond that, bigram CLR becomes
the dominant predictor ($R^2 \approx 0.75$ at layer~5) while current token
decays to $\approx 0.10$ and position contributes negligibly throughout.
The filtered intervention (Figure~\ref{fig:filt_all}(b)) mirrors this pattern
causally: at layer~1 the filtered bars are near zero (${\approx}10\%$) despite
$\hat{\mV}_\mathrm{opt}$ causing ${\approx}115$--$123\%$ disruption.
Layers~0 and~2--3 show partial filtered effects (${\approx}25$--$56\%$ vs
$\hat{\mV}_\mathrm{opt}$'s ${\approx}70$--$102\%$), and at layers~4--5 the
filtered bars closely track $\hat{\mV}_\mathrm{opt}$ (${\approx}65$--$85\%$).
This reveals a genuine developmental transition: in early layers
$\hat{\mV}_\mathrm{opt}$ encodes prediction-relevant information that is not
linearly accessible from simple token statistics, whereas in later layers the
orthogonal directions develop a linearly readable bigram-based predictor.

\paragraph{Results: Linear Regression (\texttt{E2}).}
Figures~\ref{fig:r2_all}(c) and~\ref{fig:filt_all}(c) show the results for
\texttt{E2}.
The $R^2$ plot (Figure~\ref{fig:r2_all}(c)) reveals a two-phase per-feature
progression separated by a sharp transition.
In the early layers (0--3) the input statistics dominate: the current
input $\vx_t$ peaks at $R^2 \approx 0.64$ at layer~0 (with a secondary peak
of $\approx 0.55$ at layer~3), and the input norm $\|\vx_t\|^2$ peaks at
$R^2 \approx 0.77$ at layer~2, indicating that the orthogonal direction
in this regime tracks the raw current input together with its scale.
From layer~5 onward, the prediction score
$\vx_t \cdot \hat{\vw}_t$ (the inner product of the current input with
the causal running ridge estimate, which measures how well the input
aligns with the estimated task direction) takes over as the dominant
feature, peaking at $R^2 \approx 0.76$ at layer~9 and
$R^2 \approx 0.83$ at layer~12;
$\|\vx_t\|^2$ persists as a moderate secondary contributor in the deeper
layers (with sporadic peaks of $\approx 0.23$--$0.37$ at layers~10--11 and
13).
The running estimate $\hat{\vw}_t$ itself contributes negligibly throughout,
confirming that the compressed orthogonal direction tracks the scalar
prediction score rather than the full ridge/weight estimate.
The combined $R^2$ ranges from $\approx 0.21$ at the layer-4 transition dip
to $\approx 0.93$ at layer~12, with a notable dip at layer~4 coinciding
with the transition between the input-statistics-dominated and
$\vx_t \cdot \hat{\vw}_t$-dominated regimes.
The filtered intervention (Figure~\ref{fig:filt_all}(c)) closely reproduces
$\hat{\mV}_\mathrm{opt}$'s causal effect on OOD and minor losses across
most layers, with filtered and full bars essentially indistinguishable at
every layer except layer~4.
At the layer-4 transition, the combined probe $R^2$ dips to ${\approx}0.21$
and the $R^2{\ge}0.05$ filter retains only one of the two $V_\mathrm{opt}$
directions (with low joint $R^2$); the filtered bar consequently collapses
to ${\approx}12\%$ even though $\hat{\mV}_\mathrm{opt}$ itself disrupts
${\approx}182\%$ of the ICL gain.
This single-layer outlier coincides exactly with the
input-statistics-to-prediction-score regime change identified in the
$R^2$ analysis above, and the filtered subspace recovers
$\hat{\mV}_\mathrm{opt}$'s causal effect once the joint $R^2$ rises again
at deeper layers, confirming that the causally active content elsewhere
is well explained by these observable statistics.

\begin{figure}[!htbp]
\centering
\begin{subfigure}[b]{0.41\linewidth}
  \centering
  \includegraphics[width=\linewidth]{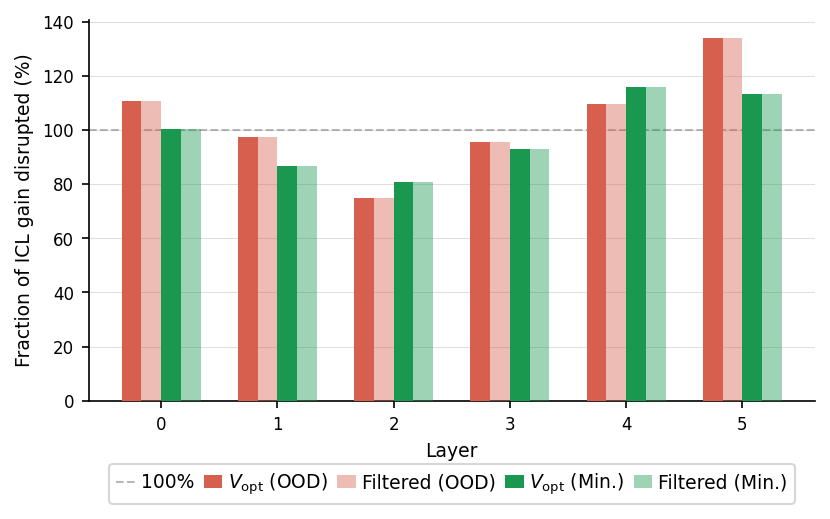}
\end{subfigure}
\hfill
\begin{subfigure}[b]{0.41\linewidth}
  \centering
  \includegraphics[width=\linewidth]{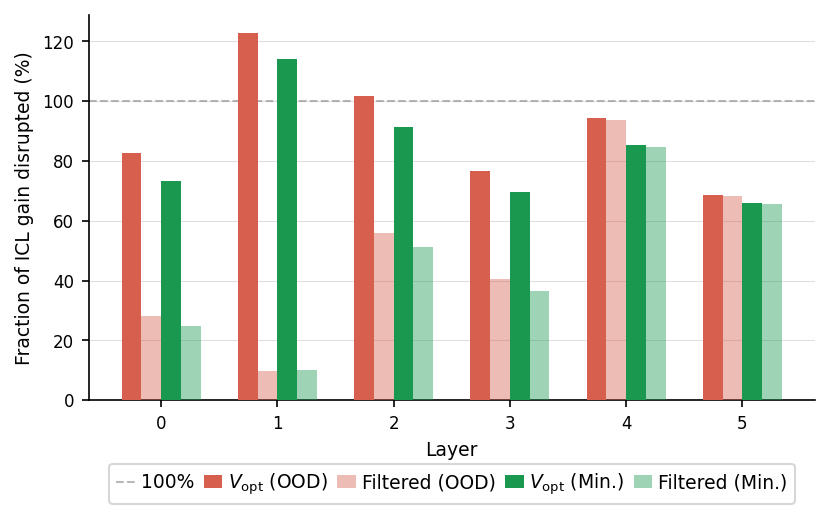}
\end{subfigure}

\vspace{0.6em}

\begin{subfigure}[b]{0.56\linewidth}
  \centering
  \includegraphics[width=\linewidth]{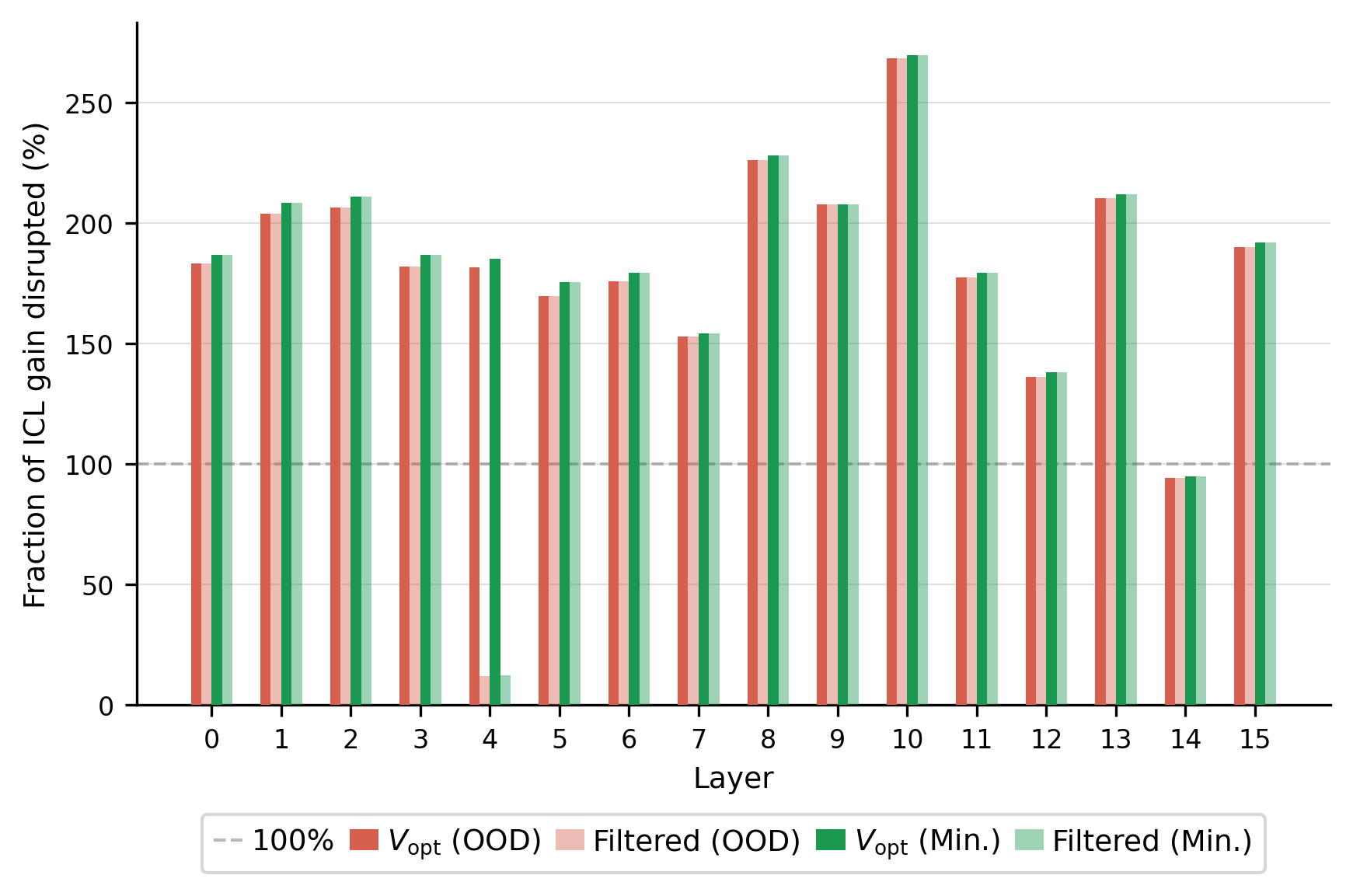}
\end{subfigure}
\caption{%
\textbf{Causal comparison of $\hat{\mV}_\mathrm{opt}$ vs.\ filtered subspace
$\mV_\mathrm{filt}$ across all three experiments.}
Dark bars show the full $\hat{\mV}_\mathrm{opt}$ intervention; light bars show
the filtered subspace derived from the combined observable features.
\textbf{(a) \texttt{E1} (Dice):}
filtered bars are nearly indistinguishable from $\hat{\mV}_\mathrm{opt}$
at every layer (${\approx}75$--$135\%$), confirming that the causally active
content is linearly accessible from token statistics throughout.
\textbf{(b) \texttt{E3} (Latent Markov):}
filtered bars are near zero at layer~1 (${\approx}10\%$) despite
$\hat{\mV}_\mathrm{opt}$ causing ${\approx}115$--$123\%$ disruption;
partial at layers~0 and 2--3 (${\approx}25$--$56\%$); closely match
$\hat{\mV}_\mathrm{opt}$ at layers~4--5 (${\approx}65$--$85\%$),
confirming a developmental transition in linear accessibility.
\textbf{(c) \texttt{E2} (Linear Regression):}
filtered bars closely match $\hat{\mV}_\mathrm{opt}$ at every layer
except layer~4; the layer-4 filtered bar collapses to ${\approx}12\%$
because the combined probe $R^2$ dips to ${\approx}0.21$ at the
input-statistics-to-prediction-score transition and the $R^2{\ge}0.05$
filter retains a direction with low joint explanatory power, missing the
causally active component there.
}
\label{fig:filt_all}
\end{figure}

\paragraph{Summary.}
The three settings reveal distinct patterns in both \emph{what} and \emph{when}
$\hat{\mV}_\mathrm{opt}$ encodes prediction-relevant context.
In \texttt{E1} (Dice), simple token statistics explain $\hat{\mV}_\mathrm{opt}$
well at all depths ($R^2 \ge 0.83$ combined), but the dominant feature shifts
across layers: current token identity predominates early while unigram CLR
takes over in later layers.
The filtered subspace accounts for essentially all of $\hat{\mV}_\mathrm{opt}$'s
causal power throughout the network.
In \texttt{E3} (Latent Markov), token statistics explain almost nothing in
layers~0--2 ($R^2 \approx 0.01$--$0.04$), then become strongly predictive from
layer~3 onwards as bigram CLR and current token emerge jointly
($R^2 \approx 0.61$ at layer~3, rising to $0.79$--$0.85$ at layers~4--5).
Correspondingly, the filtered intervention is near zero at layer~1 and
partial at layers~0 and 2--3, only reproducing the full causal effect at
layers~4--5; the large early-layer effects of $\hat{\mV}_\mathrm{opt}$ are
not linearly explainable by observable token statistics.
In \texttt{E2} (Linear Regression), the dominant feature shifts from input
statistics in the early layers, where $\vx_t$ peaks at $R^2 \approx 0.64$
at layer~0 and $\|\vx_t\|^2$ peaks at $R^2 \approx 0.77$ at layer~2, to
the (causal) prediction score $\vx_t \cdot \hat{\vw}_t$ from layer~5
onward (peak $R^2 \approx 0.83$ at layer~12), with $\hat{\vw}_t$
contributing negligibly throughout.
The filtered subspace accounts for the causal effect on OOD and minor tasks
across most layers, with the exception of the layer-4 R\textsuperscript{2}
transition dip where the filter discards the causally active directions.
In all cases, the filtered subspace $\mV_\mathrm{filt}$ reproduces essentially
all of $\hat{\mV}_\mathrm{opt}$'s causal power once the combined $R^2$ is
appreciable, directly confirming that $\operatorname{col}(\mU_\perp)$ stores
linearly accessible running prediction statistics that drive extrapolative
in-context learning.

\subsection{Near-orthogonality of the OOD subspace is learned: principal angles against a random-subspace null}
\label{app:principal-angles}

Empirically, OOD representations in the \texttt{M2} regime are near-orthogonal to the ID task subspace (\Secref{sec:representation_geometry}). In a hidden dimension of $d=128$ this could arise by chance, since two random low-dimensional subspaces are already nearly orthogonal. This section shows that the observed near-orthogonality is not inherited from random initialization but is learned during training. Early in training the loss-relevant OOD direction overlaps the task subspace far above chance, and it moves to chance-level overlap as \texttt{M2} develops. Unlike a random direction, the learned direction is also predictive of the OOD statistics and necessary for OOD prediction.

\paragraph{Setup.} We use the \texttt{E3} model at diversity exponent $m=10$, where the model is in the \texttt{M2} mode, and repeat the analysis at $m=6$. Our analysis covers every layer. The task subspace $\operatorname{col}(\hat{\mTheta})$ is estimated by averaging task vectors over sequences that are not used anywhere else in the analysis.

\paragraph{Two estimates of the OOD subspace.} The probe subspace is the row space of a ridge regression of OOD hidden states on the bigram transition statistics of \Secref{app:orth_decomp}. It captures where the in-context information is linearly decodable. The loss-relevant subspace consists of the directions found by gradient ascent on the OOD loss with no constraint on where they lie. It captures the directions that matter for OOD prediction.

\paragraph{Affinity and the random-subspace null.} We measure how much an OOD subspace $\mU$ overlaps the task subspace by its affinity $\rho$, defined as the mean squared cosine of the principal angles between $\mU$ and $\operatorname{col}(\hat{\mTheta})$. Equivalently, $\rho$ is the fraction of the energy of $\mU$ that lies inside the task subspace. It equals one when $\mU$ lies inside the task subspace and zero when the two are orthogonal. Task subspaces estimated from two disjoint halves of the held-out sequences have affinity $0.92$--$0.999$ to each other, so the estimation error of $\operatorname{col}(\hat{\mTheta})$ is small. The null distribution is that of $\rho$ for $2{,}000$ random subspaces of the same dimension as $\mU$, drawn from the Haar measure. Its mean is $r/d$ for a rank-$r$ subspace. At $d=128$ this is about $0.016$ for the rank-2 loss-relevant subspace and $0.05$--$0.08$ for the probe subspace of rank $7$--$10$. Random subspaces are therefore nearly but not exactly orthogonal to the task subspace. An angle alone cannot separate a learned subspace from a random one, but the comparison to the matched null and the behavior across training can.

\begin{table}[htbp]
\centering
\small
\setlength{\tabcolsep}{5pt}
\caption{\textbf{Affinity of the loss-relevant OOD subspace to the task frame, per layer (final checkpoint).} Affinity $\rho$ of the loss-relevant subspace at $m=10$ ($m=6$ in parentheses), the $95\%$ band of the matched rank-$2$ Haar null, the position relative to that band, and the affinity of the probe subspace.}
\label{tab:pa-layers}
\begin{tabular}{@{}c c c l c@{}}
\toprule
layer & loss-relevant $\mU$ $\rho$ & Haar null $95\%$ band & relative to null & probe $\mU$ $\rho$ \\
\midrule
0 & 0.145 (0.036)   & $[0.002, 0.043]$ & above null ($m=10$; $p<0.002$) & 0.912 (0.803) \\
1 & 0.079 (0.042)   & $[0.002, 0.043]$ & above null ($m=10$)            & 0.849 (0.795) \\
2 & 0.0039 (0.0038) & $[0.002, 0.043]$ & within null band ($p=0.16$ / $0.17$) & 0.995 (0.919) \\
3 & 0.0135 (0.0045) & $[0.002, 0.044]$ & within null band ($p=0.97$ / $0.22$) & 0.837 (0.485) \\
4 & 0.0147 (0.0155) & $[0.002, 0.043]$ & within null band               & 0.037 (0.279) \\
5 & 0.0245 (0.0184) & $[0.002, 0.043]$ & within null band               & 0.083 (0.159) \\
\bottomrule
\end{tabular}
\end{table}

\begin{figure}[htbp]
\centering
\includegraphics[width=0.9\linewidth]{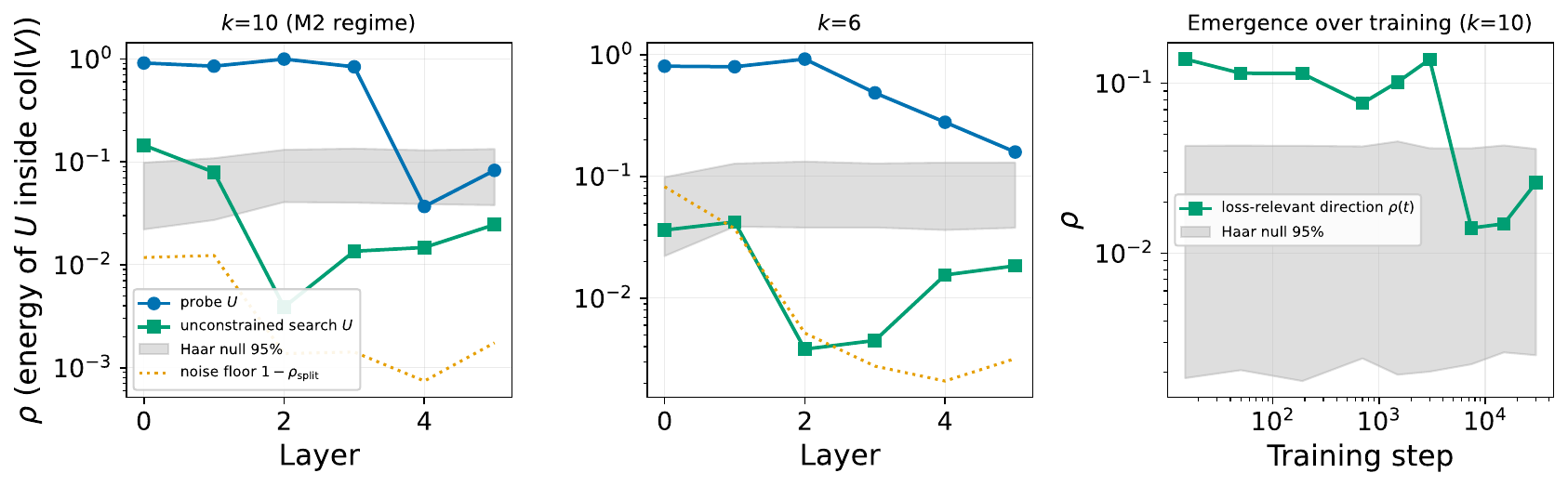}
\caption{\textbf{Principal-angle affinity against the Haar null.} \textbf{Left, center:} per-layer affinity $\rho$ of the probe and loss-relevant OOD subspaces to the task frame at $m=10$ and $m=6$ (log scale), with the $95\%$ band of the matched Haar null (gray) and the estimation noise floor (dotted). \textbf{Right:} affinity of the loss-relevant direction at the band layer across training checkpoints ($m=10$): it starts at $5$--$9\times$ the null and descends into the null band over the window in which \texttt{M2} develops.}
\label{fig:pa-angles}
\end{figure}

\paragraph{Results.} Table~\ref{tab:pa-layers} and Fig.~\ref{fig:pa-angles} give three results. First, at layers $2$--$5$, which include the analysis band, the loss-relevant OOD direction is as orthogonal to the task subspace as a random direction. Its affinity $\rho=0.004$--$0.025$ lies inside the null band. At layers $0$--$1$ it lies far above the null ($\rho=0.08$--$0.15$, $p<0.002$ at $m=10$). These are the layers of shared computation in \Secref{app:intervention_perlayer}, where task and prediction information are not yet separated.

Second, the orthogonality is learned. We repeat the gradient ascent across training checkpoints at the band layer. The loss-relevant direction starts at $5$--$9$ times the null overlap, with $\rho=0.08$--$0.14$ and largest cosine up to $0.50$ between steps $15$ and $3000$. It then drops into the null band ($\rho=0.014$--$0.026$) over the same training window in which \texttt{M2} develops (Fig.~\ref{fig:kl_combined}). If the orthogonality came from initialization, this trace would be flat.

Third, the component of OOD representations that does lie in the task subspace is informative. The OOD transition statistic predicts the task-subspace coordinates of OOD hidden states with $R^2=0.49$, against $0.05$ for random directions orthogonal to the task subspace. The model thus represents inside the simplex which major tasks the OOD task resembles. This is why the probe subspace overlaps the task subspace substantially at layers $0$--$3$, while the direction that matters for the loss is orthogonal to it. Together with the causal evidence of \Secref{app:order-sweep}, where removing the learned direction costs up to $0.28$ nats of OOD loss and removing random orthogonal directions costs about $0$, the learned subspace has four properties. It is orthogonal to the task subspace at chance level, predictive, causally necessary, and emergent under training. A random direction has only the first.

\paragraph{Why orthogonal?} We offer a tentative explanation for why the model uses near-orthogonal subspaces to represent the two modes. First, from classical results on packing feature directions into a limited number of dimensions, we know that the arrangement with the least interference is orthogonal when the number of features is at most the dimension and a tight frame when it exceeds the dimension \citep{welch1974lower, donoho2006compressed}, and neural networks have been observed to store features this way \citep{elhage2022superposition}. These results motivate the picture but do not directly apply to transformers. Second, a heuristic mechanism points the same way. Downstream layers read the task subspace as a posterior over training tasks (Theorem~\ref{thm:bayes_realization_from_hidden}), so any OOD component inside $\operatorname{col}(\hat{\mTheta})$ would corrupt the in-distribution predictions, and the training loss pushes that component into the orthogonal complement. Third, the checkpoint results above agree with this mechanism, since the orthogonality is absent early in training and forms over the window in which the extrapolative mode develops. An orthogonality theorem would likely require stronger assumptions than Theorem~\ref{thm:task-distance-lower}, and we leave it to future work.

\subsection{Layer-band selection and a replication of Table~\ref{tab:interventions} with fixed bands}
\label{app:frozen-band}

Table~\ref{tab:interventions} reports each intervention averaged over a band of mid-depth layers (the simplex injection uses one layer). This section explains how the band is chosen, offers a tentative explanation for why the task-subspace $R^2$ peaks at mid-depth, and repeats the experiment of Table~\ref{tab:interventions} with the band fixed in advance and new seeds.

\begin{rmk}[Depth dependence]\label{rmk:depth}
The empirical claims in this paper refer to a mid-depth band of layers, selected as the layers maximizing held-out ID task-subspace $R^2$ (see \Secref{app:frozen-band}), because the properties are not uniform across depth. \texttt{P0} holds at every layer at long context (Table~\ref{tab:p0_p1_combined}a) and \texttt{P3} strengthens with depth up to the middle layers (\Secref{app:posterior-alignment-layers}), but \texttt{P1} and \texttt{P2} fail at the last layers of \texttt{E2} and \texttt{E3}, so a task vector extracted there satisfies none of \texttt{P1}--\texttt{P3}. Task information likely persists at the last layers but is re-encoded for next-token prediction rather than kept as an additive simplex; see \Secref{app:ols-probe-e3} for the supporting probe decomposition and \Secref{app:frozen-band} for a tentative explanation.
\end{rmk}

\paragraph{Selection criterion.} The band is the set of layers around the layer with the highest held-out ID task-subspace $R^2$ (the projection $R^2$ of \Secref{sec:representation_geometry} on ID sequences). The criterion uses only the task-vector geometry and no intervention. In the synthetic models it agrees with two other indicators. The task--token interaction $\eta^2_{\mathrm{interaction}}$ is small at these layers (Table~\ref{tab:p0_p1_combined}(b)), and the per-layer intervention effects are mode-selective there (\Secref{app:intervention_perlayer}). The same criterion selects the layers analyzed in the pretrained models (\Secref{app:real_llm}--\ref{app:llm-twelve}).

\paragraph{Why a mid-depth band: a tentative explanation.} We offer a tentative explanation for why the task-subspace $R^2$ peaks at mid-depth. In the early layers the model computes statistics shared by both modes. At layers $0$--$1$ the loss-relevant OOD subspace overlaps the task subspace far above chance (\Secref{app:principal-angles}), and the unique task contribution in the probe is small (\Secref{app:ols-probe-e3}). In the middle layers the sufficient statistic of the context is complete, and hidden states lie close to the task subspace. The task--token interaction is small (Table~\ref{tab:p0_p1_combined}b), the unique task contribution peaks, and interventions are mode-selective (\Secref{app:intervention_perlayer}). In the last layers the hidden state is re-encoded for next-token prediction. The interaction share rises, the model's own prediction accounts for $0.86$ of the residual stream, and interventions lose selectivity because both modes feed the same prediction.

Two further observations agree with this picture. When the sufficient statistic takes more computation, as with order-$3$ data, the band moves one layer deeper (\Secref{app:order-sweep}). In the pretrained models the same criterion selects layers at $41$--$71\%$ relative depth (\Secref{app:llm-twelve}). The three-stage picture also matches stages-of-inference analyses of transformer depth \citep{tenney2019bert, nostalgebraist2020logit, belrose2023tuned, lad2024stages}, the finding that intermediate layers carry the strongest linear features across architectures \citep{skean2025layer}, and the common practice of extracting task and function vectors \citep{hendel2023context, toddfunction} and of editing and steering \citep{meng2022locating, rimsky2024caa} at middle layers. Late layers still contain task information, but no longer as a separate direction that a linear projection recovers, and we leave the characterization of the late-layer re-encoding to future work.

\paragraph{Replication with the bands fixed in advance.} The band was originally identified partly from the per-layer intervention profile. To check that the effects in Table~\ref{tab:interventions} do not depend on this choice, we repeated the experiment with the bands fixed at their published values and all data redrawn under three new seeds, for \texttt{E1} (layers $1$--$4$, diversity exponent $m=10$) and \texttt{E3} (layers $2$--$4$, $m=12$). Table~\ref{tab:frozen-band} reproduces the double dissociation. Suppressing $\operatorname{col}(\hat{\mTheta})$ erases more than the full in-context gain on major tasks while leaving OOD loss largely intact, and suppressing $\hat{\mV}_{\mathrm{opt}}$ erases a large fraction of the OOD gain while the random same-rank control stays near $3\%$. The magnitudes vary with the seed and across the layers of the band, but in every seed and both experiments the two interventions dissociate, which is the claim of Table~\ref{tab:interventions}.

\begin{table}[htbp]
\centering
\scriptsize
\setlength{\tabcolsep}{4pt}
\caption{\textbf{Replication of Table~\ref{tab:interventions} with the layer bands fixed in advance.} Relative loss degradation $\Delta\gL_{\mathrm{mode}}/g_{\mathrm{mode}}\times 100\%$, band mean $\pm$ sd over three new seeds, with the values of Table~\ref{tab:interventions} in parentheses. ``Rand.''\ is the random same-rank control of the orthogonal intervention.}
\label{tab:frozen-band}
\resizebox{\textwidth}{!}{%
\begin{tabular}{@{}l l cc cc c@{}}
\toprule
& & \multicolumn{2}{c}{Task Subspace Interv.\ (\%)} & \multicolumn{3}{c}{Orth.\ Subspace Interv.\ (\%)} \\
\cmidrule(lr){3-4}\cmidrule(lr){5-7}
Exp. & Band & Maj. & OOD & Maj. & OOD & Rand. \\
\midrule
\texttt{E1} & $1$--$4$ & $177.0\pm 8.2$ (211.1) & $31.7\pm 0.4$ (12.1) & $34.3\pm 13.4$ (4.0) & $131.0\pm 6.4$ (95.8) & $3.4\pm 0.1$ (2.9) \\
\texttt{E3} & $2$--$4$ & $129.7\pm 23.1$ (132.0) & $-1.6\pm 0.1$ (13.1) & $3.8\pm 0.6$ (2.1) & $71.6\pm 5.8$ (99.6) & $3.3\pm 0.2$ (3.0) \\
\bottomrule
\end{tabular}}
\end{table}

\section{Beyond First-Order Markov: Finite-Order Sweep and the Dyck Boundary}
\label{app:order-sweep}

The main text studies first-order Markov sources (\texttt{E1}--\texttt{E3}) and the Dyck language (\texttt{E4}), which has no finite order. This section covers the range between them. We repeat the main analyses of Secs.~\ref{sec:task-vectors}--\ref{sec:two-modes} on mixtures of order-$w$ Markov chains for $w$ up to $10$, and then measure why no compact belief state forms for Dyck. Throughout, $m$ is the diversity exponent and the mode diagnostic is $\log(\mathrm{KL}_{\mathrm{Bayesian}}/\mathrm{KL}_{\mathrm{extrapolative}})$ of \Secref{sec:two-modes} at the final checkpoint, negative for \texttt{M1} and positive for \texttt{M2}. The value of $m$ at which it changes sign, $m^\ast$, is the transition point between the two modes.

\paragraph{Setup.} An order-$w$ chain conditions on the last $w$ tokens, $\P(s_{t+1}\mid z, s_{\le t}) = \P(s_{t+1}\mid z, s_{t-w+1:t})$. The latent task $z$ is a transition tensor with $V^{w}$ rows, each drawn from $\mathrm{Dir}(\boldsymbol{1}_V)$ as in \texttt{E3}, and $w=1$ recovers \texttt{E3}. The extrapolative predictor \texttt{M2} is the smoothed empirical $(w{+}1)$-gram, and \texttt{M1} is the exact posterior over the training pool under the order-$w$ likelihood. The sweep uses vocabulary $V=8$, the training mixture of \texttt{E3}, and the same $6$-layer, $d=128$ model. A sequence of $T=192$ tokens provides fewer than $200$ transitions while an order-$w$ chain has $V^{w}$ contexts, so at $V=8$ the chain cannot be estimated from one sequence beyond $w=3$. To reach order $10$ we also use $V=2$. All models are trained for $30$k steps; the $V=8$ models of orders $1$--$3$ are also trained for $100$k steps, and we report those where available.

\paragraph{Finite order within the framework.} An order-$w$ chain is a first-order chain on the space $\gV^{w}$ of windows of $w$ tokens. The framework of \Secref{sec:math} therefore applies once the hidden state is conditioned on the task and the last $w$ tokens, $(z, s_{t-w+1:t})$, instead of the task and the current token, $(z, s_t)$, as in Eq.~\ref{eq:cell-means}. We call the number of most recent tokens that an analysis conditions on its conditioning window, $w_{\mathrm{cond}}$; the main text uses $w_{\mathrm{cond}}=1$. Dyck lies outside this family, since the next planted bracket depends on the whole bracket prefix.

\begin{table}[htbp]
\centering
\small
\setlength{\tabcolsep}{5pt}
\caption{\textbf{Mode diagnostic versus Markov order ($V=8$).} Final $\log(\mathrm{KL}_{\mathrm{Bayesian}}/\mathrm{KL}_{\mathrm{extrapolative}})$ as a function of the diversity exponent $m$; negative favors \texttt{M1}. Orders $1$--$3$ are from the $100$k-step grid; at $w\ge 4$ (shown at $30$k steps) neither reference predictor can be computed from one context, the ratio is about $+0.15$ for every $m$, and no transition point is defined.}
\label{tab:order-transition}
\begin{tabular}{@{}cc rrrrrr c@{}}
\toprule
order $w$ & steps & $m{=}0$ & $m{=}2$ & $m{=}4$ & $m{=}6$ & $m{=}8$ & $m{=}10$ & transition $m^\ast$ \\
\midrule
1 & 100k & $-5.53$ & $-5.34$ & $-4.64$ & $-3.29$ & $-0.75$ & $+2.07$ & ${\approx}\,8.5$ \\
2 & 100k & $-4.69$ & $-4.30$ & $-3.35$ & $-0.28$ & $+3.61$ & $+4.65$ & ${\approx}\,6.1$ \\
3 & 100k & $-3.71$ & $-2.44$ & $+1.30$ & $+2.78$ & $+3.45$ & $+3.43$ & ${\approx}\,3.3$ \\
\midrule
4 & 30k & $+0.12$ & $+0.18$ & $+0.20$ & $+0.21$ & $+0.19$ & $+0.16$ & -- \\
5 & 30k & $+0.14$ & $+0.16$ & $+0.16$ & $+0.16$ & $+0.16$ & $+0.14$ & -- \\
\bottomrule
\end{tabular}
\end{table}

\begin{table}[htbp]
\centering
\small
\setlength{\tabcolsep}{5pt}
\caption{\textbf{Binary-vocabulary ladder to order $10$ ($V=2$, $30$k steps).} Same diagnostic as Table~\ref{tab:order-transition}. With $V=2$ the context is long enough to estimate the chain at every order shown; the transition is present at order $10$ and its transition point moves to lower diversity as $w$ grows.}
\label{tab:order-ladder}
\begin{tabular}{@{}c rrrrrr c@{}}
\toprule
order $w$ & $m{=}0$ & $m{=}2$ & $m{=}4$ & $m{=}6$ & $m{=}8$ & $m{=}10$ & transition $m^\ast$ \\
\midrule
1  & $-0.67$ & $-0.16$ & $-0.16$ & $-0.06$ & $-0.02$ & $+0.00$ & ${\approx}\,9.7$ \\
3  & $-0.93$ & $-1.09$ & $-0.86$ & $-0.34$ & $+0.31$ & $+0.38$ & ${\approx}\,7.0$ \\
5  & $-2.20$ & $-1.97$ & $-1.36$ & $-0.28$ & $+1.70$ & $+1.62$ & ${\approx}\,6.3$ \\
7  & $-3.39$ & $-2.74$ & $-1.39$ & $+0.94$ & $+1.92$ & $+2.60$ & ${\approx}\,5.2$ \\
10 & $-2.12$ & $-0.58$ & $+1.65$ & $+2.24$ & $+2.56$ & $+2.45$ & ${\approx}\,2.5$ \\
\bottomrule
\end{tabular}
\end{table}

\begin{figure}[htbp]
\centering
\includegraphics[width=0.85\linewidth]{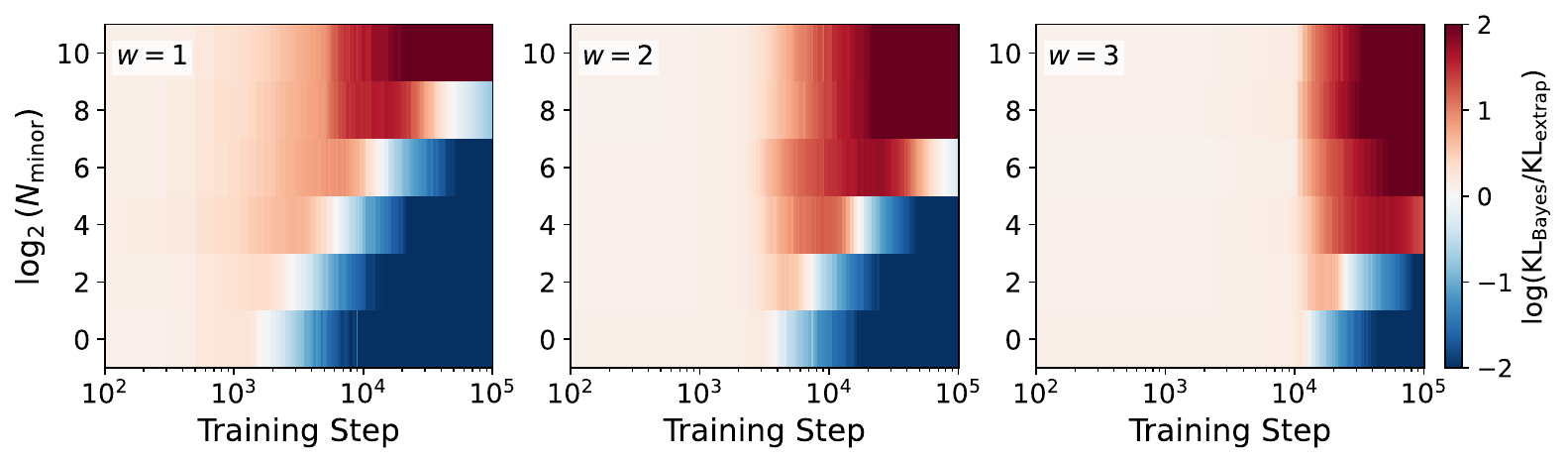}
\caption{\textbf{Phase transition at orders $w=1,2,3$ ($V=8$, $100$k steps).} Same layout and color scale as Fig.~\ref{fig:kl_combined}: training step on the horizontal axis, diversity exponent $m=\log_2 N_{\mathrm{minor}}$ on the vertical axis; blue indicates \texttt{M1}, red \texttt{M2}. The diagonal boundary is present at every order and moves down (toward lower diversity) as $w$ increases.}
\label{fig:order-transition}
\end{figure}

\paragraph{The transition at every finite order.} Table~\ref{tab:order-transition} and Fig.~\ref{fig:order-transition} show that the transition of Fig.~\ref{fig:kl_combined} appears at every order at which both reference predictors can be computed from one context. The transition point moves to lower diversity as $w$ grows ($m^\ast\approx 8.5$, $6.1$, $3.3$ at $w=1,2,3$). This is the direction Theorem~\ref{thm:task-distance-lower} predicts, since the order-$w$ family has intrinsic dimension $V^{w}(V-1)$ and the extrapolative predictor therefore leaves a fixed task subspace at a smaller $N_{\mathrm{minor}}$. At $w\ge 4$ with $V=8$ neither predictor can be computed from one context, and the diagnostic is about $+0.15$ at every $m$. With $V=2$ the transition is present up to order $10$ and its transition point again moves to lower diversity with $w$ (Table~\ref{tab:order-ladder}), so the limit at $V=8$ comes from context length rather than from order.

\paragraph{Training length.} Between $30$k and $100$k steps the transition point at $V=8$ shifts to higher diversity by about $1.2$ at every order, so longer training slightly favors retrieval, in line with \citet{carroll2025dynamics}. The transition and its ordering across $w$ are unchanged.

\begin{table}[htbp]
\centering
\scriptsize
\setlength{\tabcolsep}{4pt}
\caption{\textbf{Variance decomposition with a wider conditioning window at layer $3$, $m=6$ ($30$k steps).} Residual variance ratio $\mathrm{SS}_{\mathrm{within}}/\mathrm{SS}_{\mathrm{total}}$ (lower is better) at a middle and at the last context position, and the interaction share $\eta^2_{\mathrm{int}}$ at the last position, when conditioning on the current token only ($w_{\mathrm{cond}}=1$) and on the last $w$ tokens ($w_{\mathrm{cond}}=w$). Left: $V=8$; right: $V=2$.}
\label{tab:order-anova}
\begin{tabular}{@{}cc ccc@{}}
\toprule
\multicolumn{5}{c}{$V=8$} \\
\cmidrule(lr){1-5}
$w$ & $w_{\mathrm{cond}}$ & resid.\ (mid) & resid.\ (last) & $\eta^2_{\mathrm{int}}$ (last) \\
\midrule
1 & 1 & 0.360 & 0.160 & 0.254 \\
2 & 1 & 0.702 & 0.664 & 0.096 \\
2 & 2 & 0.220 & 0.130 & 0.010 \\
3 & 1 & 0.800 & 0.794 & 0.198 \\
3 & 3 & 0.065 & 0.059 & 0.004 \\
4 & 1 & 0.780 & 0.780 & 0.178 \\
4 & 4 & 0.059 & 0.057 & 0.004 \\
5 & 1 & 0.771 & 0.773 & 0.165 \\
5 & 5 & 0.019 & 0.019 & 0.001 \\
\bottomrule
\end{tabular}
\hspace{12pt}
\begin{tabular}{@{}cc ccc@{}}
\toprule
\multicolumn{5}{c}{$V=2$ ladder} \\
\cmidrule(lr){1-5}
$w$ & $w_{\mathrm{cond}}$ & resid.\ (mid) & resid.\ (last) & $\eta^2_{\mathrm{int}}$ (last) \\
\midrule
1  & 1  & 0.254 & 0.181 & 0.235 \\
3  & 1  & 0.643 & 0.478 & 0.340 \\
3  & 3  & 0.361 & 0.182 & 0.202 \\
5  & 1  & 0.761 & 0.653 & 0.285 \\
5  & 5  & 0.346 & 0.168 & 0.203 \\
7  & 1  & 0.861 & 0.831 & 0.236 \\
7  & 7  & 0.187 & 0.157 & 0.030 \\
10 & 1  & 0.905 & 0.899 & 0.235 \\
10 & 10 & 0.084 & 0.081 & 0.005 \\
\bottomrule
\end{tabular}
\end{table}

\begin{figure}[htbp]
\centering
\includegraphics[width=0.9\linewidth]{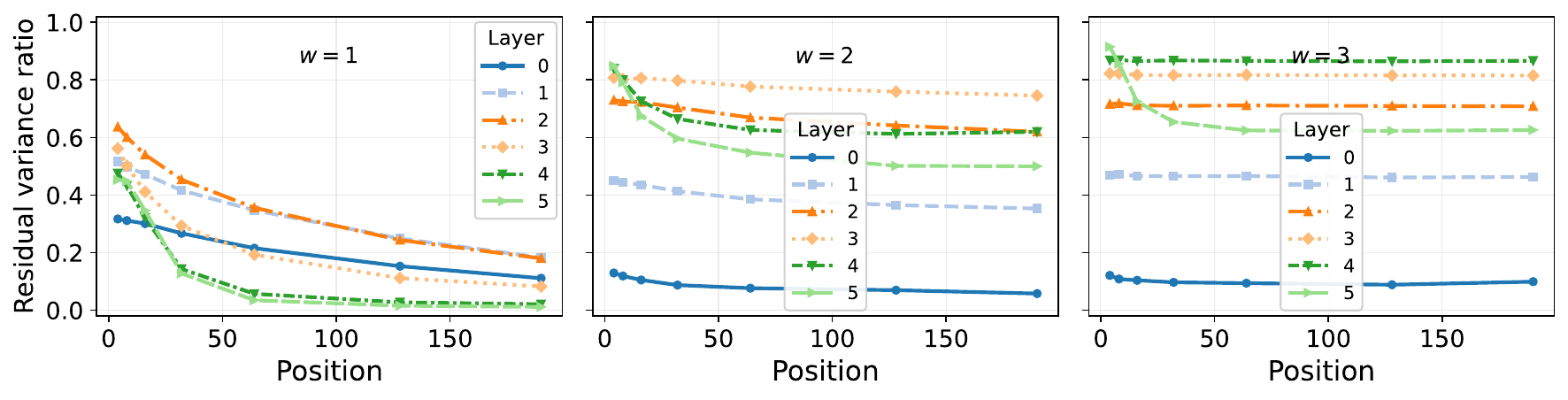}\\[2pt]
\includegraphics[width=0.9\linewidth]{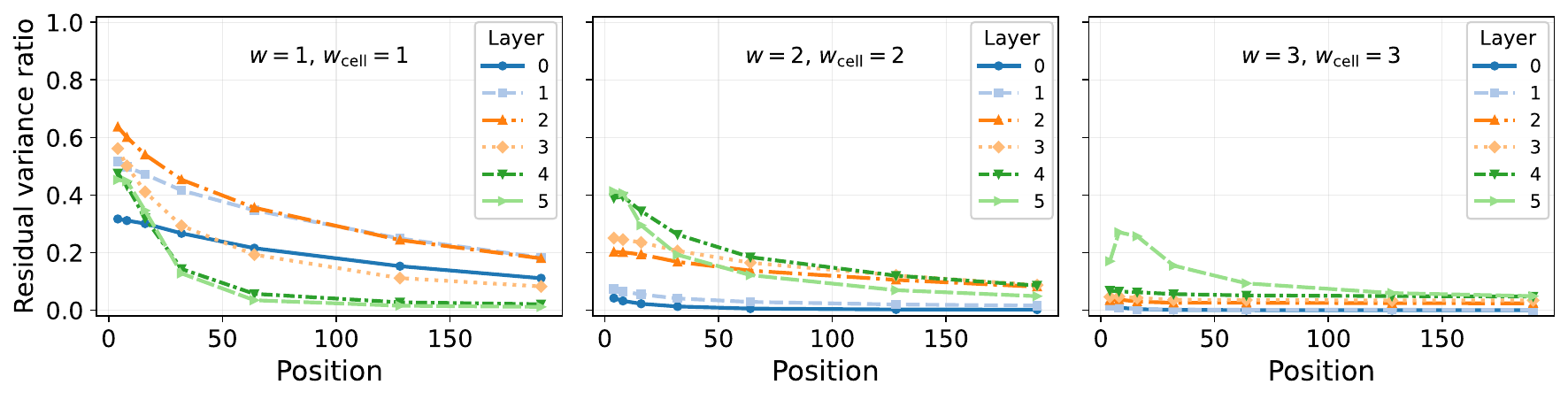}
\caption{\textbf{Residual variance ratio across context positions and layers ($V=8$, $100$k steps, $m=6$).} Panels: $w=1,2,3$. \textbf{Top:} conditioning on the current token only ($w_{\mathrm{cond}}=1$); at $w=2,3$ most of the variance is left unexplained at every position. \textbf{Bottom:} conditioning on the last $w$ tokens ($w_{\mathrm{cond}}=w$); the curves return to the \texttt{E3} profile of Table~\ref{tab:p0_p1_combined}(a), decaying with position and with depth.}
\label{fig:order-residual}
\end{figure}

\paragraph{\texttt{P0}--\texttt{P1}: the conditioning window must match the order.} Table~\ref{tab:order-anova} applies the variance decomposition of \Secref{sec:emp-task-vector} to the order-$w$ models. When the analysis conditions on the current token only, as in the main text, the residual variance ratio grows with order, reaching $0.77$ at $w=5$ for $V=8$ and $0.90$ at $w=10$ for $V=2$. Conditioning on the last $w$ tokens instead restores additivity at every order with $V=8$ and at orders $7$ and $10$ with $V=2$. The residual falls to $0.02$--$0.08$, with interaction share at most $0.005$. At orders $3$ and $5$ with $V=2$ the residual falls to $0.17$--$0.18$ but an interaction share of $0.20$ remains. The $100$k models are better still (Fig.~\ref{fig:order-residual}). \texttt{P0} and \texttt{P1} therefore hold once the conditioning window matches the order, and what fails at higher order is mainly conditioning on a single token.

\begin{table}[htbp]
\centering
\small
\setlength{\tabcolsep}{6pt}
\caption{\textbf{Posterior alignment versus order ($V=8$).} Hellinger distance in late context between the simplex coefficients $\beta_{t,k}$ and the order-aware Bayesian posterior $\alpha_{t,k}$ on ID sequences; the OOD column compares against the major-only posterior and the last column is a uniform-random reference. At order $3$ the best layer is one layer deeper.}
\label{tab:order-p3}
\begin{tabular}{@{}cc ccc@{}}
\toprule
order $w$ & layer & ID vs posterior & OOD vs major-only posterior & uniform-random \\
\midrule
1 & 3 & 0.116 & 0.575 & 0.660 \\
2 & 3 & 0.129 & 0.570 & 0.660 \\
3 & 4 & 0.127 & 0.562 & 0.660 \\
\bottomrule
\end{tabular}
\end{table}

\paragraph{\texttt{P3} at every order.} Table~\ref{tab:order-p3} repeats the analysis of Fig.~\ref{fig:beta_alpha_trajectory} with the order-aware posterior, using the Hellinger distance of Eq.~\ref{eq:est-hellinger}. The simplex coefficients track the posterior equally well at every order, far from the uniform reference, and OOD sequences stay far from any major-only posterior. The only change with order is depth. At $w=3$ the best layer is one layer deeper, since the $3$-gram statistic takes more computation.

\begin{table}[htbp]
\centering
\small
\setlength{\tabcolsep}{5pt}
\caption{\textbf{Rank-$1$ orthogonal-direction intervention per order ($m=10$).} Change in cross-entropy (nats) when the single direction, orthogonal to both the task and token subspaces, whose removal most damages OOD prediction is projected out, compared with random orthogonal directions.}
\label{tab:order-rank1}
\begin{tabular}{@{}c cccc@{}}
\toprule
order $w$ & max $\Delta$ OOD (layer) & $\Delta$ ID at that layer & mean $\Delta$ ID over layers & random direction \\
\midrule
1 & $+0.031$ (L0) & $+0.003$ & $+0.001$ & $+0.0001$ \\
2 & $+0.043$ (L4) & $-0.000$ & $+0.002$ & $-0.0000$ \\
3 & $+0.276$ (L5) & $+0.002$ & $+0.019$ & $+0.0000$ \\
\bottomrule
\end{tabular}
\end{table}

\begin{table}[htbp]
\centering
\small
\setlength{\tabcolsep}{5pt}
\caption{\textbf{The intervention of Table~\ref{tab:interventions} per order ($m=10$, layers $2$--$5$).} Rank-$6$ subspace found by optimization on held-out minor-task data, $\gamma=2.5$, post-MLP, with random same-rank controls (\Secref{app:orth_ablation}). Each entry reports $\Delta$OOD loss in nats / $\Delta$major loss as a percentage of the major baseline / $\Delta$OOD loss of the random control in nats.}
\label{tab:order-table2}
\begin{tabular}{@{}c ccc@{}}
\toprule
layer & order $1$ & order $2$ & order $3$ \\
\midrule
2 & $+0.16$ / $0.3\%$ / $+0.006$ & $+0.28$ / $7.3\%$ / $-0.003$ & $+0.31$ / $7.4\%$ / $+0.000$ \\
3 & $+0.20$ / $0.7\%$ / $+0.008$ & $+0.34$ / $6.4\%$ / $+0.016$ & $+0.40$ / $23.2\%$ / $+0.025$ \\
4 & $+0.44$ / $1.4\%$ / $+0.020$ & $+0.31$ / $4.2\%$ / $-0.001$ & $+0.49$ / $9.4\%$ / $-0.002$ \\
5 & $+0.26$ / $6.5\%$ / $+0.004$ & $+0.50$ / $5.8\%$ / $+0.010$ & $+1.49$ / $5.8\%$ / $+0.024$ \\
\bottomrule
\end{tabular}
\end{table}

\begin{figure}[htbp]
\centering
\includegraphics[width=0.9\linewidth]{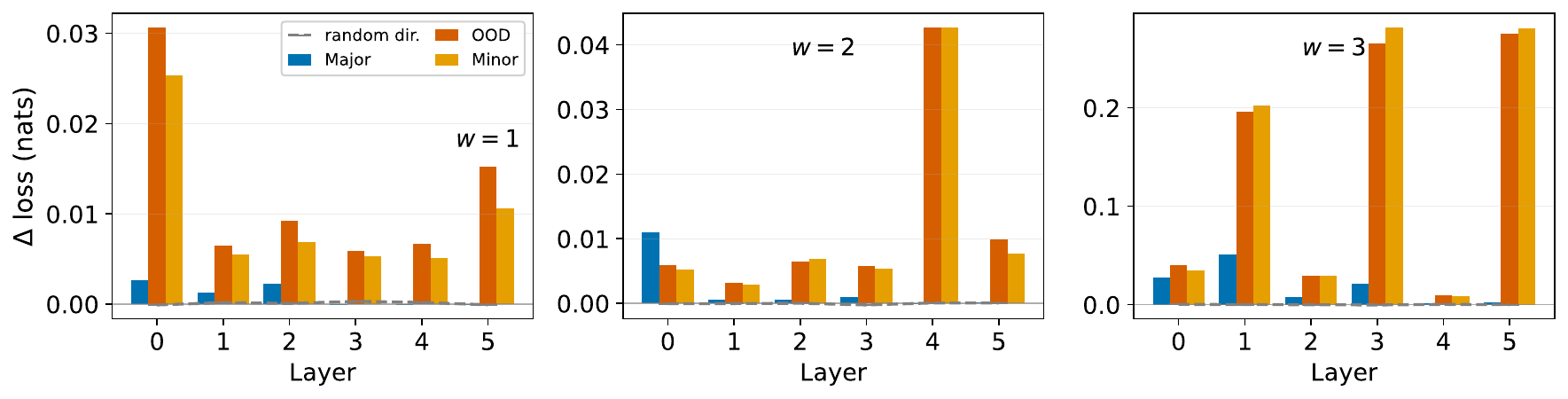}
\caption{\textbf{Per-layer loss change under the rank-$1$ orthogonal-direction intervention of Table~\ref{tab:order-rank1}.} Panels: $w=1,2,3$ (note the different vertical scales). Bars show $\Delta$ loss in nats for major, OOD, and minor sequences; the thin gray bars are random orthogonal directions. The dissociation sharpens with order and moves one layer deeper at $w=3$.}
\label{fig:order-orth}
\end{figure}

\paragraph{Orthogonal interventions at every order.} We run two versions of the orthogonal intervention of \Secref{sec:orth_ablation_main} on the $m=10$ models at each order. The first removes a single learned direction orthogonal to both the task and token subspaces (Table~\ref{tab:order-rank1}, Fig.~\ref{fig:order-orth}). It costs up to $+0.28$ nats of OOD loss at order $3$ while ID loss barely changes, and random orthogonal directions cost about $0$. The dissociation holds at every order and grows with it. The second repeats the setting of Table~\ref{tab:interventions} (Table~\ref{tab:order-table2}). We report nats rather than the fraction of the in-context gain, because at orders $2$--$3$ the OOD loss stays near the marginal entropy $\log 8$ throughout the context, so the gain is nearly zero and the fraction is undefined. At order $1$ the fraction is well defined. Layers $2$--$4$ erase $51$--$145\%$ of the OOD gain and $3$--$12\%$ of the major-task gain, in the range of the \texttt{E3} row of Table~\ref{tab:interventions}. At order $3$ the selective layer is layer $5$, one deeper than at order $1$, matching the depth shift seen for \texttt{P3}.

\begin{table}[htbp]
\centering
\small
\setlength{\tabcolsep}{8pt}
\caption{\textbf{Conditional variance of Dyck hidden states (\texttt{E4}).} Fraction of hidden-state variance left unexplained after conditioning on the open-parenthesis count versus on the full bracket prefix, averaged over Dyck prefix lengths $l=7$--$18$.}
\label{tab:order-dyck}
\begin{tabular}{@{}c cc@{}}
\toprule
layer & resid.\ fraction $\mid$ count & resid.\ fraction $\mid$ full prefix \\
\midrule
0 & 0.84 & 0.02 \\
1 & 0.81 & 0.03 \\
2 & 0.73 & 0.03 \\
3 & 0.58 & 0.03 \\
4 & 0.66 & 0.03 \\
5 & 0.63 & 0.04 \\
\bottomrule
\end{tabular}
\end{table}

\paragraph{The Dyck boundary: no compact belief state forms.} For Dyck, the natural candidate for a compact belief state is the running count of unmatched open parentheses. We ask whether the \texttt{E4} model of \Secref{sec:motivating-example} compresses its hidden state onto this count. Using $64$ tasks, $8$ planting masks with $32$ samples each, and hidden states at every planted-bracket position, we compute the fraction of variance left after conditioning on the count and after conditioning on the full bracket prefix, averaged over prefix lengths $7$--$18$. Table~\ref{tab:order-dyck} shows that the count explains at most $42\%$ of the variance, while the full prefix explains about $97\%$ at every layer. This holds at every prefix length in the range. Averaged over layers, the residual given the count stays between $0.52$ and $0.79$, while the residual given the prefix stays at $0.02$--$0.04$. The hidden state thus retains the prefix rather than a compact summary of it. This is the mechanism behind the failure of \texttt{P0} in \Secref{sec:motivating-example} and the prefix memorization in \Secref{app:dyck}.

In sum, the framework holds at every finite order once the conditioning window matches the order, and fails only for Dyck, which has no finite order.

\section{Non-Markovian Case: Dyck Path Experiment (\texttt{E4}) Details}\label{app:dyck-section}

\label{app:dyck}

This section provides the methodology and detailed results behind the planted Dyck experiment (\texttt{E4}) used as the motivating non-Markovian counterexample in \Secref{sec:motivating-example}. The experiment matters for our main thesis because it shows that property \texttt{P0} (long-context stability) is conditional on the data-generating structure rather than an innate property of trained transformers: when the source is non-Markovian and admits no compact sufficient statistic, the model retains the full prefix history instead of collapsing onto a low-dimensional summary. %
The main text (Figure~\ref{fig:dyck_combined}) reported the residual variance left in $\vh_t$ after conditioning on the task $z$ and a local context window, and visualized the cluster structure of the final-layer hidden states at a single prefix length.
This section sharpens that diagnosis along two complementary axes.
First, we train linear probes that classify hidden states by the identity of the \emph{full} planted Dyck prefix at every length $l=2,\ldots,7$ and report the resulting validation accuracies, providing direct quantitative evidence that the model retains an essentially lossless record of the prefix history rather than a compact summary statistic.
Second, we show that this record is recoverable through a single shared low-dimensional projection across all prefix lengths and mask realizations, indicating that the prefix information is arranged along consistent directions in hidden-state space.
Together these results constitute a sharp failure of property~\texttt{P0} in \texttt{E4} that does not occur in \texttt{E1}--\texttt{E3}.

\paragraph{Sequence construction.}
The main text introduces \texttt{E4} (\Secref{sec:motivating-example}) by describing each sequence as a length-$T$ background of i.i.d.\ uniform tokens into which the latent Dyck string $z$ of length $2L$ is planted at positions selected with probability~$\rho$. We expand on the planting rule here so the case $\rho T > 2L$ is unambiguous. Concretely, given $z = (z_1,\ldots,z_{2L})$ we draw a Bernoulli planting indicator vector $\vxi \in \{0,1\}^T$ with $\xi_t \stackrel{\text{i.i.d.}}{\sim} \mathrm{Bernoulli}(\rho)$, then enumerate its successes $1 \le \tau_1 < \tau_2 < \cdots$ in left-to-right order and overwrite background token $b_{\tau_j}$ with $z_j$ for every $j \in \{1,\ldots,\min(2L,\,\sum_{t=1}^{T}\xi_t)\}$. %
 With the experimental settings used throughout this paper, $T = 192$, $2L = 20$, $\rho = 0.25$, the expected number of Bernoulli successes is $\rho T = 48 \gg 2L$, so undersampling (fewer than $2L$ successes) occurs with vanishingly small probability and the entire latent string $z$ is essentially always planted; the suffix of the sequence after the $2L$-th plant is pure background and therefore carries no further task-conditional signal.

\paragraph{Prefix probe.}
Consider the model trained on experiment \texttt{E4} with $K$ Dyck tasks (balanced bracket strings of length $2L = 20$).
At each sequence position, define the \emph{Dyck prefix length} $l$ as the number of planted bracket characters observed so far.
Two tasks that share the same first $l$ planted characters are indistinguishable given the observed Dyck prefix up to that point.
We group tasks by their length-$l$ prefix and define the \emph{prefix class} of a hidden state as the identity of this prefix.
The number of distinct prefix classes grows with $l$: for instance, $l = 2$ yields only $2$ classes (open or close), while $l = 7$ yields $35$ classes.

We train a linear probe consisting of a shared $2$-dimensional linear projection $\mW \in \R^{2 \times d}$ followed by a separate two-layer $\texttt{MLP}$ classifier per prefix length.
The probe is trained jointly across all prefix lengths $l = 2, \ldots, 7$ (at $l=1$ there is a single class), using hidden states extracted from layer~$5$ (the final layer).
Training and validation sets are constructed by generating sequences under multiple random planting masks and collecting hidden states at positions where exactly $l$ Dyck characters have been planted.

\paragraph{Probing results.}
The probe achieves near-perfect validation accuracy ($>95\%$) for all prefix lengths up to $l=7$ (the longest tested), reaching $98.6\%$ at $l=7$ despite the rapid growth in the number of classes.
Figure~\ref{fig:dyck_proj} shows a shared $2$-dimensional linear projection at $l=7$, where the $35$ distinct prefix classes form well-separated clusters. Points are colored by the number of unmatched left parentheses, and the shaded regions depict the decision boundaries of the per-length \texttt{MLP} classifier. This visualization is identical to the right panel of Figure~\ref{fig:dyck_combined} in projection and coloring; the only difference is that here each cluster is labeled by its literal bracket prefix, whereas in the main-text panel the same prefixes are encoded compactly in hexadecimal for legibility. The hexadecimal encoding is obtained by mapping each step to a bit ($\text{`('}\!\to\!1$, $\text{`)'}\!\to\!0$), dropping the leading `(' (present for all prefixes), padding the remaining bits to a multiple of~$4$, and converting to hexadecimal; for example, the length-$7$ prefix $(()())($ yields bits $1,0,1,0,0,1$ after dropping the first `(', and padding to $8$ bits gives $10100100_2=\texttt{A4}$. 

This result has two implications.
First, the hidden states encode a representation of the full Dyck prefix, not a compressed statistic such as the running difference between opening and closing brackets (which would yield only $\lfloor l/2 \rfloor + 1$ distinguishable states at prefix length $l$, not the valid Dyck-prefix classes observed).
Second, because this encoding is recoverable through a shared $2$-dimensional linear projection, the prefix information is arranged along consistent directions across different prefix lengths and mask realizations.

\begin{figure}[htbp]
\centering
\includegraphics[width=0.42\linewidth]{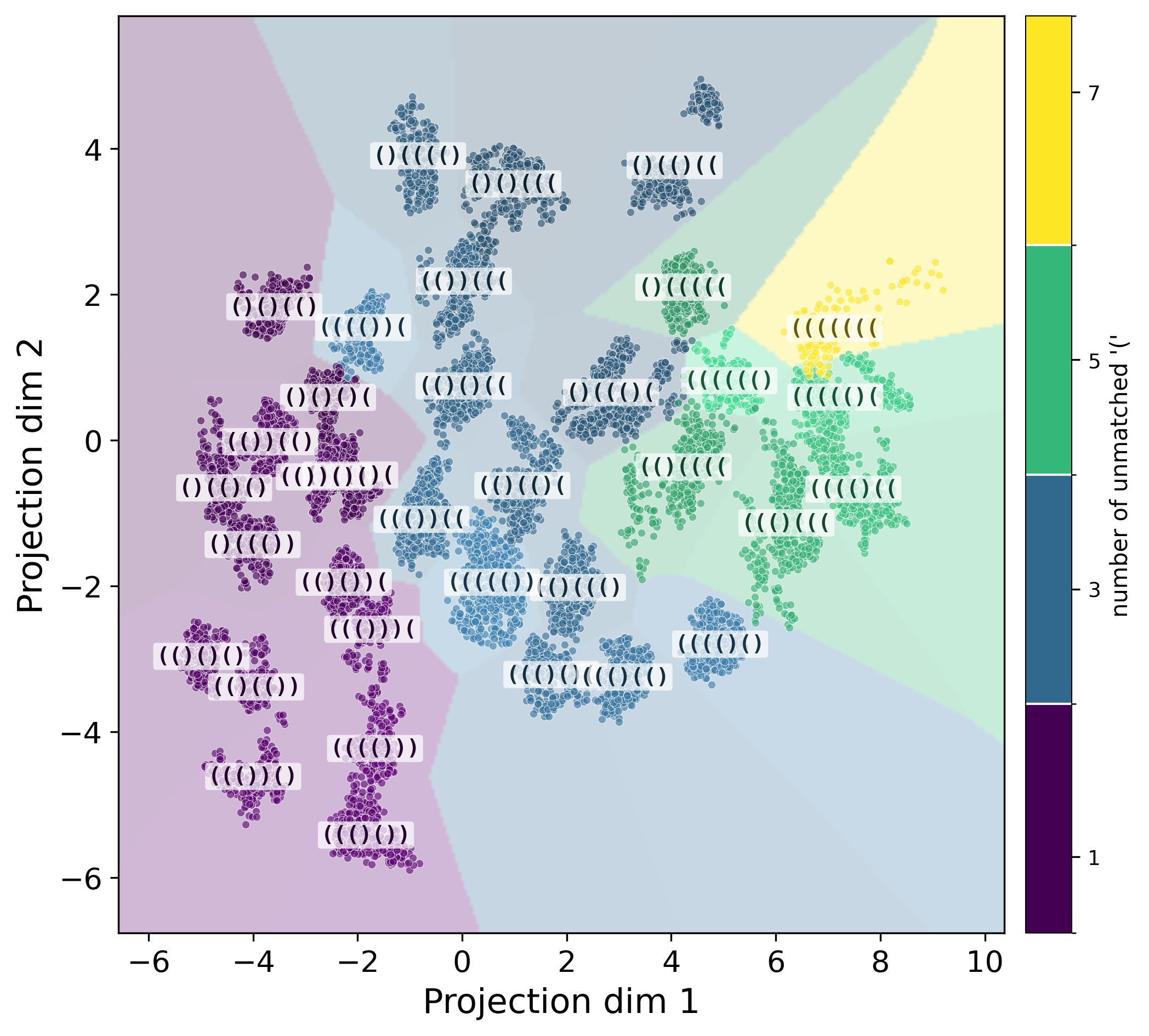}
\caption{\textbf{Hidden-state geometry for the planted Dyck language (\texttt{E4}).}
Two-dimensional projection of final-layer hidden states at prefix length $l=7$ (yielding $35$ distinct Dyck prefixes). Points form well-separated clusters, with both clusters and induced decision regions colored by the number of unmatched left parentheses. The clear separation across all prefix classes indicates that the representation preserves the full Dyck prefix history, rather than compressing it to low-dimensional summary statistics. %
} 
\vspace{-15pt}
\label{fig:dyck_proj}
\end{figure}

\end{document}